\documentclass[12pt,onecolumn]{article}

\usepackage{geometry}
\usepackage{times}
\usepackage{epsfig}
\usepackage{graphicx}
\usepackage{amsmath}
\usepackage{amssymb}
\usepackage{algorithm}
\usepackage{algorithmic}
\usepackage{enumerate}
\usepackage{booktabs}
\usepackage{float}
\usepackage{subfigure}
\usepackage{times}                                              
\usepackage{color}                                              
\usepackage{caption}                                            
\usepackage{indentfirst}                                        
\usepackage{graphicx,psfrag}                                    

\usepackage{bm}
\usepackage{makecell}
\usepackage{multirow}
\usepackage{natbib}  
\newtheorem{definition}{Definition}[section]
\newtheorem{theorem}{Theorem}[section]
\newtheorem{corollary}{Corollary}[theorem]
\newtheorem{lemma}[theorem]{Lemma}
\newtheorem{remark}{Remark}
\newenvironment{proof}{{\bf \emph{Proof.} }}{\hfill $\Box$ \\} 
\newtheorem{problem}{Problem}

\graphicspath{{./fig_all_arxivNaming/}} 

\begin{document}

\title{Choice of Scoring Rules for Indirect Elicitation of Properties with Parametric Assumptions}

\author{ Lingfang Hu \\ \textit{lhu25@uic.edu} \and Ian A. Kash \\ \textit{iankash@uic.edu} }
\date{Department of Computer Science\\ University of Illinois at Chicago, Chicago, IL 60607, USA}

\maketitle

\begin{abstract}

People are commonly interested in predicting a statistical property of a random event such as mean and variance. 
Proper scoring rules assess the quality of predictions and require that the expected score gets uniquely maximized at the truthful report of a subjective belief ex-ante and the precise prediction ex-post, in which case we call the score directly elicits the property. 
Previous research work has widely studied the existence and the characterization of proper scoring rules for different properties. However, little literature discusses the choice of proper scoring rules for applications at hand. 
In this paper, we explore a novel task, the indirect elicitation of properties with parametric assumptions, where the target property is a function of several directly-elicitable sub-properties and the total score is a weighted sum of proper scoring rules for each sub-property and thus it indirectly infers the target property. Because of the restriction to a parametric model class, different settings for the weights lead to different constrained optimal solutions. Our goal is to figure out how the choice of weights affects the estimation of the target property and which choice is the best. 
We start it with simulation studies and observe an interesting pattern: in most cases, the optimal estimation of the target property changes monotonically with the increase of each weight, and the best configuration of weights is often to set some weights as zero. To understand how it happens, first of all, we establish the elementary theoretical framework. Basically, we decompose the observation into two parts, the movement of the target property with sub-properties and the movement of the sub-property estimation with weights. The overall monotonicity pattern would happen if both movements are monotone. Next, we provide deeper sufficient conditions for the case of two sub-properties and of more sub-properties respectively. The theory on 2-D cases perfectly interprets the experimental results. In higher-dimensional situations, we especially study the linear cases and suggest that more complex settings can be understood with locally mapping into linear situations or using linear approximations when the true values of sub-properties are close enough to the parametric space. 

\end{abstract}


\section{Introduction} \label{section-introduction}



It is common in our daily life to make probabilistic forecasts for random events, which helps people understand the uncertainty and more importantly guide their decision-making such as weather forecasting \citep{brier1950verification}, the US presidential election forecasts \citep{gelman2020information}, and all different kinds of predictive models developed in machine learning \citep{alpaydin2020introduction}. People are normally interested in predicting the full probability for simple distributions or a more abstract statistical property for complicated distributions, that is, a function of probability distributions such as mean and variance. It is important to use some measurements to assess the quality of predictions in order to incentivize forecasters' efforts to make good predictions. Scoring rules are one of the popular measurements, which assign a numerical value $s(r,X)$ to each pair of prediction $r$ and realized outcome $X$. Given a property $\Gamma(\cdot)$ for prediction, one natural requirement for scoring rules is that the expected score $E_{X\sim \bm p}[s(r,X)]$ should get maximized at the precise prediction $\Gamma(\bm p)$ for all distribution $\bm p$, and we call them {\em proper} scoring rules for the property. Furthermore, if the maximizer is unique, then we call the scoring rule {\em strictly proper} for the property and the score directly elicits the property.

Early work about proper scoring rules dates back to 1950s, and they mainly restricts to the elicitation of full probabilities over finite outcome spaces \citep{brier1950verification,good1952rational,mccarthy1956measures}. It drove researchers from the theoretical side to think about more generally what properties are directly-elicitable and what the corresponding proper scoring rules look like. For example, it is well-known that general full probabilities and mean are directly-elicitable and the strictly proper scoring rules for them can be derived from strictly convex functions \citep{savage1971elicitation,gneiting2007strictly}, while variance is not directly-elicitable \citep{osband1985providing}. See \cite{gneiting2007strictly,gneiting2011making} for early surveys, and also \cite{frongillo2024recent} for recent research trends. From the practical side, proper scoring rules have also been widely applied in different domains such as meteorology, prediction markets, psychology, energy, and a survey shows that there has been been more and more publications on this track \citep{carvalho2016overview}. 

Given a directly-elicitable property, the characterization theory provides us with rich options of scoring rules to use in practice. However in practice, surprisingly most of the time people just used one of a few typical proper scoring rules for specific applications, including Brier score, Logarithmic score, and spherical score \citep{gneiting2007strictly}. Compared to the large body of work along the characterization of proper scoring rules, the literature on the choice of proper scoring rules is very sparse. 

The first question to ask here is: does the choice matter? The choice does not matter in terms of properness, which only guarantees eliciting the truthful report of personal beliefs ex-ante and assigning the best score to the precise prediction ex-post. But in fact, the choice does matter from many other perspectives beyond properness. \citet{winkler1996scoring} gave an early comprehensive discussion about the possible directions, see their Section 6 for detail. A few directions have been further developed later, but the development is sparse and slow. 
For example, from the perspective of computational efficiency, {\em local} scoring rules such as Logarithmic score are preferred in that the value for each realized outcome only depends on the predicted probability on that outcome \citep{bernardo1979expected,ehm2009local,parry2012properlocal,du2021beyond}. Being {\em accuracy-rewarding} (or {\em distance-sensitive} or {\em order-sensitive}) is another desirable criterion, which makes the score increase with prediction accuracies \citep{stael1970family,friedman1983effective,lambert2008eliciting,steinwart2014elicitation}. 
From the decision-making perspective, \cite{johnstone2011tailored} constructed proper scoring rules directly from the utility function of decision makers, which embodies the customized needs for each decision maker. 
From the incentive-compatible perspective, some recent work studied the optimal scoring rules which aim to incentivize the forecaster's best effort to refine his or her posterior belief after observing some correlated signals \citep{li2022optimization,chen2021optimal,hartline2023optimal}. Overall, different perspectives have different needs and lead to different answers. 

\begin{figure}[!htbp]
  \centering
  \includegraphics[width=0.35\textwidth]{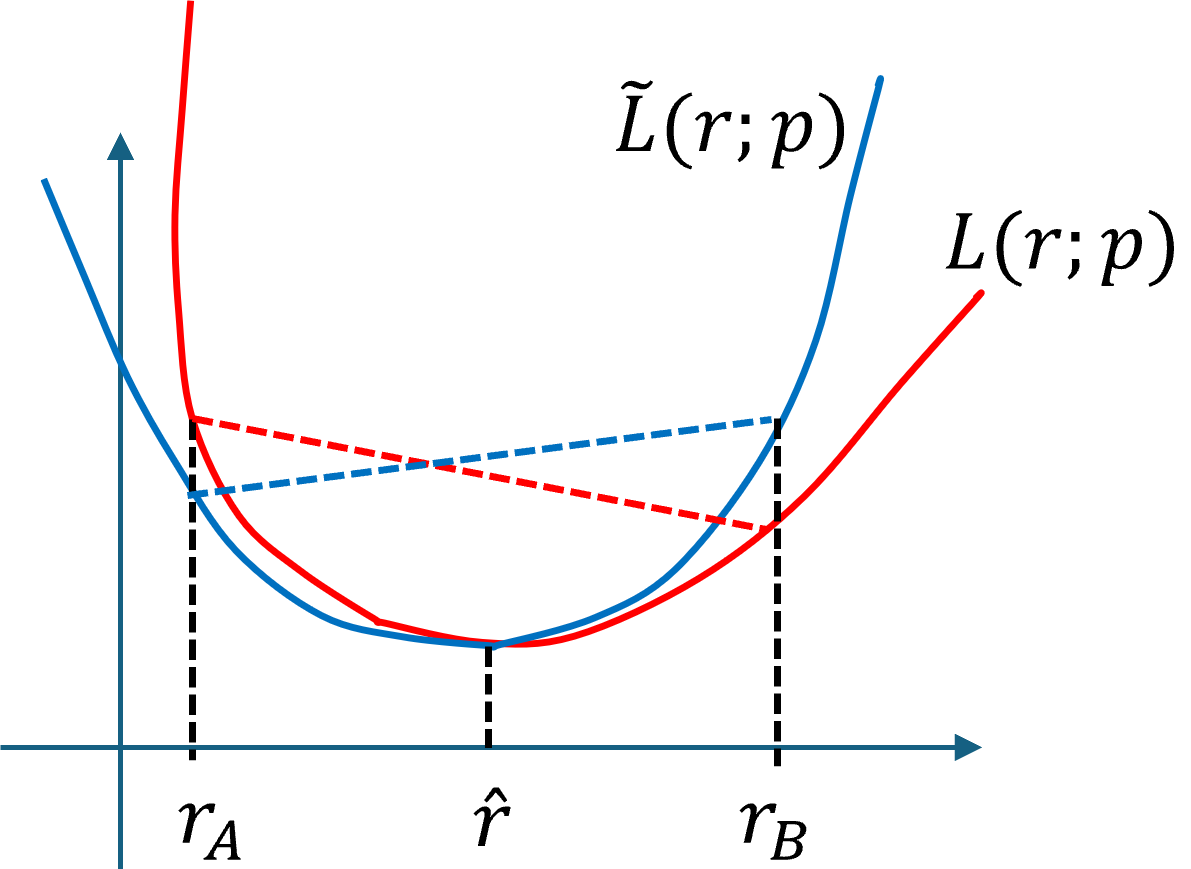}  
  \caption{ An example where two different losses assign different rankings to two imprecise predictions. $\hat{r}$ is the true value of some property of $\bm p$, and $r_A$ and $r_B$ are two different imprecise predictions. The loss $L$ prefers $r_B$ to $r_A$ while the loss $\tilde{L}$ prefers $r_A$ to $r_B$. }  
  \label{fig_loss_comparison_on_imperfect_predictions}
  
\end{figure}

In this paper, we study the choice of proper scoring rules with a focus on the evaluation of imperfect predictions. A key insight is that different proper scoring rules potentially have different attitudes towards the imperfect forecasts, although all of them takes the perfect forecast as the best one, see Figure \ref{fig_loss_comparison_on_imperfect_predictions} for an illustration. There is some work studying the ranking orders among (imperfect) forecasts incurred by different proper scoring rules. \citet{bickel2007some} and \citet{machete2013contrasting} did a theoretical examination for different scoring rules and revealed that the choice does matter in terms of the rank orders, but their discussions are limited to the three commonly used scores. \citet{merkle2013choosing} tested a broader family of proper scoring rules and demonstrated that part of scoring rules would lead to very similar rankings among forecasters while part of them lead to obviously different rankings. However, their work is a purely empirical study and lacks the theoretical understanding. In contrast, we consider the choice among general proper scoring rules, and work on a different task -- parametric model estimation, which is commonly considered in statistics and machine learning. Generally, the parametric model will not perfectly fit the data, and different scoring rules would potentially give different comparisons among models and then lead to different optimal parametric solutions. We pursue to provide a theoretical understanding of how different scoring rules affect the optimal parametric estimation. 

Another key insight from the literature on property elicitation is that more properties are not directly-elicitable but they can always be indirectly elicited in the sense that any statistical property can be calculated from the full probability distribution and full probabilities are directly-elicitable as mentioned before. A less trivial indirect-elicitation style can happen when a target property is a function of several simpler sub-properties and the sub-properties are directly-elicitable. For example, variance can be estimated by eliciting the first two moments. In this case, we can use a positive weighted sum of proper sub-scoring rules for each sub-property as the total scoring rule, and then the prediction of the target property can be calculated by combining the predictions of all of those sub-properties. 

In view of the two insights discussed above, it is worth studying the choice of proper scoring rules in the task of indirect elicitation of a property with a parametric model assumption. To simplify the analysis, we suppose that all sub-scores are fixed and narrow down the scope to only the choice of weights. Intuitively, setting a larger value for one of the weights means putting more emphasis on the precision of estimating the corresponding sub-property. Different settings for the weights would lead to different constrained optimal solutions. Then what principle should we use for the best choice of scoring rules? From the viewpoint of indirect-elicitation, a natural idea is that the total score should encourage forecasters to predict the target property as precisely as possible. So in this paper we pursue a total scoring rule which can achieve the best estimation of the target property itself. 

The framework of trade-off between different function terms is very common in the loss design in machine learning, and the best weights are usually chosen based on empirical results in practice. However, normally researchers from the applied side do not seek to understand the reason of the best setting. In contrast, our focus is to shed some light on the theoretical understanding of how the weights reflect the preference of the target property for its different sub-properties. Naively, we might expect that it is a reasonable option to set all of weights to be equal to one another. However, our results show that it is not the best choice in many cases. 

We summarize our main results in five parts as follows. 

\begin{itemize}   
    \item[(1)] We first give a separable weighted-sum loss framework (the negation of a score) for property indirect-elicitation, and then formalize the main problem of how the choice of weights would affect the optimal parametric estimation of the target property.
    \item[(2)] We then do simulation studies, from which we observe the unexpected monotonicity pattern. In all tested cases the estimation of the target property change monotonically with the increase of each weight, and the best configuration for the weights is often to set one of weights as zero.
    \item[(3)] Next, we establish the elementary theoretical framework for the monotonicity pattern. Basically, the observation are decomposed into two parts, the movement of the target property with sub-properties and the movement of the sub-property estimation with weights, and the overall monotonicity pattern would happen if both movements are monotone. 
    \item[(4)] Then we provide more concrete assumptions for the desirable structure of the parametric model and the target property such that the elementary conditions specified in the general framework can be satisfied. The theory for the case of two sub-properties perfectly explain our observation from the simulation studies. 
    \item[(5)] Furthermore, we extend the analysis from 2-D into higher-dimensional situations, and especially study the linear cases. We suggest that more complex settings can be understood by locally mapping into linear situations or using linear approximations.
\end{itemize}

\section{Preliminaries}

\subsection{Proper scoring rules or losses}


Let $O$ the outcome space of a random event, $X$ the random variable over $O$, and $\mathcal{D}$ a convex subset of probability distributions over $O$. 
Let $\bar{\mathbb{R}} = \mathbb{R}\cup \{+\infty, -\infty\}$, and $\mathcal{R}$ be any non-empty set of reports where the report is generally real-valued or vector-valued.

\begin{definition}[Statistical Property]
A statistical property is a function of probability distributions $\Gamma(\cdot): \mathcal{D} \rightarrow \mathcal{R}$. 
\end{definition}
 
\begin{definition}[Scoring Rule]
A scoring rule is a function $s(\cdot, *): \mathcal{R} \times O \rightarrow \bar{\mathbb{R}}$, representing a score assigned to a prediction report $\bm r \in \mathcal{R}$ with some materialized outcome $o \in O$. The expected score for reporting $\bm r$ given a true distribution $\bm p$ is 
\[
S(\bm r; \bm p) \triangleq E_{X\sim \bm p} [s(\bm r, X)].
\]

\end{definition}

\begin{definition}[Proper Scoring Rules w.r.t. Properties]
A scoring rule $s(\cdot, *): \mathcal{R} \times O \rightarrow \bar{\mathbb{R}}$ is proper for a property $\Gamma(\cdot): \mathcal{D} \rightarrow \mathcal{R}$ if for all $\bm p \in \mathcal{D}$,
\[
\Gamma(\bm p) \in \arg\max\limits_{ \bm r \in \mathcal{R} } S(\bm r; \bm p).
\]
If $\Gamma (\bm p)$ is the unique maximizer for all $\bm p \in \mathcal{D}$, we say that the scoring rule $s$ is strictly proper with the property $\Gamma$ and $s$ directly elicits $\Gamma$. 
\end{definition}
Here we require $S(\Gamma(\bm p); \bm p) \in \mathbb{R}$ for all $\bm p \in \mathcal{D}$ as described in \citep{gneiting2007strictly}.\footnote{We generally assume that scoring rules satisfy some regularity conditions, such as being integrable or quasi-integrable. The ultimate goal is to guarantee that scoring rules are well-defined and the expected score is finite at desirable places. }
We also refer to a proper loss function as the negation of a proper scoring rule, that is , $L(\bm q; \bm p) = -S(\bm q; \bm p)$.

\begin{definition}[Accuracy-rewarding scoring rules \citealp{lambert2008eliciting}] \label{def-accuracyRewarding}
A scoring rule $s(\cdot, *): \mathcal{R} \times O \rightarrow \bar{\mathbb{R}}$ with the report space $\mathcal{R} \subseteq \mathbb{R}^m$ ($m\in \mathbb{N}_{+}$) is accuracy-rewarding with respect to a property $\Gamma(\cdot): \mathcal{D} \rightarrow \mathcal{R}$ if for all $\bm p \in \mathcal{D}$ $S(\bm r; \bm p) < S(\bm r'; \bm p)$ when either $r_i<r'_i \leq \Gamma_i(\bm p)$ or $\Gamma_i(\bm p)\leq r'_i < r_i$ for all $i\in \{1, \cdots, m\}$.
\end{definition}

Obviously, accuracy-rewarding scoring rules are strictly proper. From the other side, it still remains an open problem whether or not strictly proper scoring rules are always accuracy-rewarding, although \cite{lambert2008eliciting} and \cite{steinwart2014elicitation} showed that under some conditions the answer is yes. Being accuracy-rewarding is an important requirement for our theoretical analysis throughout the paper and thus we will use it as an explicit assumption on scoring rules.

\subsection{Property elicitation: direct or indirect}

A property $\Gamma(\cdot): \mathcal{D} \rightarrow \mathcal{R}$ is said to be directly-elicitable if there exists a strictly proper scoring rule eliciting it. It is well-known that $\Gamma(\bm p)= \bm p$ is always directly-elicitable \citep{gneiting2007strictly}. 
Besides, quantiles, ratios of expectations, and (vector-valued) linear properties, i.e. $\Gamma(\bm p) = E_{X\sim \bm p} [\bm \rho(X)]$ such as mean and general moments, are directly-elicitable. For example, according to Lemma 7 in \cite{abernethy2012characterization}, we know that the following equation is one type of strictly proper losses for linear properties,
\[
l(\bm r, X) = (\bm \rho (X) - \bm r)^T C (\bm \rho (X) - \bm r),
\]
where $C$ is a symmetric and positive-definite matrix. 

More properties are not directly-elicitable, such as variance, skewness, and kurtosis. However, they can always be indirectly elicited in the sense that we can always first elicit the full probability distribution $\Gamma_0(\bm p) = \bm p$ and then calculate the property from the full probability with $\Gamma(\bm p) = \Gamma(\Gamma_0(\bm p))$. Very often, the indirect elicitation would require less information. For example, variance can be indirectly elicited from the first two moments. There is growing literature on elicitation complexity \citep{lambert2008eliciting,fissler2016higher,frongillo2021elicitation}.

\section{Main problem}

In this paper, we consider the parametric model estimation for the indirect elicitation of a target property. Our goal is to figure out how the choice of scoring rules would affect the optimal parametric estimation of the target property. Now we formally present the problem. 

Consider a target property $\Gamma(\bm p) = t(\hat{\bm r}(\bm p))$, where $\hat{\bm r}(\cdot) = (\hat{r}_1 (\cdot), \cdots, \hat{r}_M (\cdot))$ is a vector of $M$ real-valued sub-property functions and $t$ is a ``link'' function that calculates the target property from them. For example, variance can be determined by the first two-order moments $(E_{\bm p}[X], E_{\bm p}[X^2])$ via the link function $t(\bm r) = r_2 - r_1^2$. Suppose that each sub-property is directly-elicitable, and let $L_i(r_i; \bm p)$ be a strictly proper sub-loss function w.r.t. the sub-property $\hat{r}_i$ for each $i$. Then the positive-weighted sum of sub-losses, denoted by 
\begin{equation} \label{Eq-total-loss}
\mathcal{L}_{\bm c} (\bm r; \bm p) = \sum_{i=1}^M c_i L_i (r_i; \bm p),
\end{equation} 
elicits the vector of sub-properties $\hat{\bm r} = (\hat{r}_1, \cdots, \hat{r}_M)$ and then indirectly infers the target property via the link function $t$, where $\bm c \in \mathbb{R}_{++}^M $,\footnote{$\mathbb{R}_{++}$ denotes the strictly positive reals}. We call such a total loss function {\em fully separable}.\footnote{ We pursue fully separable losses here in order to simplify the problem. If a sub-property vector is elicitable as a whole while there exist some not directly-elicitable components, generally the fully separable form of total losses do not exist. Thus we assume the direct-elicitability of each sub-property. }

There are an infinite number of options for the total loss $\mathcal{L}_{\bm c} (\bm r; \bm p)$ w.r.t. the sub-losses $\{L_i\}$ and the weights $\bm c$. The choice of the total loss does not matter in terms of properness, that is, strictly proper losses will always minimize the loss value at the precise prediction $\hat{\bm r}$. But normally different scoring rules would have different attitudes towards the imprecise predictions. Particularly, when restricting the predictions to a parametric space which generally does not cover the true value, different scoring rules would take different parametric (and imprecise) predictions as the best. 

Specifically, we apply a parametric model assumption to the probability distribution, i.e. $\bm q_{\bm \theta} \in \mathcal{D}_{\Theta} \subseteq \mathcal{D}$ with $\bm \theta \in \Theta$. For example, Gaussian distribution space can be represented as $\mathcal{D}_{\Theta} = \{ N(\mu, \sigma^2)| \bm \theta = (\mu, \sigma^2) \}$. This yields the corresponding parametric formula of sub-properties $r_i(\bm \theta) \triangleq \hat{r}_i(\bm q_{\bm \theta})$ for each $i$ and the target property $\gamma(\bm \theta) \triangleq \Gamma(\bm q_{\bm \theta}) = t(\bm r (\bm \theta))$. Given $\mathcal{L}_{\bm c}$, $\mathcal{D}_{\Theta}$, and $\bm p$, the best value of the parameter $\bm \theta$ (if existing) is
\begin{equation}\label{eq_parametric_model_optimization}
\bm \theta_{ \mathcal{L}_{\bm c} }^* (\bm p) \in \arg\min\limits_{\bm \theta \in \Theta} \mathcal{L}_{\bm c} (\bm r(\bm \theta); \bm p)
\end{equation}
If $\bm r(\bm \theta_{ \mathcal{L}_{\bm c} }^*) \neq \hat{\bm r} $, then possibly the choice of $\mathcal{L}_{\bm c}$ would affect the value of $\bm \theta_{ \mathcal{L}_{\bm c} }^*$ and consequently $\gamma(\bm \theta_{ \mathcal{L}_{\bm c} }^*)$. Since our task is to estimate the target property, we naturally use the quality of the parametric estimation for the target property as a guidance of the choice. The more precise the estimated value of the target property is, the better the loss function is. Furthermore, in this paper we will only discuss the choice of weights $\bm c$ with fixing all sub-losses $\{L_i\}$ and use $\bm \theta_{ \bm c }^* (\bm p)$ in replace of $\bm \theta_{\mathcal{L}_{\bm c} }^* (\bm p)$. We leave the exploration of alternative sub-losses to future work. In fact, we will see in our theoretical results that the choice of sub-losses does not affect our observation and conclusions. 

We summarize the preceding discussion of the core problem studied in this paper as follows. 


\begin{problem} \label{main_problem} 
    Consider a target property $\Gamma(\bm p) = t(\hat{\bm r}(\bm p))$ where $\hat{\bm r}(\bm p) \in \mathbb{R}^M$ is a vector of $M$ directly-elicitable sub-properties of $\bm p$ and $t$ is a link function. The total loss $\mathcal{L}_{\bm c} (\bm r; \bm p) = \sum_{i=1}^M c_i L_i (r_i; \bm p)$ elicits $\hat{\bm r}$ with $L_i(r_i; \bm p)$ eliciting $\hat{r}_i $ for each $i$ and $\bm c \in \bar{\mathbb{R}}_{+}^M$.\footnote{$ \bm c \in \bar{\mathbb{R}}_{+}^M$ means $0 \leq c_i \leq +\infty$ for all $i$. Here, we allow $c_i = 0$ or $+\infty$ since we want to check the limit situation $c_i \rightarrow 0$ or $+\infty$, although $0 < c_i < +\infty$ is required from the normal application perspective.} 
    Let $\mathcal{D}_{\Theta}$ a parametric probability space and $\bm q_{\bm \theta} \in \mathcal{D}_{\Theta}$. Correspondingly, $\bm r(\bm \theta) = \hat{\bm r}(\bm q_{\bm \theta}) \in \mathcal{R}_{\Theta}$, and $\gamma(\bm \theta) = t(\bm r (\bm \theta))$. Let $\bm \theta_{ \bm c }^* (\bm p) \in \arg\min\limits_{\bm \theta \in \Theta} \mathcal{L}_{\bm c} (\bm r(\bm \theta); \bm p)$. We aim to figure out for a given $\bm p$ how $\gamma (\bm \theta_{\bm c}^* (\bm p) )$ changes w.r.t. $c_i$ for each $i$ while fixing $c_{-i}$\footnote{ $c_{-i}$ refers to all $c_j$ with $j\neq i$. }, and what the best setting of the weight $c_i^*$ is to minimize $|\Gamma(\bm p) - \gamma (\bm \theta_{\bm c}^*(\bm p) )|$, and why.
\end{problem}


\begin{remark} \label{remark_parameterDimension}
    If $\hat{\bm r} \in \mathcal{R}_{\Theta}$, which is generally not true in practice, then the choice of $\bm c$ does not matter because every proper loss will be minimized at the ``true'' value whatever $\bm c$ is. To make the problem interesting, throughout the paper we assume that $dim(\Theta) = dim(\mathcal{R}_{\Theta}) = M-1$\footnote{$dim(\cdot)$ represents the intrinsic dimension or the number of degrees of freedom for parameters. For example, a curve on $\mathbb{R}^2$ has a dimension of $1$. }, in order to ensure that in general $\hat{\bm r} \notin \mathcal{R}_{\Theta}$. With $M=2$, $\theta$ should be a scalar, and then the image of $\mathcal{R}_{\Theta}$ is a curve in $\mathbb{R}^2$. For example, when eliciting the first two moments to infer variance, the assumption $\bm q_\theta \sim Poisson(\theta)$ leads to $\bm r (\theta) = (\theta, \theta + \theta^2)$, and then $\hat{\bm r} (\bm p) \in \mathbb{R}^2$ generated from data will be possibly located outside of the curve of $\bm r (\theta)$. In contrast, the assumption $\bm q_{\bm \theta} \sim N(\mu, \sigma^2)$ with $\bm \theta = (\mu, \sigma^2)$ leads to $\bm r(\bm \theta) = (\mu, \mu^2+\sigma^2)$ whose image is the area $\mathcal{R}_{\Theta} = \{ \bm r \in \mathbb{R}^2 | r_2 \geq r_1^2\}$, and $\hat{\bm r} (\bm p)$ generated from any valid distribution will fall into this area. Thus, in the latter situation, $\bm \theta_{\bm c}^* (\bm p) $ does not change with $\bm c$ for all given $\bm p$. Similarly, with $M=3$, $\bm \theta$ should be a two-dimensional vector and then $\mathcal{R}_{\Theta}$ is a surface in $\mathbb{R}^3$.
\end{remark}

\section{Simulation studies \& observations} \label{section-initialSimulations}

In the previous section, we present formally the mathematical model of our studied problem as Problem \ref{main_problem}, in which we use a fully separable weighted sum of proper losses to indirectly elicit a target property and we aim to figure out how the change of weights in the total loss affects the optimal parametric estimation of the target property. Now we start the discussion with designing simulation studies in order to obtain some guidance on the theoretical analysis.

As stated in Problem \ref{main_problem}, we will check for a given $\bm p$ how the optimal parametric estimation of the target property $\gamma (\bm \theta_{\bm c}^* (\bm p) )$ changes w.r.t. $c_i$ for each $i$ while fixing $c_{-i}$, and what $c_i^*$ is to make $\gamma (\bm \theta_{\bm c}^* (\bm p) )$ closest to the true value $\Gamma(\bm p)$. We will test different settings of $\Gamma$, $\mathcal{D}_{\Theta}$, and $\bm p$. 

We first try the variance property with $\Gamma(\cdot) = var(\cdot)$, whose sub-properties are moments $\hat{r}_i = E_{X \sim \bm p} [X^i]$. For the sub-losses $L_i(r_i; \bm p)$ for all $i$, we choose the same loss type, the squared loss, that is, $L_i(r_i; \bm p) = E_{X \sim \bm p} [(r_i - X^i)^2]$. We consider different settings for $\mathcal{D}_{\Theta}$ such as Poisson, Exponential, and Binomial distributions, see Table \ref{table-experiment-design-for-variance} in Appendix \ref{appendix-section-experimentsForVariance} for full detail. All losses would be calculated based on the empirical distribution $\hat{\bm p}$. To generate $\hat{\bm p}$, we first set a template distribution $\bm p_0$, and then use some random seed to randomly generate $1000$ samples from $\bm p_0$. For variance property, we choose $\bm p_0$ from Gaussian family since all the tested models can be approximated by a special case of Gaussian distributions and then the estimation task looks reasonable. Besides, we set $ c_{-i} = 1 $ by default. In fact, the setting of $ c_{-i}$ does not matter for our empirical observations and theoretical results. Finally, we show the plot of $\gamma (\bm \theta_{\bm c}^* (\hat{\bm p}) )$ w.r.t. $c_i$ for each $i$, which we also refer to as the $\Gamma_{c_i}$ curve.

We show major results for $\Gamma(\cdot) = var(\cdot)$ in Figure \ref{fig-property-c-curves-var}. Clearly, $\Gamma_{c_i}$ curves are monotone in all tested cases and the $c_i^*$ would be $0$ or $+\infty$. For more results with different empirical distribution $\hat{\bm p} $, please see Appendix \ref{appendix-section-experimentsForVariance}. 

Plus, we also do the similar experiments for the skewness property with $\Gamma(\cdot) = skew(\cdot)$ and some other model assumptions such as log-normal distributions, and then obtain the same pattern among most cases, even for $\bm p_0 \notin \mathcal{D}_{\Theta}$. We will see more detail in Section \ref{section-3D}. In all, the $\Gamma_{c_i}$ curves for different kinds of $\Gamma$, $\mathcal{D}_{\Theta}$, and $\bm p$, are monotone most of the time in our simulation studies. 

The monotonicity pattern of the $\Gamma_{c_i}$ curves looks surprising, but such a surprising phenomenon is so common in our simulation studies. We are naturally interested in why this monotonicity pattern happens so commonly and how it happens. In the following sections, we will uncover the theory behind this pattern.

\begin{figure}[!htbp]
  \centering
  \subfigure[$\bm q_{\theta} = Poisson(\theta)$]{\includegraphics[width=0.3\textwidth]{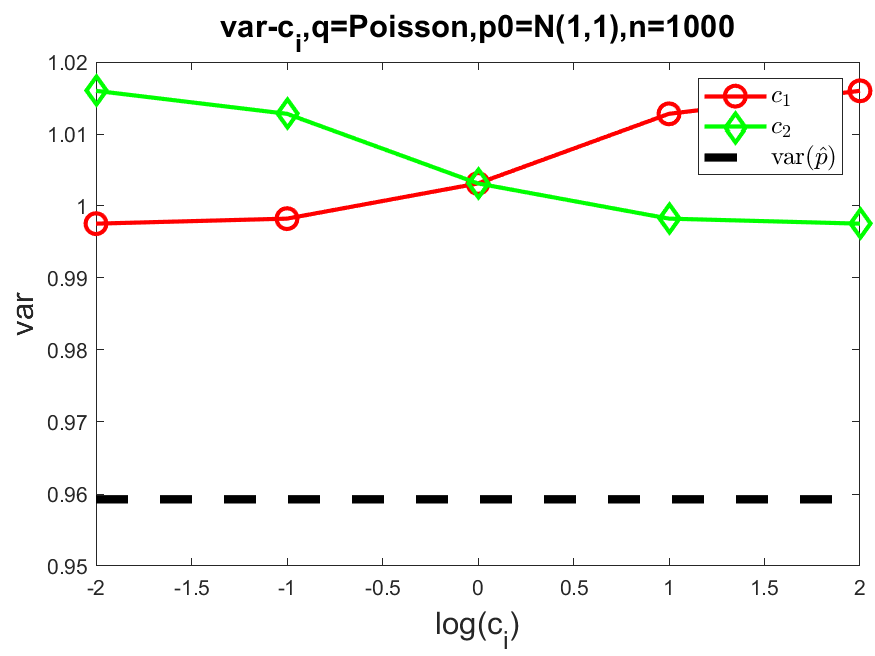}}
  \subfigure[$\bm q_{\theta} = \chi^2(\theta)$]{\includegraphics[width=0.3\textwidth]{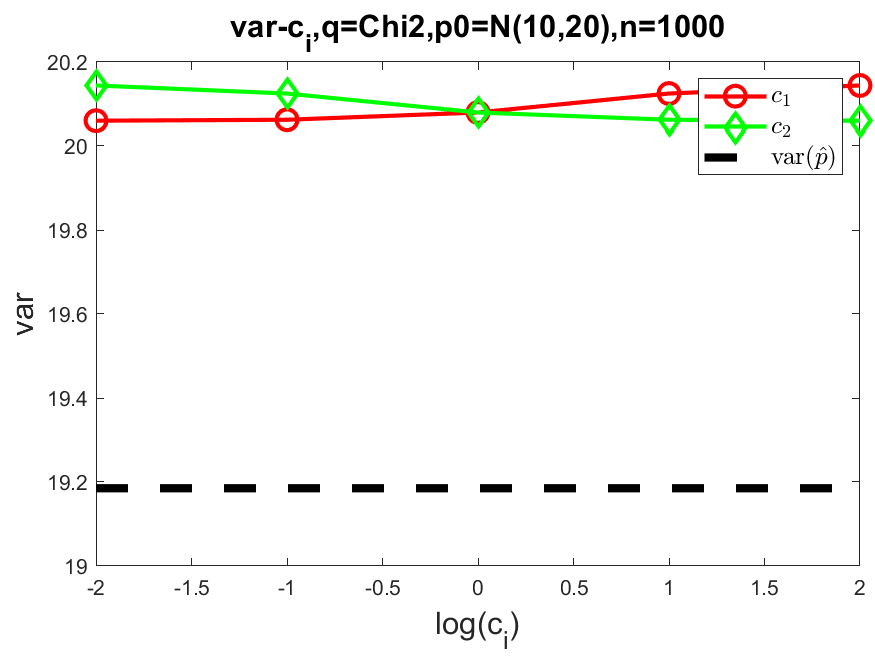}}
  \subfigure[$\bm q_{\theta} = Exp(1/\theta)$]{\includegraphics[width=0.3\textwidth]{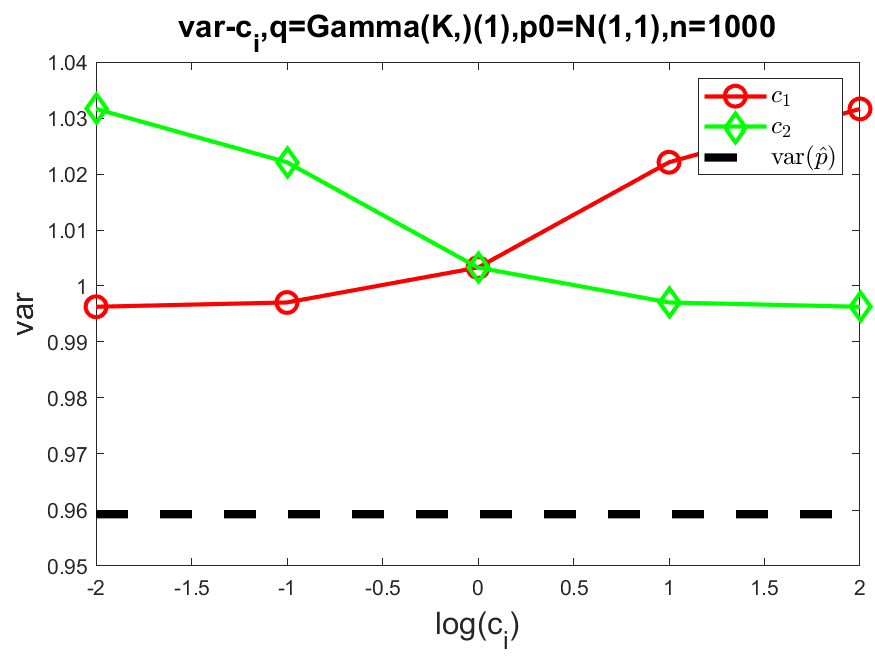}}
  \subfigure[$\bm q_{\theta} = Gamma(2, \theta)$]{\includegraphics[width=0.3\textwidth]{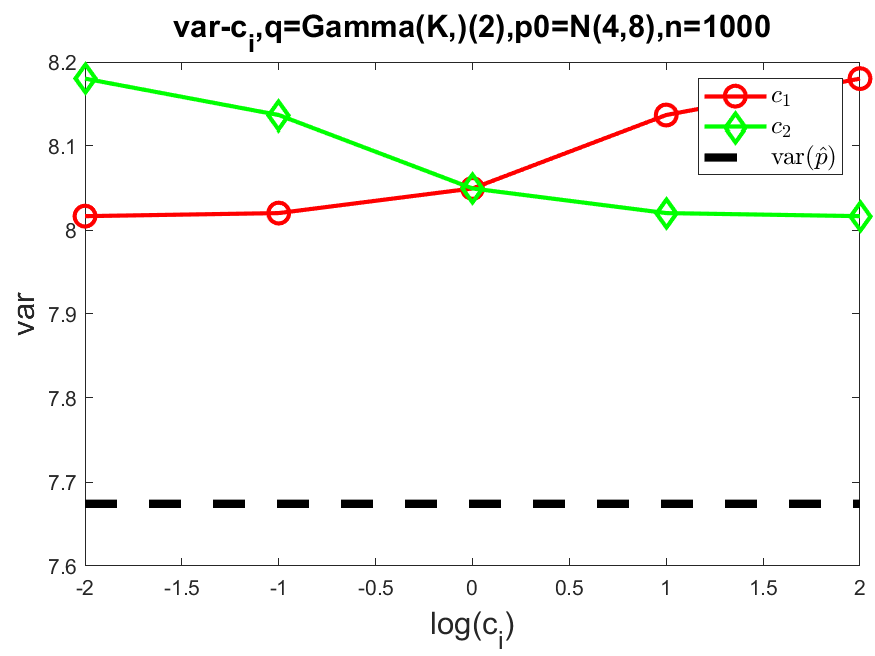}}
  \subfigure[$\bm q_{\theta} = Binomial(10, \theta)$]{\includegraphics[width=0.3\textwidth]{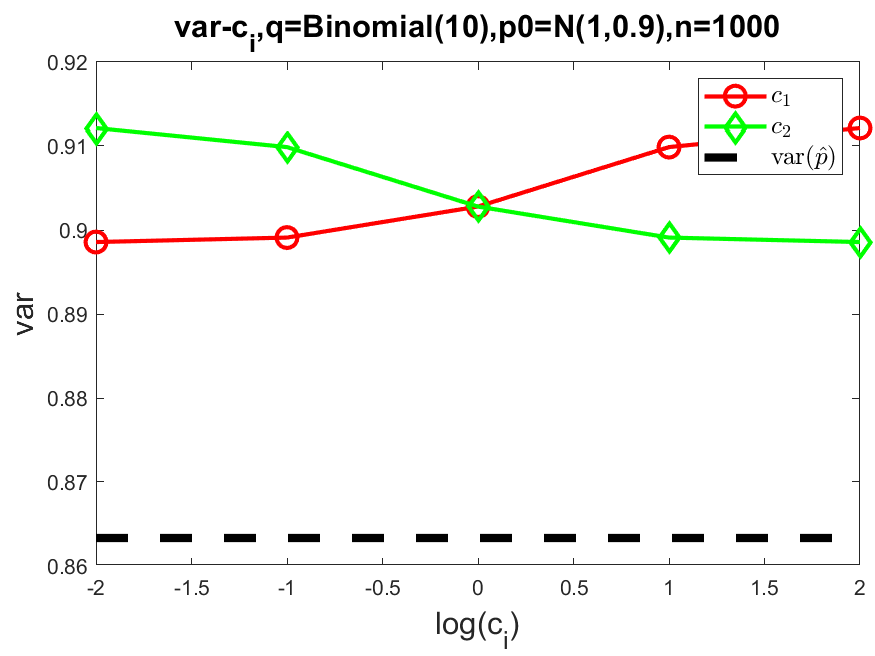}}
  
  \caption{ The $\Gamma_{c_i}$ curves for $\Gamma = var(\cdot)$ with different $\bm q_{\theta}$ and $\hat{\bm p}$. } 
  \label{fig-property-c-curves-var}

\end{figure}

\section{Elementary theoretical framework} \label{section-general-theory}

From the above simulation studies, we have observed that the $\Gamma_{c_i}$ curve is commonly monotone for each $i$, i.e., the estimation $\gamma (\bm \theta_{\bm c}^*) $ changes monotonically as increasing $c_i$. To explain why the monotonicity pattern of $\gamma (\bm \theta_{\bm c}^*) $ with $c_i$ happens commonly and how, in this section we will construct elementary theoretical framework for general situations, in which deeper analysis of special situations can reside as we will see in later sections. 

To analyze the change of $\gamma (\bm \theta_{\bm c}^*) $ with $\bm c$, a naive way is to first analyze how $\bm \theta_{\bm c}^*$ changes with $\bm c$ and then to calculate $\gamma (\bm \theta_{\bm c}^*) $. But the analysis of $\bm \theta_{\bm c}^*$ is not straightforward at all, since $\bm \theta_{\bm c}^*$ relies on the optimization procedure involving specific settings of $\{L_i\}$, $\mathcal{D}_{\Theta}$, and $\bm p$. To circumvent the barrier and simplify the analysis, we come up with an alternative. A key insight is that $\gamma  $ is linked to $\bm \theta$ via $\gamma = t(\bm r)$ and $\bm r = \bm r (\bm \theta)$, and intuitively increasing $c_i$ should lead to a more accurate estimation of $\hat{r}_i$. So, the analysis of $\gamma (\bm \theta_{\bm c}^*) $ can be decomposed into two parts: how $\bm r(\bm \theta_{\bm c}^*)$ moves with $c$ in the space $\mathcal{R}_{\Theta}$ and how $\gamma = t(\bm r ) $ changes with $\bm r$. Obviously, if both changes are monotone, then the composed change on $\gamma$ will be also monotone. Such a decomposition trick shows a much more concise profile of the analysis with skipping the intermediate step of the change of $\bm \theta^*_{\bm c}$ in the parameter space $\Theta$. 


We begin with the following definitions about the movement of $\bm r(\bm \theta_{\bm c}^*)$ with $c$ and the change of $\gamma = t(\bm r ) $ with $\bm r$. 

\begin{definition}[A trajectory of sub-properties] \label{condition-subProperties}
    In Problem \ref{main_problem}, define a trajectory of sub-properties in $\mathcal{R}_{\Theta}$ to be a one-dimensional curve $\mathcal{T}_{x} = \{ \bm r (x) \in \mathcal{R}_{\Theta} | x\in I \subseteq \mathbb{R} \}$ in $\mathcal{R}_{\Theta}$ where $I$ is the index set. E.g., the points of $\bm r(\bm \theta^*_{\bm c}(\bm p))$ with increasing $c_i$ while fixing $c_{-i}$ is a trajectory denoted by $\mathcal{T}_{c_i} (\mathcal{L}_{\bm c}, c_{-i}, \bm p) = \{\bm r (c_i) = \bm r(\bm \theta^*_{\bm c}(\bm p)) | c_i\in \mathbb{R}_{+} \}$ or $\mathcal{T}_{c_i} (\bm p)$ for short, which means that the trajectory relies on $\bm p$ and potentially $\mathcal{L}_{\bm c}$ and $c_{-i}$.  
     
\end{definition}

\begin{definition}[One-sided trajectory]
     In Problem \ref{main_problem}, given $\tilde{\bm r} \in \mathbb{R}^M $, a trajectory of sub-properties $\mathcal{T}_{x} $ is called one-sided from $\tilde{r}_i$ for some $i$ if either $r_i (x) \leq \tilde{r}_i $ for all $x$ or $r_i(x) \geq \tilde{r}_i $ for all $x$.
\end{definition}  

\begin{definition}[Monotone function along a trajectory]
     In Problem \ref{main_problem}, a function $f(\bm r)$ is monotone along a trajectory of sub-properties $\mathcal{T}_{x} $ in $r_i$ for some $i$, if for all $\bm r^A \in \mathcal{T}_{x} $ and $\bm r^B \in \mathcal{T}_{x} $ with $r^A_i < r^B_i$, $f(\bm r^A) \leq f(\bm r^B)$ (non-increasing) or $f(\bm r^A) \geq f(\bm r^B)$ (non-decreasing).
\end{definition}

We also need the following intuitive but important lemma, which can be used to clarify why increasing $c_i$ would always lead to a more accurate estimation of $\hat{r}_i$.

\begin{lemma}\label{lemma-increaseC-decreaseFunction}
Suppose $f_1: \Theta \rightarrow \mathbb{R}$, $f_2: \Theta \rightarrow \mathbb{R}$. Let $F_{\bm c}(\bm \theta) = c_1 f_1(\bm \theta)+c_2 f_2(\bm \theta)$ with $\bm c = (c_1, c_2) \in \mathbb{R}^2$. Suppose the minimizer of $F_{\bm c}$ exists for all $\bm c$, and denote $\bm \theta_{\bm c}^* \in \arg\min_{\bm \theta} F_{\bm c}(\bm \theta)$. If we change $\bm c$ to $\tilde{\bm c} = (c_1 + \Delta c_1, c_2)$ with $\Delta c_1> 0$, then we have $f_1(\bm \theta^*_{\tilde{\bm c}}) \leq f_1(\bm \theta_{\bm c}^*)$.
\end{lemma} 

\begin{proof} 
Notice that $F_{\tilde{\bm c}}(\bm \theta) = \Delta c_1 f_1(\bm \theta) + F_{\bm c}(\bm \theta)$. Because $\bm \theta_{\tilde{\bm c}}^* \in \arg\min_{\bm \theta} F_{\tilde{\bm c}}(\bm \theta)$, then $ F_{\tilde{\bm c}}(\bm \theta_{\tilde{\bm c}}^*) \leq F_{\tilde{\bm c}}(\bm \theta_{\bm c}^*)$, i.e., 
\[
\Delta c_1 f_1(\bm \theta_{\tilde{\bm c}}^*) + F_{\bm c}(\bm \theta_{\tilde{\bm c}}^*) 
\leq 
\Delta c_1 f_1(\bm \theta_{\bm c}^*) + F_{\bm c}(\bm \theta_{\bm c}^*).
\]
Since $\bm \theta_{\bm c}^* \in \arg\min_{\bm \theta} F_{\bm c} (\bm \theta)$, $F_{\bm c}(\bm \theta_{\tilde{\bm c}}^*) \geq F_{\bm c}(\bm \theta_{\bm c}^*)$. Then $f_1(\bm \theta^*_{\tilde{\bm c}}) \leq f_1(\bm \theta_{\bm c}^*)$.
\end{proof}


Now, we can establish the elementary conditions for the sub-property trajectory $\mathcal{T}_{c_i} $ and the link function $\gamma = t(\bm r ) $ such that both of them are "monotone" and thus the composite result $\gamma (\bm \theta_{\bm c}^*) = t( r(\bm \theta_{\bm c}^*) )$ would be naturally monotone.

\begin{theorem}\label{theorem-general}
For Problem \ref{main_problem}, given $\bm p$, suppose that $\theta^*_{\bm c}(\bm p) $ exists for all $\bm c$. For some $i$ and fixed $c_{-i}$, we have the following statements.
\begin{itemize}
    \item[(A)] If the sub-loss $L_i$ is accuracy-rewarding, and the trajectory $\mathcal{T}_{c_i} (\bm p))$ is one-sided from $\hat{r}_i(\bm p)$, then $r_i(\bm \theta^*_{\bm c} (\bm p) )$ increases or decreases to $\hat{r}_i (\bm p) $ with increasing $c_i$.
    \item[(B)] Based on (A), furthermore $ \gamma (\bm \theta^*_{\bm c}(\bm p) ) $ is monotone in $c_i$ if and only if the link function $\gamma = t(\bm r)$ is monotone along the trajectory $\mathcal{T}_{c_i} (\bm p) $ in $r_i$.
\end{itemize}

\end{theorem}

\begin{proof}
    Lemma \ref{lemma-increaseC-decreaseFunction} shows that increasing $c_i$ would definitely incur a smaller or equal value of $L_i (r_i(\bm \theta_{\bm c}^*); \bm p)$ , no matter what the sub-loss is. If (A) the trajectory $\mathcal{T}_{c_i} (\bm p)$ is one-sided and also the sub-loss $L_i$ is accuracy-rewarding, according to Definition \ref{def-accuracyRewarding} we know that $r_i(\bm \theta_{\bm c}^*)$ increases or decreases to $\hat{r}_i$ with increasing $c_i$ and fixed $c_{-i}$. (B) is straightforward and thus the proof is omitted here. 
\end{proof}

Theorem \ref{theorem-general} provides an intuitive theoretical framework. It points out the nature of how the monotonicity pattern of $ \gamma (\bm \theta^*_{\bm c} ) $ with $c_i$ would happen. It would be interesting to scrutinize why the two intuitive conditions (A) and (B) are commonly satisfied by the models used in our simulation studies in a large diversity. More importantly, the two conditions do not provide the corresponding characterization of the model structure, and then we are unable to verify them directly. So, the remaining work is to dig the deeper assumptions about when the conditions (A) and (B) in Theorem \ref{theorem-general} can be satisfied. In fact, as we will see later, these conditions are not as strong as they look like and the deeper assumptions behind them are mild. We will start the exploration in the case with two sub-properties and then move to higher-dimensional cases. Also, we will connect the theory back to the simulated results shown in Section \ref{section-initialSimulations}.

\section{Extensive analysis for 2-D cases}\label{section-2D}

Theorem \ref{theorem-general} describes two sufficient conditions (A) and (B) to support the monotonicity pattern of $\gamma (\bm \theta_{\bm c}^*)$ in $\bm c$. 
The condition (A) guarantees that the sub-property trajectory $\mathcal{T}_{c_i}$ goes monotonically to the true value $\hat{r}_i$ with increasing $c_i$, and the condition (B) tells us that $\gamma$ changes monotonically along $\mathcal{T}_{c_i}$. As we mentioned in the end of the previous section, we need to furthermore uncover the deeper assumptions behind these conditions. These assumptions should expose the desirable structure of the parametric sub-property model and the target property such that the conditions (A) and (B) in Theorem \ref{theorem-general} can be satisfied. In this section, we will discuss the 2-D situation, where there are only two sub-properties $M=2$. In the end, we will establish the correspondence between theory and experiments. 

To begin with, we need to specify several new notations. First, as specified in Remark \ref{remark_parameterDimension}, when $M=2$, the parameter $\theta \in \Theta \subseteq \mathbb{R}$ (thus not denoted in bold font) and correspondingly the image of $\mathcal{R}_\Theta$ would be a curve on $\mathbb{R}^2$ such that $\hat{\bm r} \notin \mathcal{R}_{\Theta}$ in general. Second, if possible, $\bm r(\theta)= (r_1(\theta), r_2(\theta))$ will be equivalently denoted by a function $r_2 (\theta) = R(r_1(\theta))$ w.r.t. $r_1$ for convenience. Similarly, we denote the contour of the target property $\{\bm r | t(\bm r) = t_0\}$ at any level $t_0$ by a function $ r_2 = T(r_1;t_0)$ if possible.

\subsection{Assumption for the parametric sub-properties} \label{section-2D-conditionForSubproperties}

As we have seen in Theorem \ref{theorem-general}, the condition (A) requires the sub-property trajectory $\mathcal{T}_{c_i}$ to be one-sided from $\hat{r}_i$ for some $i$ and $c_{-i}$. Now we discuss what structure of the parametric sub-property model $\mathcal{R}_{\Theta}$ can satisfy the condition (A), that is, ensuring $\mathcal{T}_{c_i}$ one-sided from $\hat{r}_i$.

We first have the following Theorem \ref{theorem-conditionForSubproperties-2D} for the case of strictly monotone $\mathcal{R}_{\Theta}$. It is necessary for readers to read through the proof of Theorem \ref{theorem-conditionForSubproperties-2D} to shape an intuition about how the mechanism works. 

\begin{figure}[!htbp]
  \centering
  \includegraphics[width=0.45\textwidth]{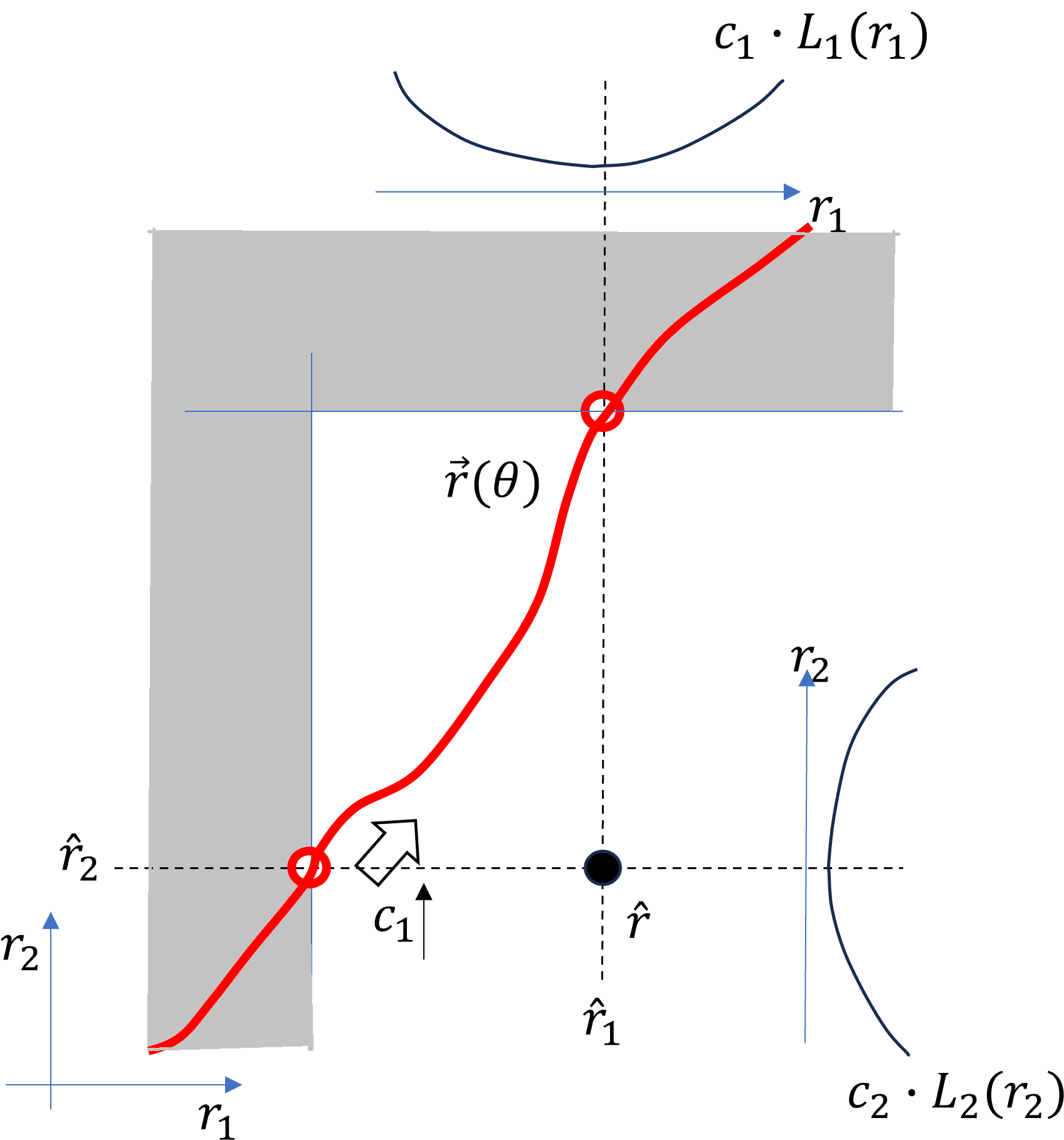}
  
  \caption{ The impossible area of $\bm r(\theta^*_{\bm c})$ for all $\bm c$ with a strictly monotone $\mathcal{R}_{\Theta}$ in a 2-D space, marked with gray shading. } 
  \label{fig_2D_model_monotone_illustration} 

\end{figure}

\begin{theorem}\label{theorem-conditionForSubproperties-2D}
For Problem \ref{main_problem} with $M=2$, suppose that $\theta^*_{\bm c}(\bm p) $ exists for all $\bm p$ and $\bm c$ and all sub-losses are accuracy-rewarding. If $\bm r(\theta)$ characterizes a strictly monotone function $r_2 = R(r_1)$ \footnote{In this paper, the "strictly monotone" function only requires the inequality $f(x_1) < f(x_2)$ or conversely for all $x_1 < x_2$. The continuity is not required. }, then $\mathcal{T}_{c_i} (\bm p)$ is one-sided from $\hat{r}_i (\bm p)$ for each $i$ and all $c_{-i}$, that is, $\mathcal{T}_{c_i} (\bm p)$ is in the same quadrant centered at $\hat{\bm r} (\bm p) $ for all $i$ and $c_{-i}$. Thus, the condition (A) in Theorem \ref{theorem-general} is satisfied, and $r_i(\bm \theta^*_{\bm c} (\bm p) )$ moves monotonically to $\hat{r}_i (\bm p) $ with increasing $c_i$ from $0$ to $+ \infty$ for each $i$.
\end{theorem}

\begin{proof}
    Given any $\bm p$, if $\hat{r}(\bm p) \in \mathcal{R}_\Theta$, then $\bm r(\theta^*_{\bm c})\equiv \hat{\bm r}$ for all $\bm c$ and the conclusion holds. Next, let us consider the non-trivial case with $\hat{r}(\bm p) \notin \mathcal{R}_\Theta$.
    
    Without loss of generality, assume that $r_2 = R(r_1)$ is strictly increasing, as shown in Figure \ref{fig_2D_model_monotone_illustration}. There are two intersection points between the model curve $\bm r(\theta)$ and the coordinate axes centered at $\hat{\bm r}$, which are marked in red circles. Notice that the outer shaded area can be ruled out for the candidates of $\bm r(\theta^*_{\bm c})$ for all $\bm c$, which means that the possible area of $\bm r(\theta^*_{\bm c})$ is intrinsically constrained. The reason is as follows. 
    
    The top intersection corresponds to a precise estimation of $\hat{r}_1$. Because all sub-losses are accuracy-rewarding, all points $\bm r$ in the top shaded area except the top intersection incur a higher loss value for $L_1$ and/or for $L_2$ and thus they could never be an option for $\bm r(\theta^*_{\bm c})$ no matter what $\bm c$ is. That is, the top shaded area is {\bf dominated} by the top intersection. Similarly, the left shaded area is dominated by the left intersection. So, $\bm r(\theta^*_{\bm c})$ would only be located between the two intersection points (included) for all $\bm c$. 
    
    Thus, $\mathcal{T}_{c_i} (\bm p)$ is in the same quadrant centered at $\hat{\bm r} (\bm p) $ for all $i$ and $c_{-i}$. It is easy to verify that the conclusion is always true no matter where $\hat{r}(\bm p)$ is (but may be in a different quadrant of a different $\hat{r}(\bm p)$). 

    In Figure \ref{fig_2D_model_monotone_illustration}, the left intersection corresponds to the worst estimation of $\hat{r}_1$ but a precise estimation of $\hat{r}_2$ which can be achieved by setting $c_1=0$, while the top intersection corresponds to the precise estimation of $\hat{r}_1$ that can be achieved by setting $c_1= + \infty$. Following from the condition (A) in Theorem \ref{theorem-general}, we see that $\bm r(\theta^*_{\bm c})$ moves monotonically from the left intersection to the top with increasing $c_1$ from $0$ to $+\infty$ and moves conversely with increasing $c_2$. 
\end{proof}

\begin{remark}
    Theorem \ref{theorem-conditionForSubproperties-2D} discusses the special case where $\bm r(\theta)$ corresponds to a strictly monotone function. One may be curious about whether the strict monotonicity of $\bm r(\theta)$ is a necessary assumption to satisfy the condition (A). The answer is no. Basically, as illustrated by the proof of Theorem \ref{theorem-conditionForSubproperties-2D} with Figure \ref{fig_2D_model_monotone_illustration}, only the part of $\bm r(\theta)$ inside the white possible area matters and we do not need to care about what $\bm r(\theta)$ looks like in the shaded impossible area. What's more, even if the part of $\bm r(\theta)$ inside the possible area is not monotone, the condition (A) can still be satisfied in many cases. Nonetheless, with exploring a more general example of $\bm r(\theta)$, we find that the possible area for candidates of $\bm r(\theta^*_{\bm c})$ for all $\bm c$ would be composed of strictly monotone pieces within each quadrant centered at $\hat{\bm r} (\bm p)$ and it suffices to consider each quadrant at a time and then consider the results in all quadrants together. We leave detail to Appendix \ref{appendix-generalSubproperty-2D}.
\end{remark}

\subsection{Assumption for the target property} \label{section-2D-conditionForTargetProperty}

We have seen above that $\bm r(\theta)$ corresponding to a strictly monotone function can guarantee the condition (A). Next we will discuss what structure of the link function $\gamma = t(\bm r)$ can satisfy the condition (B) in Theorem \ref{theorem-general}, that is, the function $\gamma = t(\bm r)$ is monotone along $\mathcal{T}_{c_i} $ in $r_i$ for some $i$ and $c_{-i}$. We have the following theorem. Basically, it reveals that the relationship between the derivative of $\bm r(\theta)$ and the derivative of the contours of $\gamma = t(\bm r)$ matters. 

\begin{theorem}\label{theorem-conditionForTargetProperty} 
    For Problem \ref{main_problem} with $M=2$, suppose that $\bm r (\theta)$ corresponds to a differentiable function $r_2 = R(r_1)$, and $\gamma = t(\bm r)$ is continuous and partially differentiable. Also assume that all non-empty contours $T(r_1; t_0)$ are differentiable and monotone. We have $\gamma = t(\bm r)$ is monotone along $\bm r (\theta)$ if and only if $R'(r_1)- T'(r_1; t_0)$ keeps the sign unchanged ($0$ included) at all points where $R(r_1) = T(r_1; t_0)$.  
\end{theorem}

\textit{Sketch of proof.} Continuous $\gamma = t(\bm r)$ and monotone $T(r_1; t_0)$ for all $t_0$ guarantees that the sign of $\frac{\partial t}{\partial r_2}$ keeps unchanged ($0$ included). Plus the assumption that $R'(r_1)- T'(r_1; t_0)$ keeps the sign unchanged, then we can prove that $\gamma = t(r_1, r_2)$ is monotone along the curve $R(r_1)$ with taking the derivative of $t(r_1, R(r_1))$ w.r.t. $r_1$. Thus, the condition (B) is satisfied. The complete proof is given in Appendix \ref{appendix-proofs-2D-general}.

\subsection{General theory for 2-D cases}

Now we can integrate the discovered assumptions above about the structure of $\mathcal{R}_{\Theta}$ and $\gamma = t (\bm r) $ into the framework of Theorem \ref{theorem-general}, and then obtain the complete theory for the 2-D cases as follows. 

\begin{theorem} \label{theorem-2D-general}
    For Problem \ref{main_problem} with $M=2$, suppose that $\theta^*_{\bm c}(\bm p) $ exists for all $\bm p$ and $\bm c$. We assume that: (i) all sub-losses are accuracy-rewarding, and $\bm r(\theta)$ or equivalently $r_2 = R(r_1)$ is differentiable and strictly monotone; (ii) $\gamma = t(\bm r)$ is continuous and partially differentiable, and all non-empty contours $T(r_1; t_0)$ are differentiable and monotone. Then we have the following statements.
    \begin{itemize}
        \item[(a)] If $R'(r_1) $ and $ T'(r_1; t_0)$ have the same sign and $R'(r_1) < T'(r_1; t_0)$ for all $r_1$ and $t_0$, then increasing $c_1$ makes $\gamma (\theta_{\bm c}^*(\bm p) )$ move closer to $\Gamma (\bm p) $ for all $\bm p$, and $c_1^* = +\infty$. 
        
        \item[(b)] If $R'(r_1) $ and $ T'(r_1; t_0)$ have the same sign and $R'(r_1) > T'(r_1; t_0)$ for all $r_1$ and $t_0$, then increasing $c_1$ makes $\gamma (\theta_{\bm c}^*(\bm p) )$ move farther away from $\Gamma (\bm p) $ for all $\bm p$, and $c_1^* = 0$.
        \item[(c)] If $R'(r_1)$ and $T'(r_1; t_0)$ have different signs for all $r_1$ and $t_0$, then increasing $c_1$ makes $\gamma (\theta_{\bm c}^*(\bm p) )$ move first closer to $\Gamma (\bm p) $ and then farther away from $\Gamma (\bm p) $ for all $\bm p$, and $c_1^* \in (0,+\infty)$.
    \end{itemize}
\end{theorem}

\textit{Sketch of the proof.} The assumption (i) makes the condition (A) satisfied for all $\bm p$ according to Theorem \ref{theorem-conditionForSubproperties-2D}. Obviously, each of the condition of (a) (b) (c) can guarantee that $R'(r_1)- T'(r_1; t_0)$ keeps the sign unchanged. Plus the assumption (ii), we can prove that the condition (B) is satisfied according to Theorem \ref{theorem-conditionForTargetProperty}. Following from Theorem \ref{theorem-general}, we see that $ \gamma (\theta^*_{\bm c}(\bm p))$ is monotone in $c_1$ with fixed $c_2$, for all $\bm p$. Furthermore, each case of (a) (b) (c) involves one possible order among the values of $t$ at $ \bm r (\theta^*_{\bm c}) $ with $c_1=0$, at $ \bm r (\theta^*_{\bm c}) $ with $c_1=+\infty$, and at $\hat{\bm r}$. Then, the final conclusion holds. The complete proof is given in Appendix \ref{appendix-proofs-2D-general}. 

We give the illustrative example for each case of (a) (b) (c) below in Figure \ref{fig_dR_dT_consistent} to help readers understand the theory. In the case (a), for the $\hat{r}(\bm p)$ located at the bottom right corner, increasing $c_1$ will make $\bm r(\theta^*_{\bm c})$ move from the bottom left to the top right, and consequently $\gamma (\theta^*_{\bm c}) = t( \bm r(\theta^*_{\bm c}) )$ moves towards the best. In the case (b), increasing $c_1$ makes $\gamma (\theta^*_{\bm c}) $ move far away from the best. In the case (c), increasing $c_1$ makes $\gamma (\theta^*_{\bm c}) $ move first towards the best and then far away from the best.

\begin{figure}[!htbp]
  \centering
  \subfigure[]{\includegraphics[width=0.3\textwidth]{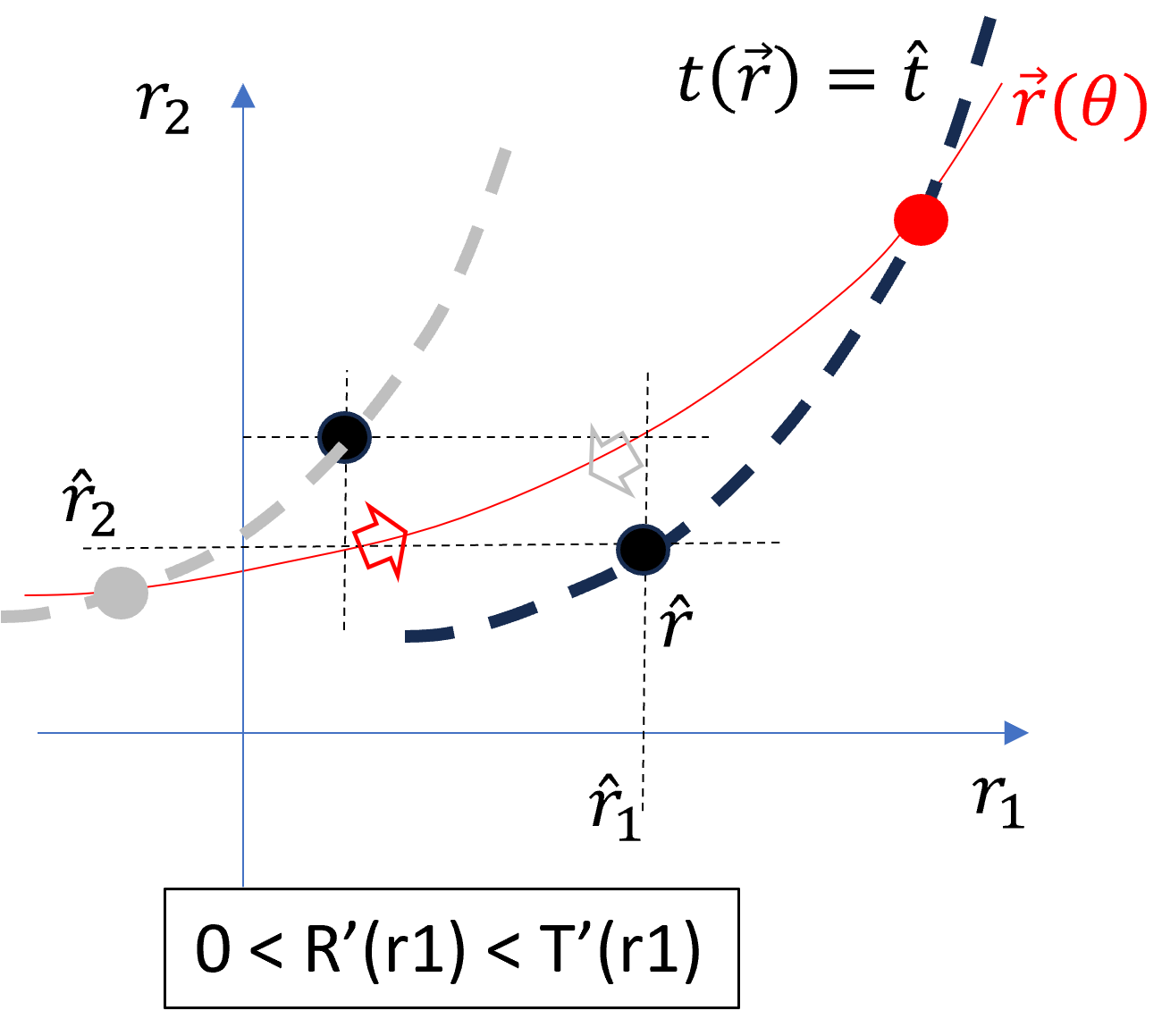}}
  \subfigure[]{\includegraphics[width=0.3\textwidth]{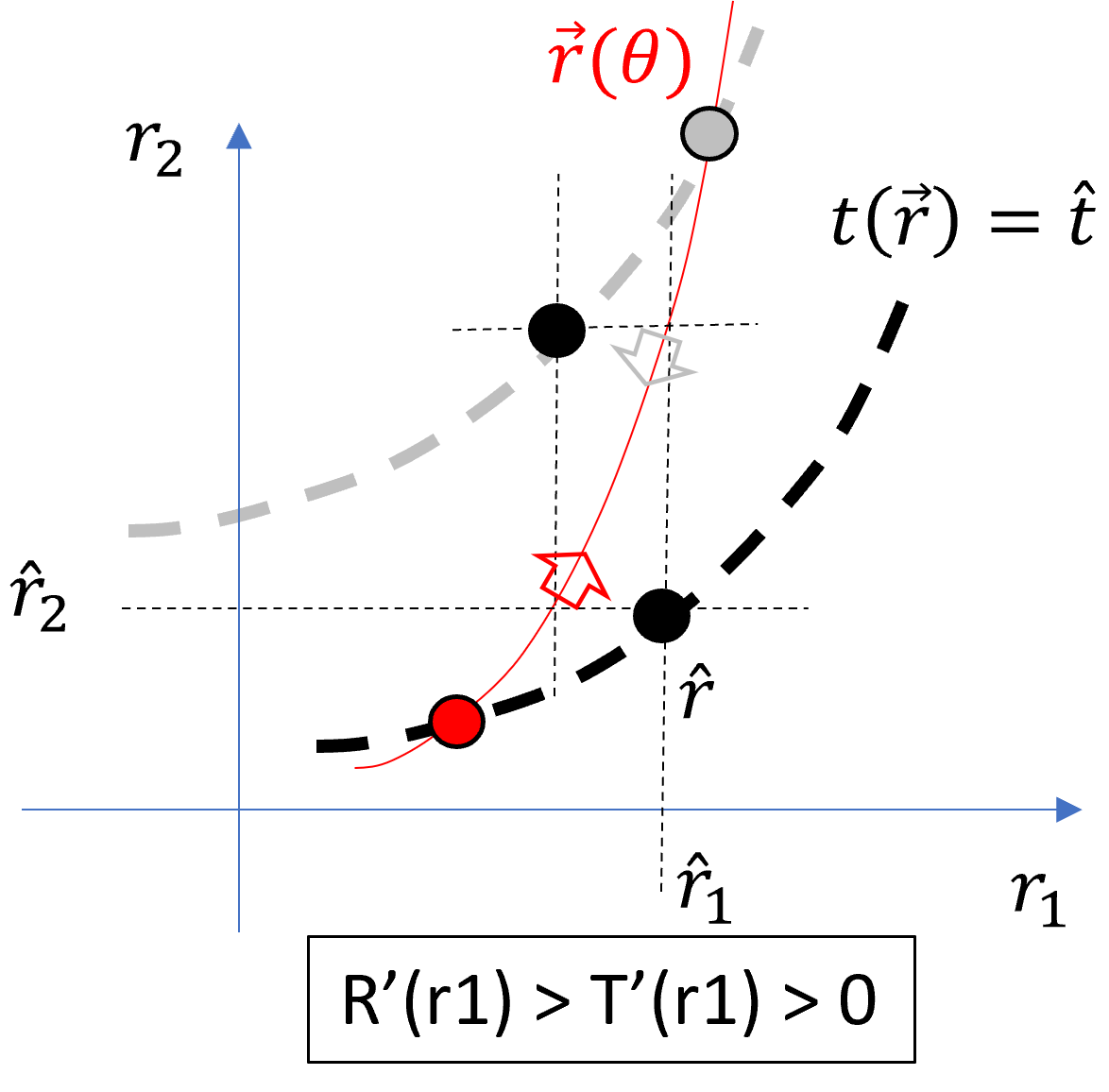}}
  \subfigure[]{\includegraphics[width=0.3\textwidth]{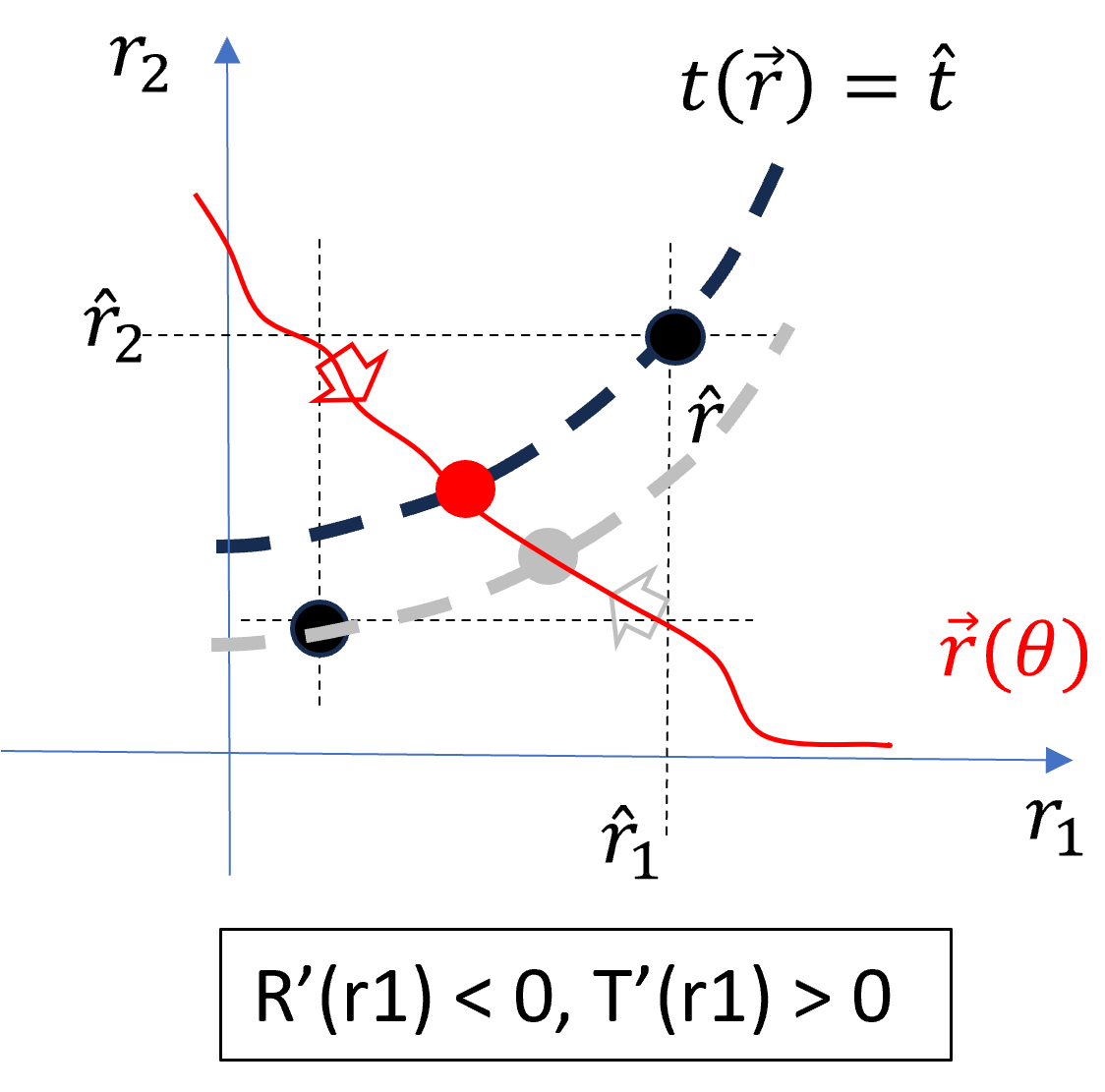}}
  \caption{ Example illustrations for Theorem \ref{theorem-2D-general}, where each sub-figure corresponds to each sub-case specified in the theorem. Recall that here $\bm r(\theta)= (r_1(\theta), r_2(\theta))$ is equivalent to a function $r_2 = R(r_1)$ and the contour $\{\bm r | t(\bm r) = t_0\}$ at any level $t_0$ is denoted by a function $ T(r_1;t_0)$. In each sub-figure, we plot two different $\hat{r}(\bm p)$, but it is easy to verify that the location of $\hat{r}(\bm p)$ does not affect the analysis.} 
  \label{fig_dR_dT_consistent}

\end{figure}

\begin{remark}
    First, for $M=2$, if $c_1^* = 0$ for some $c_2$, then $c_1^* = 0$ for all $c_2$ and $c_2^* = +\infty$ for all $c_1$. It is because the total loss $\mathcal{L}_{\bm c}$ is equivalent to the normalized loss 
    $
        \frac{1}{c_1 + c_2} \mathcal{L}_{\bm c} = \frac{c_1}{c_1 + c_2} L_1 + (1- \frac{c_1}{c_1 + c_2}) L_2
    $
    in the sense that they always have the same minimizer. Thus, the value of $c_1$ intrinsically determines the normalized value of $c_2$. But for $M>2$, it is not guaranteed. Second, for the case (c), notice that the true value of the target property $\Gamma(\bm p) = t(\hat{\bm r}(\bm p) )$ may not be reachable by $c_i^*$ under the current assumptions.  
\end{remark}

\subsection{Correspondence to empirical results} \label{section-2D-empiricalAnalysis}

Theorem \ref{theorem-2D-general} falls into the framework described in Theorem \ref{theorem-general}, and specifies what structure of the parametric sub-property model $\bm r(\theta)$ and the target property link function $\gamma = t(\bm r) $ can satisfy the condition (A) and (B) in Theorem \ref{theorem-general}. Now we connect the theory to the empirical results shown in simulation studies to see how the theorem works in practice. 

In Section \ref{section-initialSimulations}, for the variance target property, we have observed that all $\Gamma_{c_i}$ curves, that are, the image of $\gamma (\theta_{\bm c}^*)$, are monotone among different tested models of distributions. Furthermore, the value of the best weight $c_i^*$ is $0$ or $+\infty$ for each $i$. Now, we plot the figure for the contours of the target property link function and also the parametric model of sub-properties in Figure \ref{fig-var-for-q-model}, and see how the figures for different parametric models match to one of the (a) (b) (c) cases described in Theorem \ref{theorem-2D-general}.

\begin{figure}[!htbp]
  \centering
  \subfigure[$ Poisson(\theta) \approx N(\theta, \theta)$]{\includegraphics[width=0.3\textwidth]{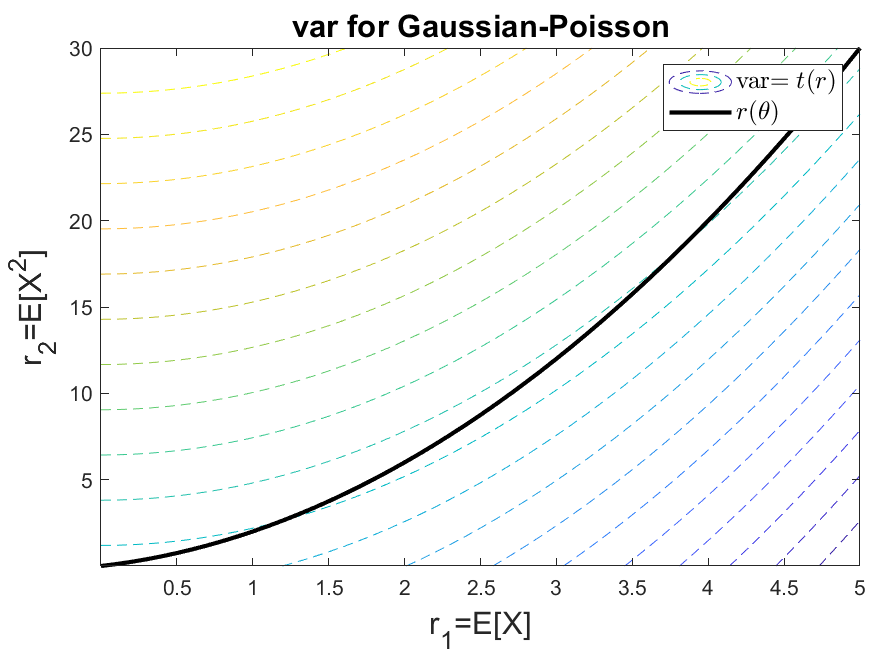}}
  \subfigure[$ \chi^2 (\theta) \approx N(\theta, 2 \theta)$]{\includegraphics[width=0.3\textwidth]{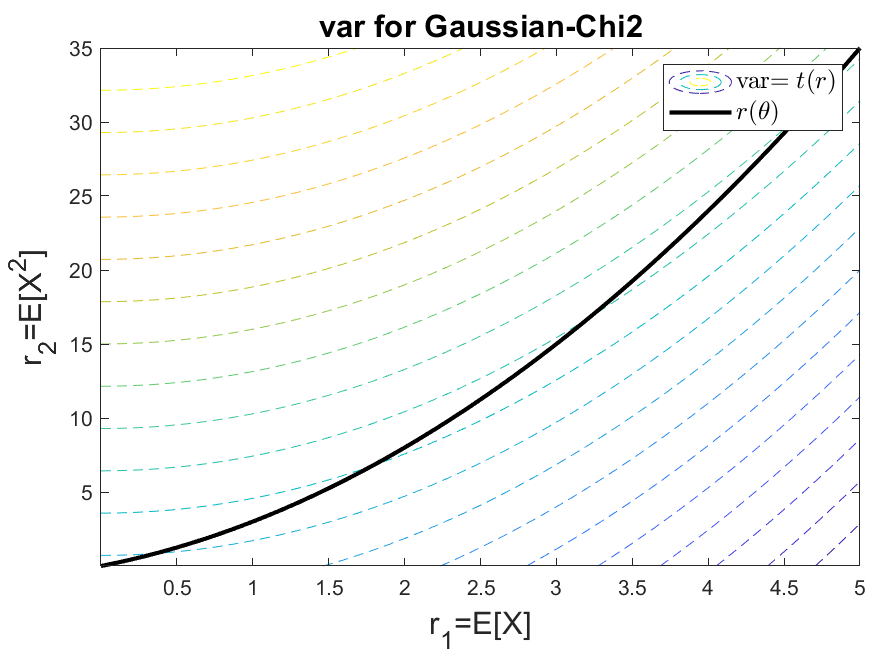}}
  \subfigure[$ Exp(\frac{1}{\theta}) \approx N(\theta, \theta^2 )$]{\includegraphics[width=0.3\textwidth]{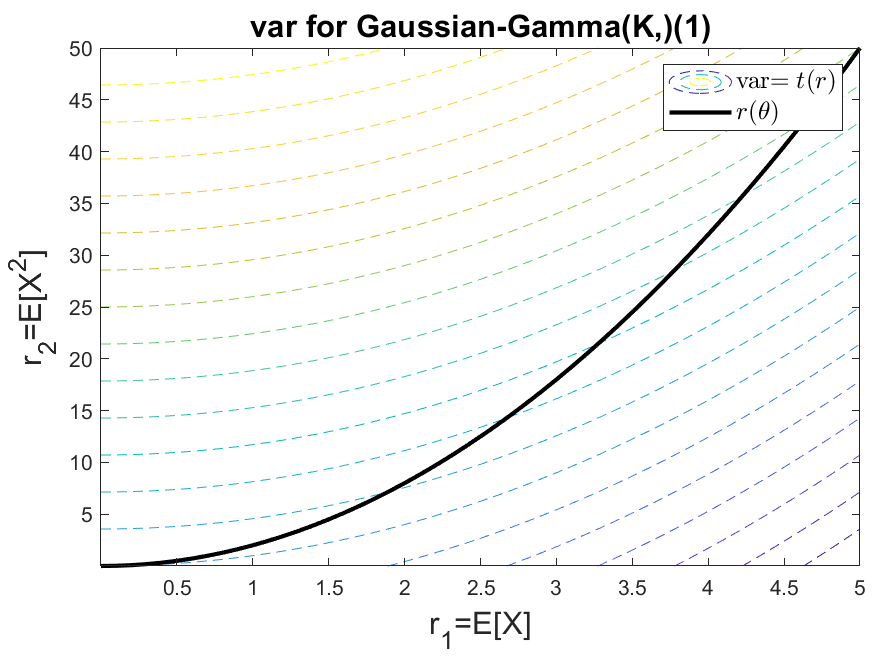}}
  \subfigure[$ Gamma(2, \theta) \approx  N(2 \theta, 2 \theta^2 )$]{\includegraphics[width=0.3\textwidth]{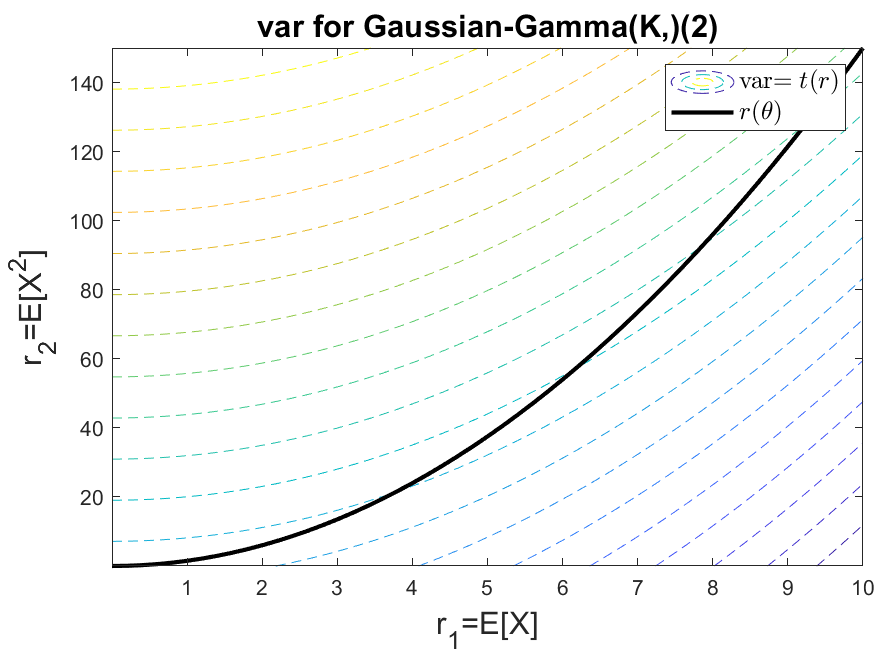}}
  \subfigure[$ B(1,\theta) $ $\approx$ $ N(\theta, \theta (1-\theta) )$]{\includegraphics[width=0.3\textwidth]{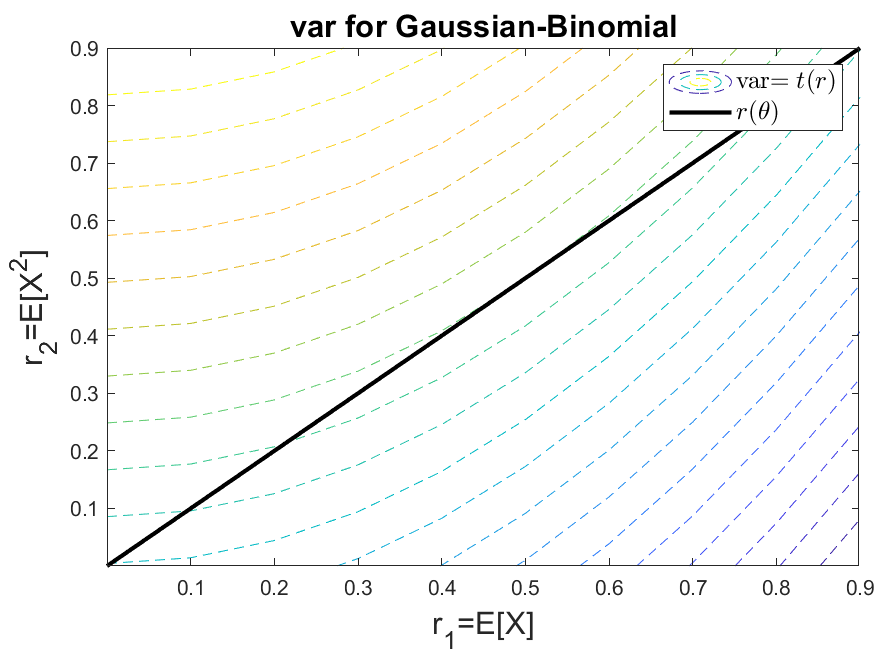}}
  \subfigure[$ N(\theta, 2-2 \theta )$]{\includegraphics[width=0.3\textwidth]{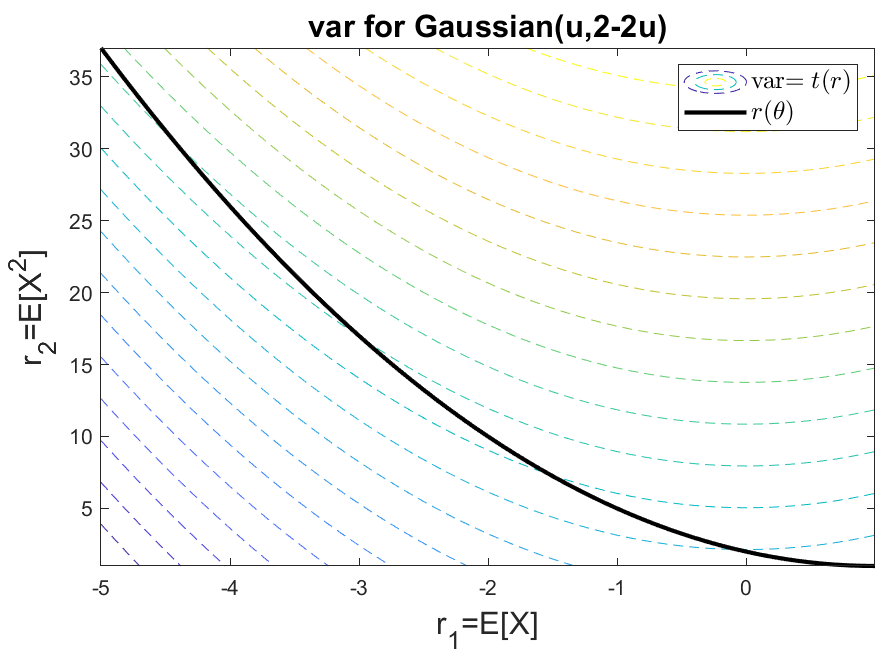}}
  
  \caption{The display of different types of model assumptions for $\bm r(\theta)$ and the contours of the link function $\gamma = t(\bm r) = r_2 - r_1^2$ for $\Gamma = var(\cdot)$. The black curves are the models of $\bm r(\theta)$ while the colored dashed curves are the contours of the link function. }
  \label{fig-var-for-q-model}

\end{figure}

In Figure \ref{fig-var-for-q-model}, examples (a)-(d) corresponds to the situation of (b) in Theorem \ref{theorem-2D-general} and Figure \ref{fig_dR_dT_consistent}, where $R'(r_1) > T'(r_1; t_0) > 0$ for all $t_0$ and thus $\gamma (\theta_{\bm c}^*) = t (\bm r(\theta_{\bm c}^*) )$ always move farther away from $\Gamma (\bm p) = t (\hat{\bm r}(\bm p)) $ for all $\hat{\bm r}(\bm p)$ as increasing $c_1$. Example (e) is special, where the distribution model is Bernoulli and the relationship between $R'(r_1) $ and $ T'(r_1; t_0)$ are not consistent on the whole domain. In the area around $r_1 = \theta = 0.1$, $R'(r_1) > T'(r_1; t_0)$, while in the area around $r_1 = \theta = 0.9$, $R'(r_1) < T'(r_1; t_0)$. In the area around $r_1 = \theta = 0.5$, the link function would not be monotone. Thus, how $t (\bm r(\theta_{\bm c}^*))$ behaves as increasing $c_1$ depends on where the $\hat{\bm r}$ is. Appendix \ref{appendix-section-experimentsForVariance} provides sufficient supporting simulated results. That being said, we can see here that if $\hat{\bm r} (\bm p)$ is relatively close to the sub-property model curve $\bm r (\theta)$, then the possible area of $\bm r(\theta_{\bm c}^*)$ will not cover the area of $r_1 = \theta = 0.5$ with a high probability. This is why the monotonicity pattern of $\gamma (\theta_{\bm c}^*)$ is still common with different $\hat{\bm r} (\bm p)$ in the Bernoulli model setting. Example (f) corresponds to the situation of (c) in Theorem \ref{theorem-2D-general}, but it uses an artificial distribution. So, our theory on 2-D cases perfectly explains how the experimental results occur with different settings in the simulation studies.


\section{Extensive analysis for higher-dimensional spaces} \label{section-3D}

In the previous section, we have studied in 2-D space what structure of the parametric model of sub-properties and the target property link function can satisfy the condition (A) and (B) in Theorem \ref{theorem-general} in Section \ref{section-general-theory}. In this section, we will discuss the situation in higher-dimensional spaces with $M>2$ in Problem \ref{main_problem}. 

\subsection{Theoretical results}
We give the theoretical results in higher-dimensional cases as follows. 

\begin{figure}[!htbp]
  \centering

  \subfigure[An example of $\bm r(\bm \theta)$ in 3-D satisfying the assumption in Theorem \ref{theorem-conditionForSubproperties-MD}. ]{ \includegraphics[width=0.45\textwidth]{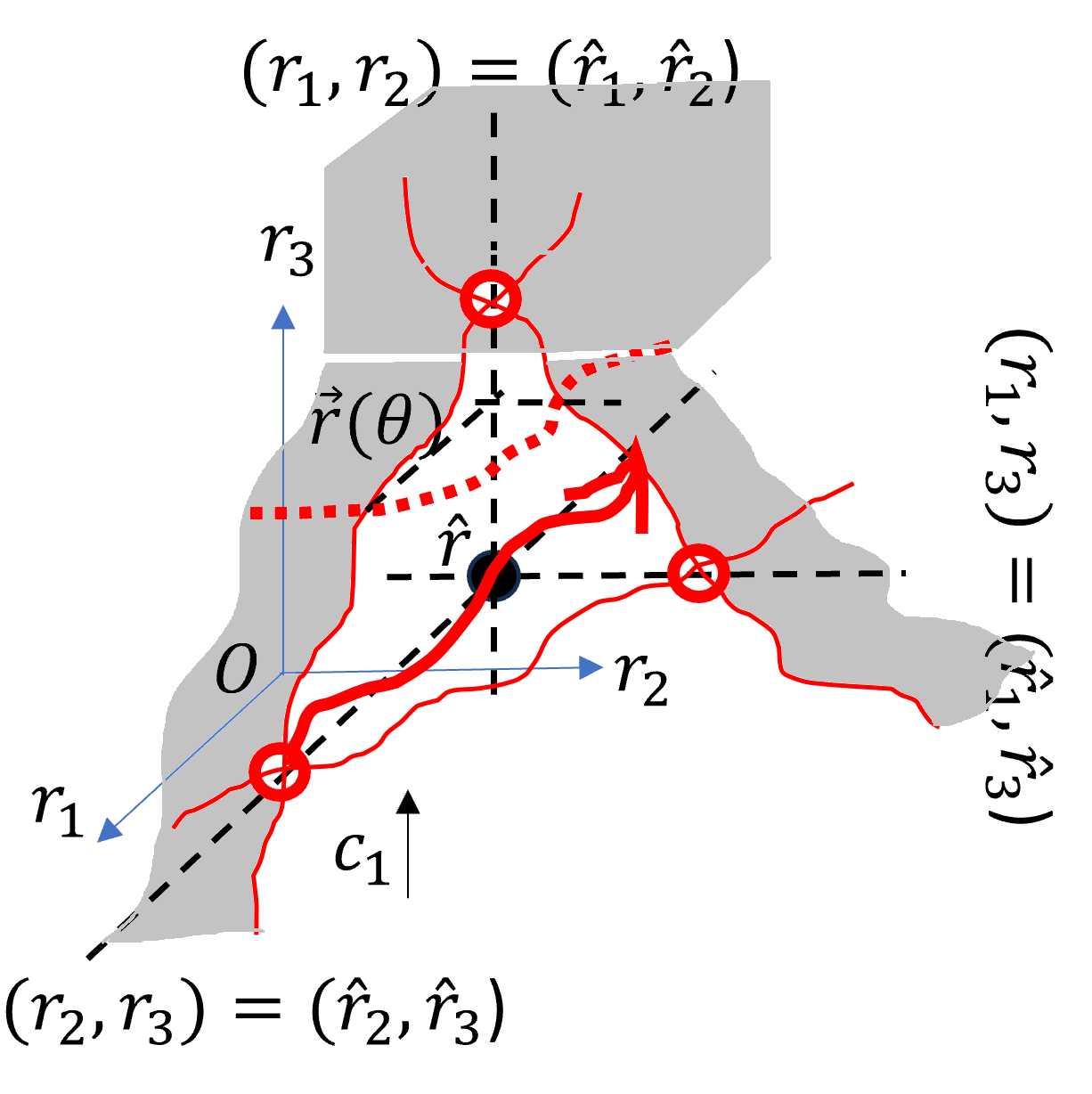} 
  \label{fig_model_error_sign_c_illustration_proof}}
  \quad
  \subfigure[An example of linear cases in 3-D.]{\includegraphics[width=0.45\textwidth]{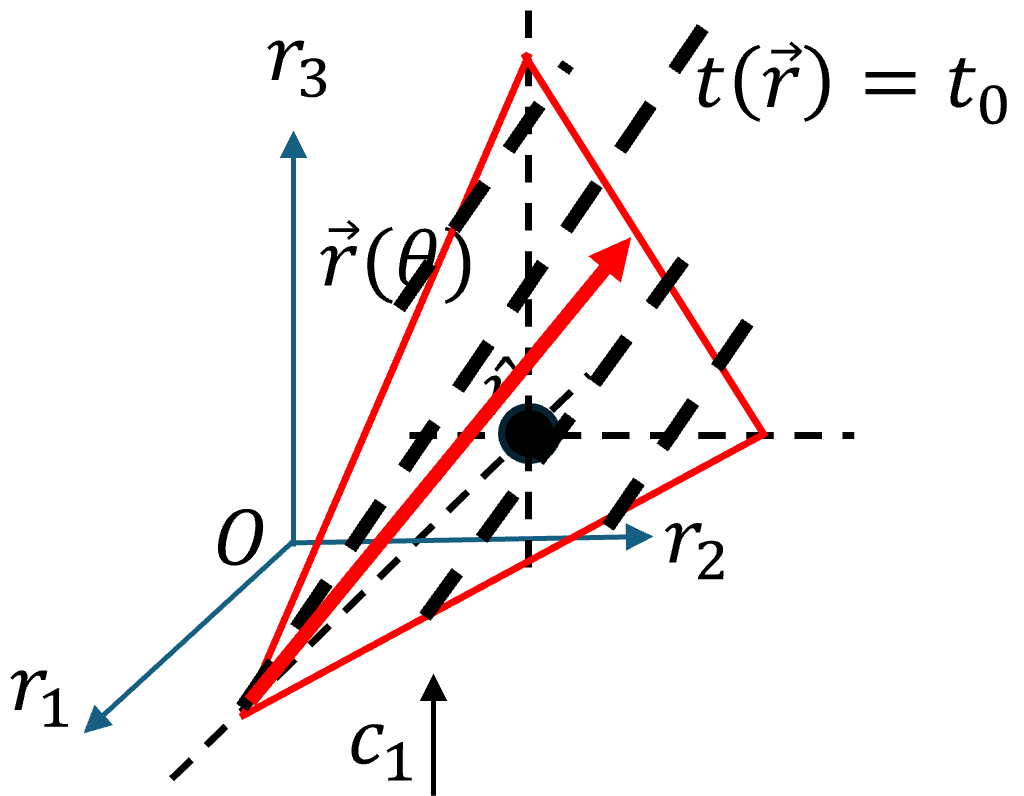} \label{fig_property_model_c_linear_3D}}
  \caption{ Illustrative examples in 3-D where $\mathcal{R}_{\Theta}$ is a surface and $\bm \theta \in \mathbb{R}^2$. } 

\end{figure} 

First, to check what structure of $\mathcal{R}_{\Theta}$ with $M>2$ will satisfy the condition (A), we try embedding the analysis for 2-D cases into any $M$-D situation and then have the following theorem. 

\begin{theorem}\label{theorem-conditionForSubproperties-MD}
For Problem \ref{main_problem} with $ M \geq 2 $, suppose that $\bm \theta^*_{\bm c}(\bm p) $ exists for all $\bm p$ and $\bm c$ and all sub-losses are accuracy-rewarding. If the slice of $\bm r(\bm \theta)$ with fixing any $M-2$ components, that is, the set $\{ \bm r (\bm \theta) | r_{-(i,j)}(\bm \theta) = r'_{-(i,j)} \}$ for all $i<j$ and all $r'_{-(i,j)}$, is always a strictly monotone curve, then $\mathcal{T}_{c_i} (\bm p)$ is in the same orthant centered at $\hat{\bm r} (\bm p) $ for all $i$ and $c_{-i}$.
\end{theorem}


\begin{proof}
    We can use the induction method to prove Theorem \ref{theorem-conditionForSubproperties-MD}. In order to visualize the basic idea of the induction step and to make it easier for readers to understand the general proof, we will discuss first how to extend the 2-D theory into 3-D theory.  
    
    Figure \ref{fig_model_error_sign_c_illustration_proof} shows an illustrative example of $\bm r(\bm \theta)$ in 3-D satisfying the assumption of Theorem \ref{theorem-conditionForSubproperties-MD}. The axis-$r_i$ centered at $\hat{\bm r} $ corresponds to the precise value of $\hat{r}_{-i}$ for each $i$. Based on the assumption, non-trivially there would be three different intersections between $\mathcal{R}_{\Theta}$ and the three axes centered at $\hat{\bm r} (\bm p)$. Define the intercept on each axis-$r_i$ by $\varepsilon_i$, and then the global coordinates of the intersection point with axis-$r_i$ would be $A_i \triangleq (\hat{r}_i + \varepsilon_i, \hat{r}_{-i} ) $.\footnote{Here the representation of the coordinates are conceptual, the real order of coordinates should be re-organized such that the indices of $\hat{r}$ are increasing such as $(\hat{r}_{1}, \hat{r}_2 + \varepsilon_2, \hat{r}_{3} ) $. } Without loss of generality, we assume that $\varepsilon_i > 0$ for all $i$, which is true in Figure \ref{fig_model_error_sign_c_illustration_proof}.

    Now we slice $\mathcal{R}_{\Theta}$ orthogonal to axis-$r_3$ and then get a curve on each slicing plane. Denote the slicing plane at any $r'_3$ by $H_{r'_3} = \{ \bm r | r_3 = r'_3 \}$ and the sliced curve by $S_{r'_3} = \{ \bm r (\bm \theta) | r_3(\bm \theta) = r'_3 \}$. According to the current assumption, $S_{r'_3}$ is strictly monotone for all $r'_3$. We have the following two claims about $S_{r'_3}$.

    \begin{itemize}
        \item[(1)] When $r'_3 \geq \hat{r}_3 + \varepsilon_3$, the sliced curve $S_{r'_3}$ is above the top intersection $A_3$, and thus is dominated by $A_3$ and can be ruled out for the optimal estimation for all $\bm c$. This can be proved with the similar analysis in the proof of Theorem \ref{theorem-conditionForSubproperties-2D}.
        \item[(2)] When $r'_3 < \hat{r}_3 + \varepsilon_3$, the sliced curve $S_{r'_3}$ is strictly below the top intersection $A_3$ such as the red dotted curve in Figure \ref{fig_model_error_sign_c_illustration_proof}. Use $\varepsilon'_k (r'_3)$ for $k \in \{1,2\}$ to denote the intercept of the sliced curve $S_{r'_3}$ on each axis of the plane $H_{r'_3}$ centered at $(\hat{r}_1, \hat{r}_2, r'_3)$. We can verify that $\varepsilon'_k (r'_3)$ keeps the same sign of $\varepsilon_k$ for all $k$ due to the strict monotonicity assumption specified in the theorem. According to Theorem \ref{theorem-conditionForSubproperties-2D}, the shaded area of each slice is the impossible area of $\bm r (\bm \theta_{\bm c}^*)$ for all $\bm c$. Then the possible area of $\bm r (\bm \theta_{\bm c}^*)$ for all $\bm c$ would be in the portion $\{ \bm r | r_3 = r'_3, r_k(\bm \theta) \geq \hat{r}_k, k=1,2 \}$. 
    \end{itemize}

    Combining (1) and (2) above, we see that the possible area of $\bm r (\bm \theta_{\bm c}^*)$ for all $\bm c$ would be in the portion $\{ \bm r (\bm \theta) | r_k(\bm \theta) \geq \hat{r}_k, k =1,2 \}$.
    
    Similarly, we can show that the possible area of $\bm r (\bm \theta_{\bm c}^*)$ for all $\bm c$ would be in the portion $\{ \bm r (\bm \theta) | r_k(\bm \theta) \geq \hat{r}_k, k =2,3 \}$, with slicing $\mathcal{R}_{\Theta}$ orthogonal to axis-$r_1$. In one word, $\bm r (\bm \theta_{\bm c}^*)$ for all $\bm c$ can be proved to stay in the portion $\{ \bm r (\bm \theta) | r_k(\bm \theta) \geq \hat{r}_k, k \neq j \}$ with slicing $\mathcal{R}_{\Theta}$ orthogonal to axis-$r_j$ for each $j$. Therefore, in total $\bm r (\bm \theta_{\bm c}^*)$ for all $\bm c$ would be in the portion $\{ \bm r (\bm \theta) | r_k(\bm \theta) \geq \hat{r}_k, \forall k \}$.

    We can extend the above trick of proof to the general induction step from $(m-1)$-dimensional cases to $m$-dimensional cases with $m \geq 3$. We leave the complete proof to Appendix \ref{appendix-proofs-MD-subproperties}
\end{proof}

Now we check the structure of the target property link function $\gamma = t (\bm r) $ along the sub-property trajectories. With one more dimension, there are more possibilities for sub-property trajectories. As Figure \ref{fig_model_error_sign_c_illustration_proof} shows, as increasing $c_1$ from $0$ to $+\infty$, the trajectory $\mathcal{T}_{c_1}$ goes from the left intersection with axis-$r_1$ (corresponding to the precise estimation of $\hat{r}_2$ and $\hat{r}_3$ ) to the hyperplane $\{ \bm r | r_3 = \hat{r}_3 \}$ (corresponding to the precise estimation of $\hat{r}_3$). But the endpoint of $\mathcal{T}_{c_i}$ would depend on $c_{-1}$. Besides depending on $c_{-i}$, the specific trajectory $\mathcal{T}_{c_i}$ also relies on the whole optimization problem, which is hard to analyze. Consequently, for the condition (B) about the behavior of the target property along sub-property trajectories, it is not straightforward to extend the 2-D theory into higher-dimensional cases. 

Due to these challenges, we study the linear special case, where both $\mathcal{R}_{\Theta}$ and $\gamma = t(\bm r)$ are linear. We will also assume that the trajectories $\mathcal{T}_{c_i} (\bm p))$ are linear for all $\bm p$, $i$ and $c_{-i}$, because it is unknown whether the linearity of $\bm r(\bm \theta)$ can guarantee this or not. It is intuitive that in linear situations the target property $\gamma = t(\bm r)$ is monotone along the linear trajectory $\mathcal{T}_{c_i} (\bm p)$ for all $i$. We have the following theorem about linear cases.

\begin{theorem}\label{theorem-linearCases}
Suppose that $\bm \theta^*_{\bm c}(\bm p) $ exists for all $\bm p$ and $\bm c$, and that $t(\bm r)$ and $\bm r(\bm \theta)$ are linear. Additionally, suppose that trajectories $\mathcal{T}_{c_i} (\bm p)$ are also linear for all $\bm p$, $i$, and $c_{-i}$. Then, $\gamma (\bm \theta^*_{\bm c}(\bm p) ) = t (\bm r (\bm \theta^*_{\bm c}(\bm p) )$ is monotone with increasing $c_i$ for all $\bm p$, $i$, and $c_{-i} $, and is constant if furthermore $\mathcal{T}_{c_i} (\bm p)$ is on an isosurface of $\gamma (\bm r) = t (\bm r)$. 
\end{theorem}
    
The proof of Theorem \ref{theorem-linearCases} is straightforward and we leave it to Appendix \ref{appendix-proofs-MD-linearCases}. The linear case is highly constrained, but it is an important element for the analysis of more complicated situations, as we will see later.

\subsection{Correspondence to empirical results} \label{section-3D-empiricalAnalysis}

In this section, we will discuss the simulation study in 3-D. We find that the theory on linear cases can help interpret the empirical phenomenon.  

For 3-D experiments, we consider the skewness property $\Gamma = skew(\cdot)$. We try different settings of $\bm q_{\bm \theta} $ and $\bm p_0$, including log-normal, log-logistic, Gamma, and Beta distribution models. Similar to the results in 2-D, in most cases $\Gamma_{c_i}$ curves in 3-D are also monotone in $c_i$ for all $i$. Due to the page limit, we leave all experimental design and results to Appendix \ref{appendix-section-experimentsForSkewness}.

\begin{remark}
During the simulation with the skewness property and the log-normal model for distributions, we find that due to the intrinsic dominance of some poorly-behaved sub-losses, the optimal estimation is highly sensitive to the initial iterative point in optimization. To solve this issue, we use the weight-renormalization trick. Re-denote the weights by $\tilde{\bm c} = \bm c ./ \bm k$, where $\bm k \in \mathbb{R}_{++}^M$ and $\bm c \in \bar{\mathbb{R}}_{+}^M$ and $./$ is the element-wise division. This renormalization would not affect the theoretical analysis in this paper, but is helpful from practical perspectives. See Appendix \ref{appendix-LogN-optimizationDifficulties} and \ref{appendix-weightRenormalization} for more detail.
\end{remark}

\begin{figure}[!htbp]
  \centering
  \subfigure[$\bm q_{\bm \theta} = LogLogistic$]{\includegraphics[width=0.31\textwidth]{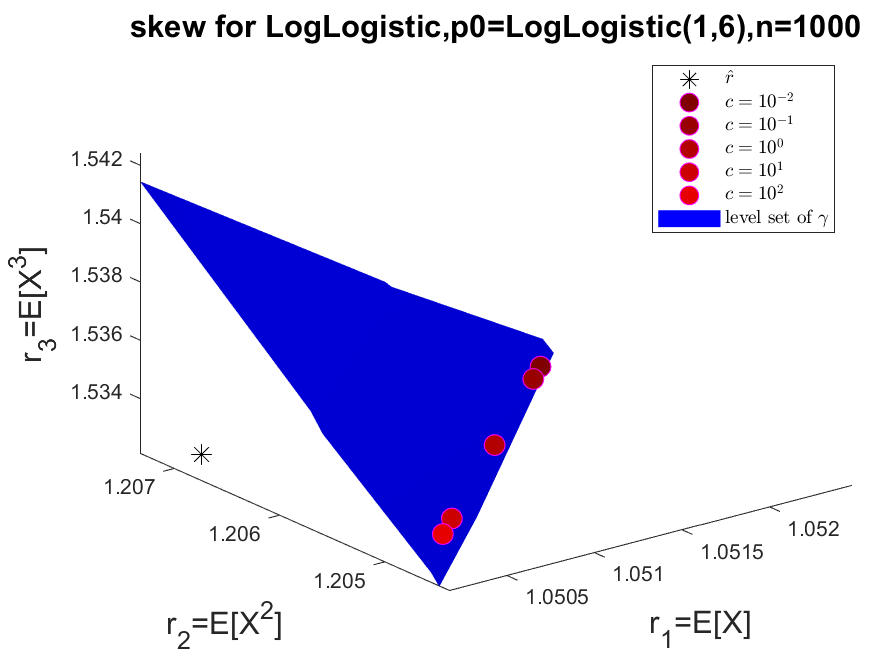} }
  \subfigure[$\bm q_{\bm \theta} = Gamma$]{\includegraphics[width=0.31\textwidth]{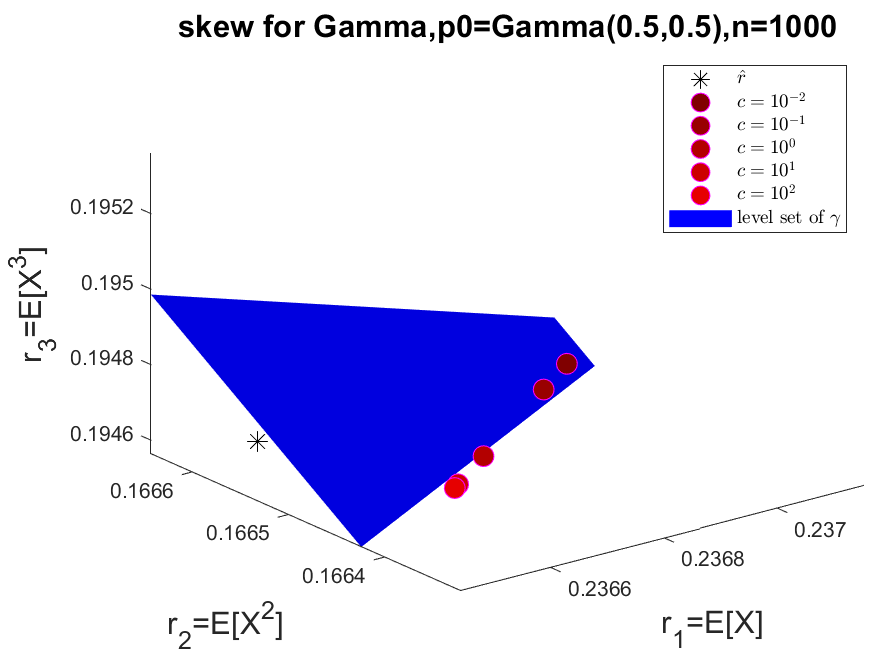} }
  \subfigure[$\bm q_{\bm \theta} = Beta$]{\includegraphics[width=0.31\textwidth]{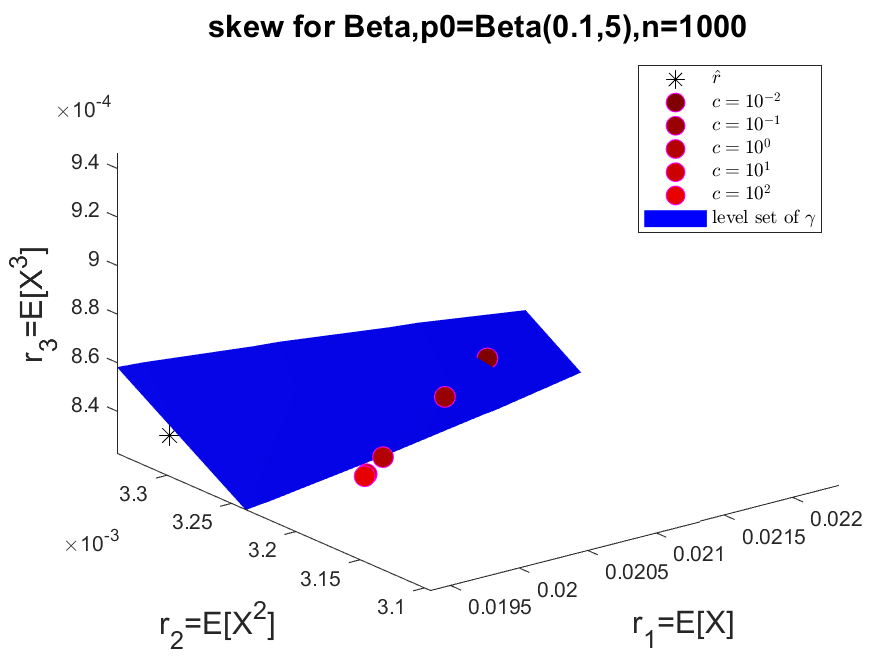} }
  \caption{ The sub-property trajectory $\mathcal{T}_{c_1}(\hat{\bm p})$ for different $\bm q_{\bm \theta}$. The blue surface is the level set of $\Gamma = skew(\cdot)$ nearby each trajectory. } 
  \label{fig_skew_for_q_model}

\end{figure}

Figure \ref{fig_skew_for_q_model} plots the sub-property trajectory $\mathcal{T}_{c_1}(\hat{\bm p})$ for different $\bm q_{\bm \theta}$ and also the level set of $\Gamma = skew(\cdot)$ nearby the trajectory. Interestingly, we can see that in each case the trajectory $\mathcal{T}_{c_1}(\hat{\bm p})$ and the isosurface of the target property are approximately linear, and then results can be roughly interpreted by the linear theory. This is reasonable from the ex-post perspective. Recall that if we focus on a small local area of a well-behaved function, it would be monotone or approximately linear most of the time. So we think that a more general case of the studied problem can be approximated by a linear case when $\hat{r}(\bm p)$ is close enough to the model surface $\bm r(\bm \theta)$. This can interpret from the high-level view why most of the empirical results that we have tested so far show the monotonicity pattern of the optimal model estimation for the target property as increasing each weight.

Furthermore, it is possible to locally find an isomorphism between a surface and a hyperplane based on the theory in differentiable geometry. Let us talk about the log-normal model $\bm q_{\bm \theta} = LogN(u,v^2)$ with $\bm \theta = (u,v^2)$ and skewness property here. Recall that the moments of the log-normal model is 
\[
r_i(\bm \theta) = E[X^i] = e^{iu + \frac{1}{2} i^2 v^2}, \, i=1, \cdots, 3.
\]
Then we have $\log(r_i) = iu + \frac{1}{2} i^2 v^2$, which is linear w.r.t. $(u,v^2)$. Applying the $\log$-transformation on the coordinates $\bm r = (r_1, r_2,r_3)$ with $ (x, y,z)= \log(\bm r)$, we have 
\[
\left \{
\begin{array}{l}
\frac{1}{2}(4x(\bm \theta)-y(\bm \theta)) = u, \\
-2x(\bm \theta)+y(\bm \theta) = v^2 \geq 0, \\
3x(\bm \theta)-3y(\bm \theta)+z(\bm \theta)=0.
\end{array}
\right .
\]
Then it is very clear that the image of the log-normal model, $(x(\bm \theta), y(\bm \theta), z(\bm \theta))$ is a subset of a hyperplane on $(x,y,z)$ space. What's more,
\begin{eqnarray*}
   \Gamma & = & skew(\bm q_{\bm \theta}) = (e^{v^2}+2) \sqrt{e^{v^2} -1 } \\
   & \approx & e^{v^2} e^{\frac{v^2}{2}} = e^{\frac{3}{2} v^2} \\
   & = & e^{\frac{3}{2} (y(\bm \theta)-2x(\bm \theta))}.
\end{eqnarray*}
It means that the level set of skewness for the log-normal model is approximately a hyperplane on $(x,y,z)$ space if $v^2$ is large enough. So, $\Gamma = skew(\cdot)$ and $\bm q_{\bm \theta} = LogN(u,v^2)$ correspond to a linear case in a large domain after applying the $\log$ coordinate transformation. Even if for more general cases where the mapped target property function may not be linear, it would be still easier to analyze the target property's behavior on a mapped linear sub-property model space. 


\section{Conclusions}

In this paper, we study the choice of proper losses for the task of indirect-elicitation of a target property through the lens of parametric model estimation. 
First, we introduce a fully separable loss function framework consisting of a weighted sum of several sub-losses, where the target property is a function of several sub-properties and each sub-loss directly elicits a sub-property. With fixing the choice of sub-losses, we propose a problem of how to choose the weights for sub-losses to achieve the best parametric estimation for the target property. 
The empirical results of simulation studies display that in most of the tested cases with different settings of the target property and the distribution model assumption and the true distribution, the optimal parametric estimation for the target property change monotonically with the increase of each weight. 

Concentrated on the monotonicity pattern, we establish a general theoretical framework about the potential reason behind it. An important trick is to decompose the change of the parametric estimation of the target property into two parts: the change of the parametric estimation of sub-properties with weights and the change of the target property with sub-properties. Consequently, we provide high-level conditions for the sub-property model and the target property respectively. Basically, we require that the changes in both parts are "monotone". But, these high-level conditions do not indicate what the desirable structures of the sub-property model and the target property should be and thus cannot be directly verified. 
So next we provide deeper theoretical analysis to uncover the assumption for the desirable structure such that the high-level conditions can be satisfied. We start it with 2-D cases. It shows that for the sub-property model curve, being strictly monotone is sufficient. As for the target property, if the relationship between the derivatives of the target property contour and that of the sub-property model curve can keep consistent, then the corresponding condition can be satisfied. The theoretical results in 2-D cases perfectly explains our observation in the 2-D simulated results. We also try extending the analysis into higher-dimensional situations, and we especially study the linear case where both the parametric model of sub-properties and the target property function are linear. With connecting to 3-D empirical results, we realize that any general case can be locally approximated by a linear case or mapped into a linear case if the true value of the sub-properties is close enough to the model assumption. We also find the map for the log-normal distribution model and skewness property. All these explain reasonably why the monotonicity pattern of the parametric estimation is so common among all different experimental settings.

\section{Discussion}
There are some remaining issues for our work to discuss here. First, for high-dimensional situations, our theory is not complete and mostly focuses on the linear special case. We are excited about the broader picture of the story. Second, even in 2-D cases, we do not dig deeper to see what the specific setting for the best weight would be if the best value is a finite number. Basically, there are two things to figure out here. One is whether or not the precise estimation of the target property is achievable by tuning weights, and the other is what the best weight is and how to calculate it. Third, it is noticeable that among simulation studies with commonly used probability distribution models, the best configuration of each weight would be $0$ or $+\infty$ and some sub-loss should be canceled. If this is not desirable, how should we handle it. More importantly, we could ask from the theoretical perspective what the best weight being $0$ or $+\infty$ means and whether it would happen in another different problem framework or not. 

Besides those lower-level issues mentioned above, there are also some higher-level open problems for future research. First, except the monotonicity pattern discussed in this paper, what other aspects could we observe about the task framework? Second, we only discuss the choice of weights in the total loss while fixing all sub-losses, there would be definitely more diversity if incorporating the choice of sub-losses into discussion. Third, we choose the target property itself as the guideline for the best setting of weights in this paper, but we could choose a different criterion for what the "best" parametric estimation should be. Fourth, we only study fully separable form of the total loss in our framework, and have not proposed any feasible idea about the non-fully separable losses. It is worth exploring further in this direction. 


\section*{Acknowledgement}
This work is partially supported by the National Science Foundation under Grant No. 2110707. We would like to thank Shuaiang Rong and Siyu Li for the feedback on an early version of the introduction part, and Shihui Ying for the discussion on the contribution of this work.

%
%
%
%
%

\bibliographystyle{plainnat}
\bibliography{reference}

\appendix


\section{More experiments for variance} \label{appendix-section-experimentsForVariance}

The experiment design for $\Gamma=var(\cdot)$ is seen in Table \ref{table-experiment-design-for-variance}. Actually, every choice of $\mathcal{D}_{\Theta}$ here can be approximated by a special case of Gaussian distribution family under certain conditions, as listed in the column of "approximation".

\begin{table}[!htbp] 
    \caption{ The experiment design for showing the $\Gamma_{c_i}$ curves with $\Gamma=var(\cdot)$. }
    \vspace{0.01pt}
    \centering
    \begin{tabular}{c | c | c | c | c }
        \hline
        $\Gamma$ & $\hat{\bm r}$ & $t(\bm r)$ & $\bm q_{\bm \theta}$ \& $\bm p_{0}$ & approximation\\
        \hline  
        \multirow{10}*{variance} & \multirow{10}*{$(E[X], E[X^2])$} & \multirow{10}*{$r_2 - r_1^2$} & 
        $Poisson(\theta) $ & \thead[c] {$N(\theta, \theta)$\\ $\theta>0$} \\
        \cline{4-5}
        & & & $\chi^2(\theta)$ & \thead[c] {$N(\theta, 2\theta)$\\ $\theta>0$} \\
        \cline{4-5}
        & & & $Exponential(1/\theta)$ & \thead[c] {$N(\theta, \theta^2)$\\ $\theta>0$} \\
        \cline{4-5}
        & & & $Gamma(K, \theta)$ & \thead[c] {$N(K\theta, K \theta^2)$\\ $K>0$, $\theta>0$} \\
        \cline{4-5}        
        & & & $Binomial(K, \theta)$ & \thead[c] {$N(K\theta, K\theta(1-\theta))$\\ $K \in \mathbb{N}_+$, $\theta\in [0,1]$} \\
        \hline
        
                
    \end{tabular}
	\label{table-experiment-design-for-variance}
\end{table}

We show the simulated results for all different settings of $\bm q_{\theta}$ and $\bm p_0$ here. Overall, $\Gamma_{c_i}$ curves are monotone in all tested cases. Furthermore, most of the time the monotonicity directions keep consistent regardless of the setting of $\bm p_0$ given the same model assumption $\bm q_{\theta}$. But the Binomial distribution $B(K,\theta)$ is an exception, where $K$ is number of trials. Figure \ref{fig-property-c-curves-var-Binomial} shows that $\theta_0 = 0.1$ and $\theta_0 = 0.9$ lead to totally opposite monotonicity directions of $\Gamma_{c_i}$ for the same $i$. These results can be perfectly explained by our theory in 2-D spaces.

\begin{figure}[!htbp]
  \centering
  \subfigure[$ \theta_0 = 0.1 $]{\includegraphics[width=0.3\textwidth]{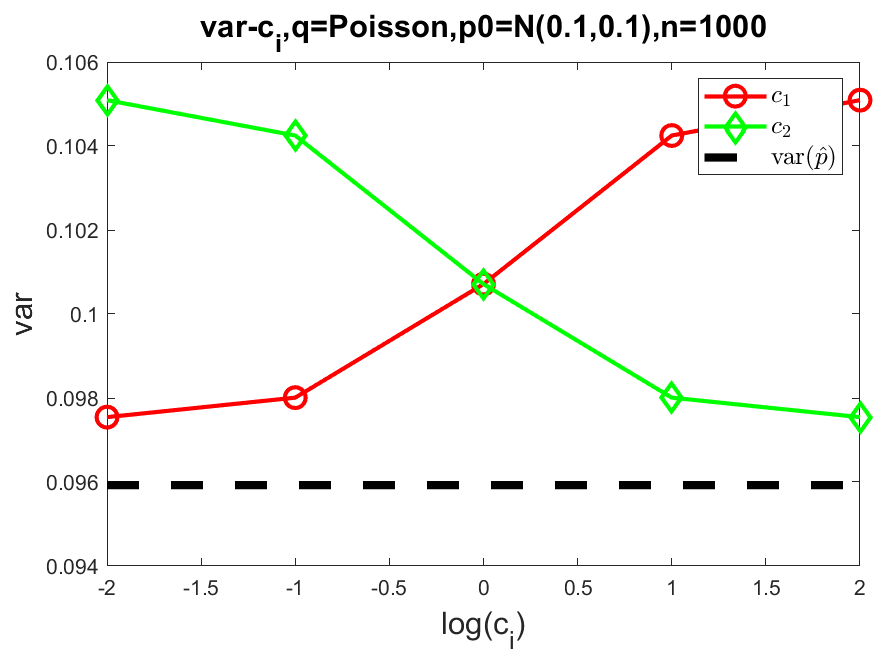}}
  \subfigure[$ \theta_0 = 1 $]{\includegraphics[width=0.3\textwidth]{var_c_1D,q=Poisson,p0=N1,1,n=1000_random3_new2.png}}
  \subfigure[$ \theta_0 = 10 $]{\includegraphics[width=0.3\textwidth]{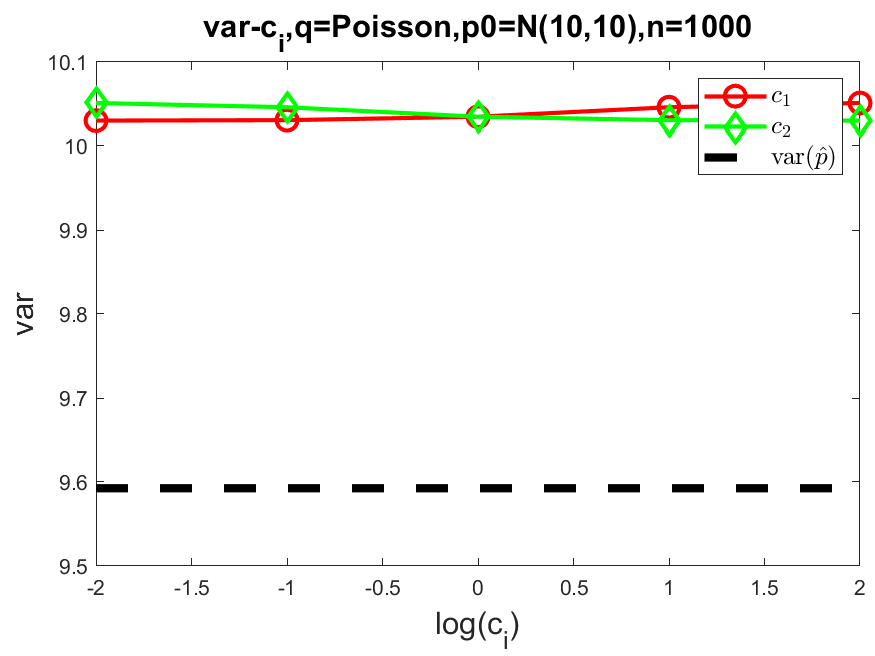}}
  
  \caption{ The $\Gamma_{c_i}$ curves for the cases of $\Gamma = var(\cdot)$ and $\bm q_{\theta} = Poisson(\theta)$ and $\bm p_0 = N(\theta_0, \theta_0)$ with different $\theta_0$. } \label{fig-property-c-curves-var-Poisson}
\end{figure}

\begin{figure}[!htbp]
  \centering
  \subfigure[$ \theta_0 = 1 $]{\includegraphics[width=0.3\textwidth]{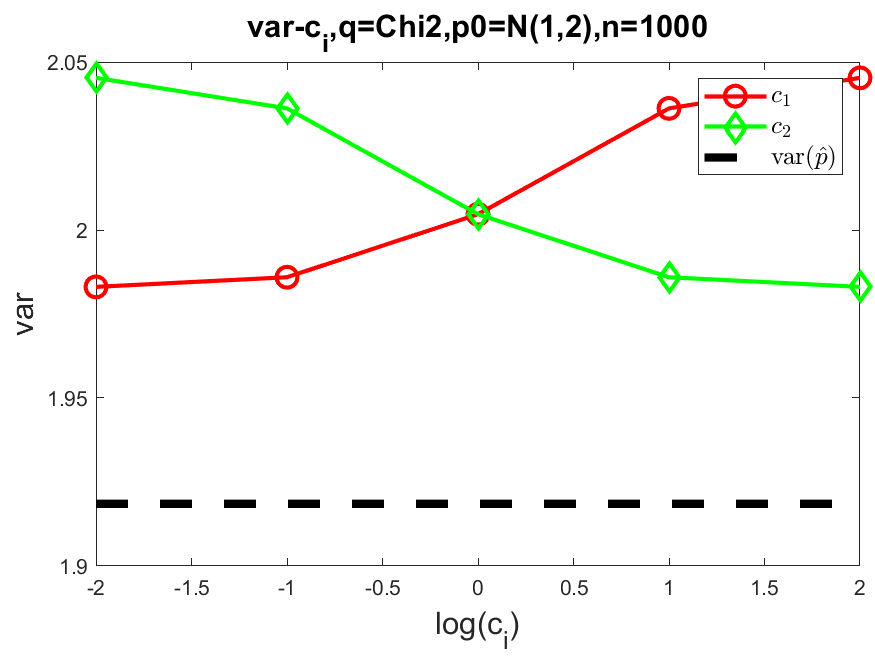}}
  \subfigure[$ \theta_0 = 10 $]{\includegraphics[width=0.3\textwidth]{var_c_1D,q=Chi2,p0=N10,20,n=1000_random3_new2.png}}
  \subfigure[$ \theta_0 = 100 $]{\includegraphics[width=0.3\textwidth]{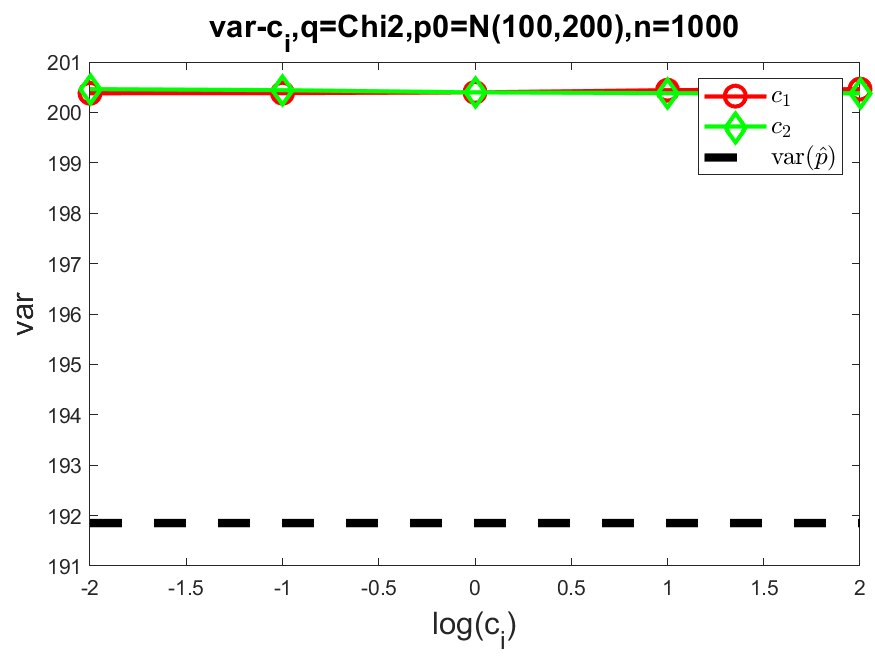}}
  
  \caption{ The $\Gamma_{c_i}$ curves for the cases of $\Gamma = var(\cdot)$ and $\bm q_{\theta} = \chi^2(\theta)$ and $\bm p_0 = N(\theta_0, 2 \theta_0)$ with different $\theta_0$. }
\end{figure}

\begin{figure}[!htbp]
  \centering
  \subfigure[$ \theta_0 = 0.01 $]{\includegraphics[width=0.3\textwidth]{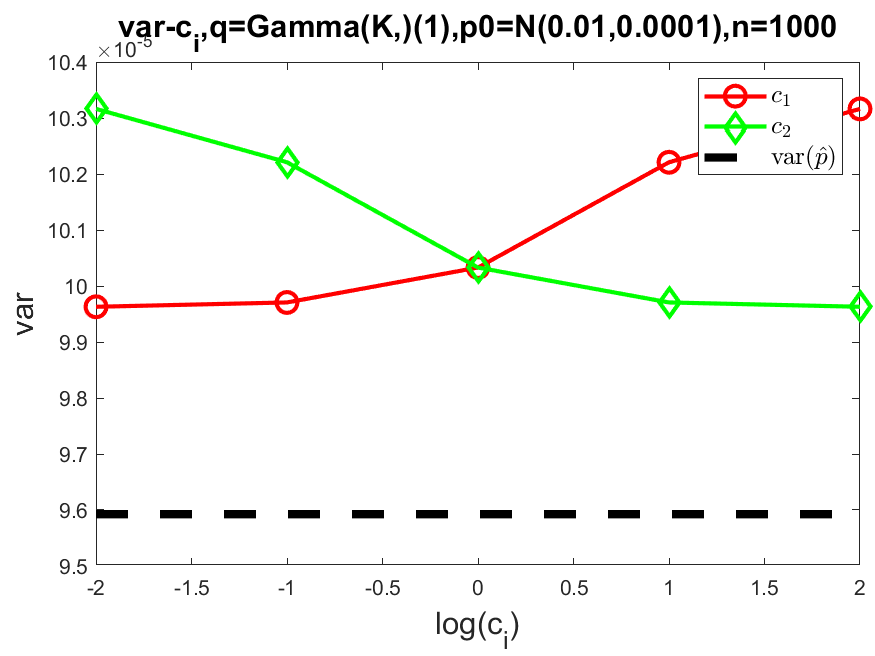}}
  \subfigure[$ \theta_0 = 1 $]{\includegraphics[width=0.3\textwidth]{var_c_1D,q=GammaK,1,p0=N1,1,n=1000_random3_new2.png}}
  \subfigure[$ \theta_0 = 100 $]{\includegraphics[width=0.3\textwidth]{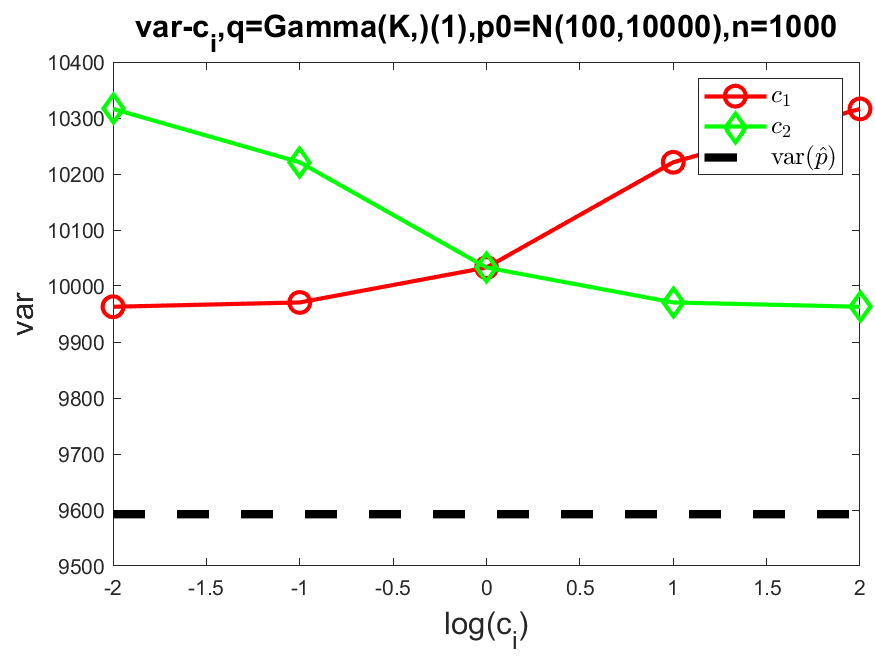}}
  
  \caption{ The $\Gamma_{c_i}$ curves for the cases of $\Gamma = var(\cdot)$ and $\bm q_{\theta} = Exponential(1/\theta)$ and $\bm p_0 = N(\theta_0, \theta_0^2)$ with different $\theta_0$. }
\end{figure}

\begin{figure}[!htbp]
  \centering
  \subfigure[$K=0.1$, $\theta_0 = 1$]{\includegraphics[width=0.3\textwidth]{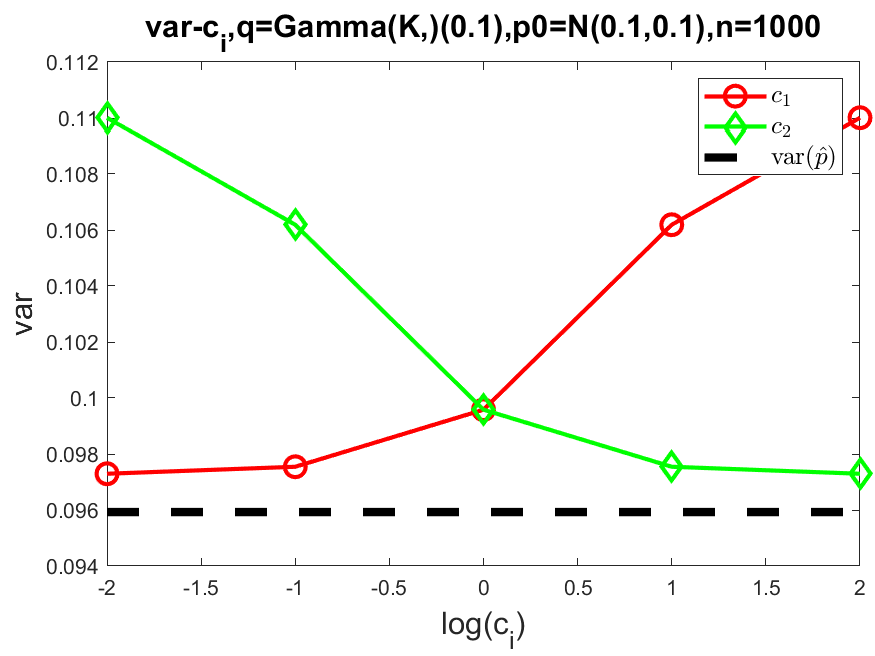}}
  \subfigure[$K=2$, $\theta_0 = 2$]{\includegraphics[width=0.3\textwidth]{var_c_1D,q=GammaK,2,p0=N4,8,n=1000_random3_new2.png}}
  \subfigure[$K=10$, $\theta_0 = 2$]{\includegraphics[width=0.3\textwidth]{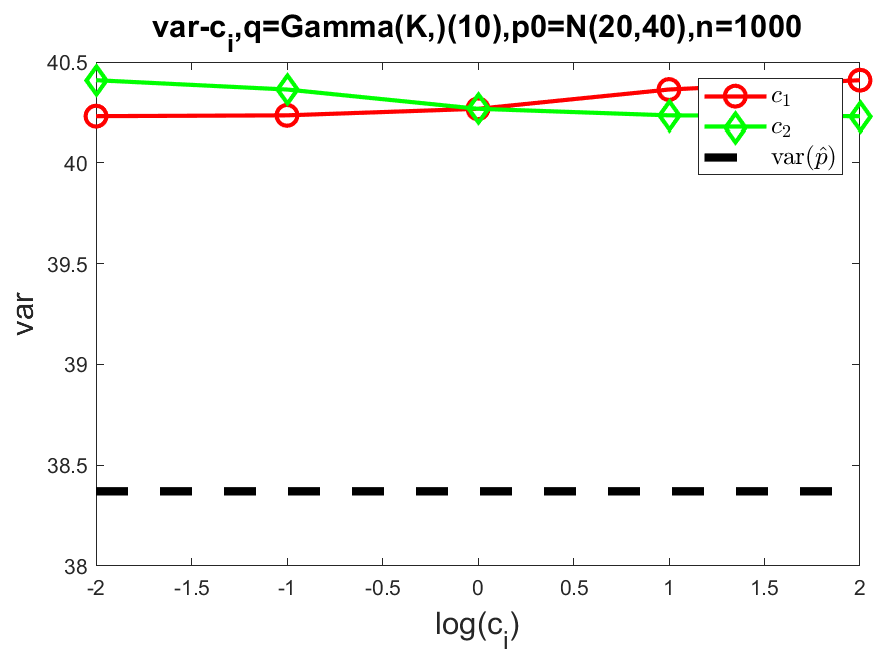}}
  
  \caption{ The $\Gamma_{c_i}$ curves for the cases of $\Gamma = var(\cdot)$ and $\bm q_{\theta} = Gamma(K,\theta)$ and $\bm p_0 = N(K \theta_0, K \theta_0^2)$ with different $\theta_0$. }\label{fig-property-c-curves-var-Gamma}
\end{figure}

\begin{figure}[!htbp]
  \centering
  \subfigure[$\theta_0 = 0.1$]{\includegraphics[width=0.3\textwidth]{var_c_1D,q=Binomial10,p0=N1,0.9,n=1000_random3_new2.png}}
  \subfigure[$\theta_0 = 0.5$]{\includegraphics[width=0.3\textwidth]{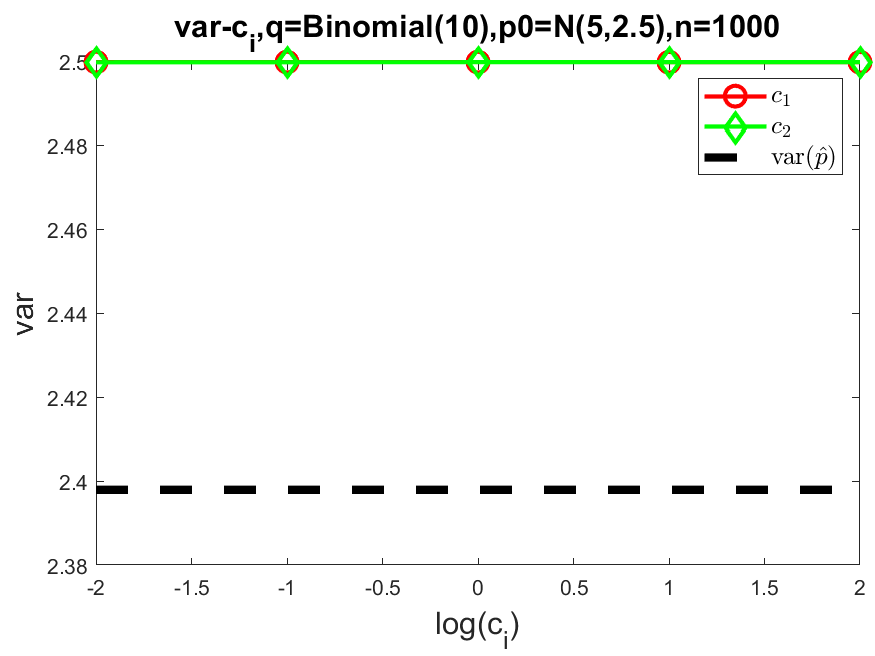}}
  \subfigure[$\theta_0 = 0.9$]{\includegraphics[width=0.3\textwidth]{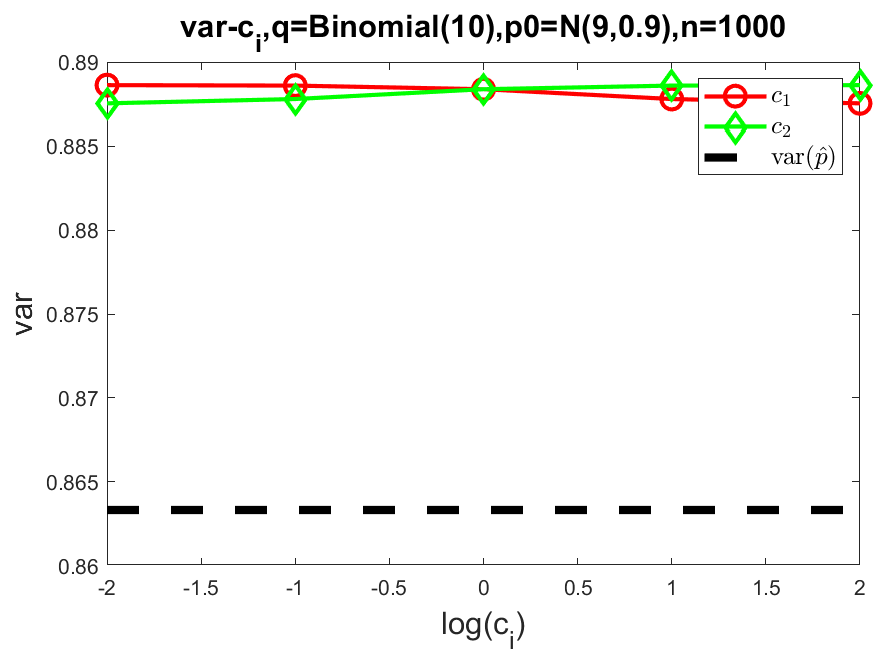}}
  
  \caption{ The $\Gamma_{c_i}$ curves for the cases of $\Gamma = var(\cdot)$ and $\bm q_{\theta} = B(10, \theta)$ and $\bm p_0 = N(10 \theta_0, 10 \theta_0 (1-\theta_0) )$ with different $\theta_0$. Notice that the monotonicity directions of $\Gamma_{c_i}$ for the same $i$ are opposite for (a) and (c).} 
  \label{fig-property-c-curves-var-Binomial}
\end{figure}


\section{Proofs for 2D cases} \label{appendix-proofs-2D-general}

Before proving Theorem \ref{theorem-2D-general}, we need to prove several lemmas. 

\begin{lemma} \label{lemma-contour-monotonicity}
    Given a continuous bivariate function $t(r_1, r_2)$, the following two statements are equivalent to each other: 
    \begin{itemize}
        \item[(1)] Each non-empty contour curve is monotone, i.e. it can be represented as a monotone function $T(r_1; t_0)$ with $t_0 \in \mathbb{R}$.
        \item[(2)] The function $t(r_1, r_2)$ is monotone along $r_1$-axis and $r_2$-axis respectively. (Or, the signs of $\frac{\partial t}{\partial r_1}$ and $\frac{\partial t}{\partial r_2}$, are unchanged ($0$ included) over the whole space when partial derivatives exist.)
    \end{itemize}
\end{lemma}

\begin{proof}

"(1) $\Longrightarrow (2)$":

We can first assume $t(r_1, r_2)$ is not always monotone along $r_2$ axis. For example, $t$ is not monotone along some vertical line $r_1 = r_1^0$. Then there must be a local minimum or maximum in some interval. According to the continuity, there must exist two points $(r_1^0, r_2^a)$ and $(r_1^0, r_2^b)$ around the minimizer or maximizer such that $t(r_1^0, r_2^a) = t(r_1^0, r_2^b)$. It means that $(r_1^0, r_2^a)$ and $(r_1^0, r_2^b)$ are on the same contour. This contradicts the condition that any contour is a function w.r.t. $r_1$ variable. The argument for horizontal lines is symmetric with no violation of the monotonicity of contours.

"$ (1) \Longleftarrow (2)$": 

Without loss of generality, assume that the function $t(r_1, r_2)$ is increasing along $r_1$-axis and $r_2$-axis respectively. 


Consider a point $(r_1^a, r_2^a)$ with $t_0 \triangleq t(r_1^a, r_2^a)$. Along the $r_1$-axis, we have $t_0 = t(r_1^a, r_2^a) \leq t(r_1^b, r_2^a)$ for any $r_1^b$ with $r_1^a < r_1^b$. Along the $r_2$-axis, if there exist $r_2^b$ such that $t(r_1^b, r_2^b) = t_0 (\leq t(r_1^b, r_2^a))$, then we have $r_2^b \leq r_2^a$. Then, $t_0 = t(r_1^a, r_2^a) = t(r_1^b, r_2^b)$, $r_1^a < r_1^b$, $r_2^a  \geq r_2^b $, all hold true. It means that $T(r_1; t_0)$ is decreasing along $r_1$-axis. 

\end{proof}

\begin{lemma}\label{lemma-targetProperty-monotonicity-along-modelCurve} 
    There are two real-valued functions $R(r_1)$ and $t(r_1, r_2)$. 
    Suppose that $R(r_1)$ is differentiable, and $t(r_1, r_2)$ is continuous and partially differentiable with the sign of $\frac{\partial t}{\partial r_2}$ unchanged ($0$ included) over the whole space. Denote the non-empty contour of $t(r_1, r_2)$ at any level $t_0$ by $T(r_1; t_0)$. Then the following two statements are equivalent to each other.

    \begin{itemize}
        \item[(1)] The function $t(r_1, r_2)$ is monotone along the curve $R(r_1)$, that is, $T_R  (r_1) = t(r_1, R(r_1))$ is monotone.
        \item[(2)] $R'(r_1)- T'(r_1; t_0)$ keeps the sign unchanged ($0$ included) at all points where $R(r_1) = T(r_1; t_0)$.
    \end{itemize}
\end{lemma}

\begin{proof}
    
\[
T_R' (r_1) = \frac{ \text{d} }{  \text{d} r_1} t(r_1, R(r_1)) = 
\left. 
\left( 
\frac{\partial t}{\partial r_1} + \frac{\partial t}{\partial r_2} R'(r_1)
\right) 
\right| _{r_2 = R(r_1)}.
\]

The definition of a contour $T(r_1; t_0)$ is that $t(r_1, T(r_1; t_0)) = t_0$. Taking the derivative w.r.t. $r_1$ on both sides, we have
\[
\left. 
\left( 
\frac{\partial t}{\partial r_1} + \frac{\partial t}{\partial r_2} T'(r_1; t_0) 
\right) 
\right| _{r_2 = T(r_1;t_0)}
= 0.
\]
Thus, we have
\begin{equation} \label{Eq-contour-derivative}
    T'(r_1; t_0) = - 
    \left. 
    \frac{\frac{\partial t}{\partial r_1}}{\frac{\partial t}{\partial r_2} }
    \right| _{r_2 = T(r_1;t_0)}.
\end{equation}

Next, 
\begin{eqnarray}
    & & R'(r_1)-T'(r_1; t_0) \notag \\
    &=& R'(r_1)+ \left. \frac{\frac{\partial t}{\partial r_1}}{\frac{\partial t}{\partial r_2}} \right| _{r_2 = T(r_1; t_0)}  \notag \\
    &=& \left. \frac{ R'(r_1)\frac{\partial t}{\partial r_2} + \frac{\partial t}{\partial r_1} }{\frac{\partial t}{\partial r_2}}  \right| _{r_2 = T(r_1; t_0)} \notag \\
    & & (\text{if } r_2=T(r_1; t_0) = R(r_1), \text{ then}) \notag \\
    &=& \left. \frac{T_R' (r_1)}{\frac{\partial t}{\partial r_2} }\right| _{r_2 = R(r_1)=T(r_1; t_0)} \label{Eq-derivative-difference}
\end{eqnarray}

Because $\frac{\partial t}{\partial r_2}$ keep the sign unchanged ($0$ included) over the whole space, $T_R'(r_1)$ keeps the sign unchanged if and only if $R'(r_1)-T'(r_1; t_0)$ keeps the sign unchanged ($0$ included) wherever $T(r_1; t_0) = R(r_1)$.

\end{proof}

Now we prove Theorem \ref{theorem-conditionForTargetProperty}.

\begin{proof}
    Because $t(\bm r)$ is continuous and partially differentiable and contours $T(r_1; t_0)$ for all $t_0$ are monotone, the sign of $\frac{\partial t}{\partial r_2}$ keeps unchanged ($0$ included) over the whole space according to Lemma \ref{lemma-contour-monotonicity}. 

    Since $R'(r_1)- T'(r_1; t_0)$ keeps the sign unchanged, according to Lemma \ref{lemma-targetProperty-monotonicity-along-modelCurve}, we always have that the target property value $t(r_1, r_2)$ is monotone along the curve $R(r_1)$. Thus, the condition (B) is satisfied.
\end{proof}

We also need a straightforward corollary about the best setting of weights given the monotonicity of $\gamma$ in $c_i$. The corollary is true for any $M\geq 2$.

\begin{corollary}\label{corollary-optimalC}
    For Problem \ref{main_problem}, given $\bm p$, suppose that $ \gamma(\bm \theta^*_{\bm c}(\bm p)) = t (\bm r (\bm \theta^*_{\bm c}(\bm p)) ) $ is monotone in $c_i$ with fixed $c_{-i}$ for some $i$. Denote the value of $ t (\bm r (\bm \theta^*_{\bm c}(\bm p)) )$ at $c_i = 0$ by $t_{c_i=0}$ and that at $c_i = +\infty$ by $t_{c_i = +\infty}$. 
    \begin{itemize}
        \item[(1)] If $t_{c_i=0} < t_{c_i = +\infty} < t(\hat{\bm r})$ or the order is reversed, then $c_i^* = +\infty$.
        \item[(2)] If $t(\hat{\bm r}) < t_{c_i=0} < t_{c_i = +\infty}$ or the order is reversed, then $c_i^* = 0$.
        \item[(3)] If $t_{c_i=0} < t(\hat{\bm r}) < t_{c_i = +\infty}$ or the order is reversed, then $c_i^* \in (0,+\infty)$.
    \end{itemize}
\end{corollary}

Next, we will prove Theorem \ref{theorem-2D-general}.

\begin{proof}

First, because all sub-losses are accuracy-rewarding and the model curve $r_2 = R(r_1)$ is strictly monotone, the condition (A) is satisfied for all $\bm p$ according to Theorem \ref{theorem-conditionForSubproperties-2D}. Second, obviously, each of the condition of (a) (b) (c) can guarantee that $R'(r_1)- T'(r_1; t_0)$ keeps the sign unchanged. Recalling that all contours are differentiable and monotone, according to Theorem \ref{theorem-conditionForTargetProperty}, the condition (B) is satisfied, i.e., $T_R  (r_1) = t(r_1, R(r_1))$ is monotone. Thus, both the condition (A) and (B) are satisfied. Following from Theorem \ref{theorem-general}, we see that $ t (\bm r (\bm \theta^*_{\bm c}(\bm p)) ) $ is monotone in $c_1$ with fixed $c_{-1}$ i.e. $c_2$, for all $\bm p$.

$\bm r(\theta_{\bm c}^*)$ will always move from the point $\bm r^A \triangleq (R^{-1}(\hat{r}_2), \hat{r}_2)$ to the point $\bm r^B \triangleq (\hat{r}_1, R(\hat{r}_1) )$ when increasing $c_1$. Without loss of generality, assume that $R^{-1}(\hat{r}_2) < \hat{r}_1$, and $\frac{\partial t}{\partial r_2} > 0$.

(a) 

Without loss of generality, assume that $0 < R'(r_1)  < T'(r_1; t_0)$, for instance, as Figure \ref{fig_dR_dT_consistent} (a) shows. According to Eq(\ref{Eq-derivative-difference}), we know that 
\begin{equation*}
    T_R'(r_1) < 0.
\end{equation*}
It means that $T_R(r_1)$ is decreasing. Thus,  
\begin{equation} \label{Eq-contourLevels-compare-at-corners}
    t (\bm r^A) \triangleq t_A > t( \bm r^B)  \triangleq t_B  .
\end{equation}

Next, we need to compare the value between $t_B$ and $t(\hat{\bm r})$. 

Because $R^{-1}(\hat{r}_2) < \hat{r}_1$ and $R(r_1)$ is increasing, we have 
\begin{equation} \label{Eq-corners-r2-compare}
    \hat{r}_2 < R(\hat{r}_1).
\end{equation}
For the contour curve at the level $t_B$, its intersection with the segment between $\bm r^A$ and $\hat{\bm r}$ is $(T^{-1}(\hat{r}_2; t_B), \hat{r}_2)$. Because $T(r_1; t_B)$ is increasing and Eq(\ref{Eq-corners-r2-compare}), we have 
\begin{equation}
T^{-1}(\hat{r}_2; t_B) < T^{-1}(R(\hat{r}_1); t_B) = \hat{r}_1. 
\end{equation}

Because Eq(\ref{Eq-contour-derivative}), and $0 < T'(r_1; t_0)$, and $\frac{\partial t}{\partial r_2} > 0$, we have $\frac{\partial t}{\partial r_2} < 0$, which means that $t(r_1, r_2)$ is decreasing along $r_1$-axis. Thus, we have 
\begin{equation} \label{Eq-contourLevels-compare-to-rHat}
    t_B = t(T^{-1}(\hat{r}_2; t_B), \hat{r}_2) > t(\hat{r}_1, \hat{r}_2) = t(\hat{\bm r})
\end{equation}

According to Eq(\ref{Eq-contourLevels-compare-at-corners}) and Eq(\ref{Eq-contourLevels-compare-to-rHat}), we have 
\begin{equation}
    t_A = t (\bm r^A) > t_B = t(\bm r^B) > t(\hat{\bm r}).
\end{equation}

So, when increasing $c_1$, $\bm r(\theta_{\bm c}^*)=(r_1(\theta_{\bm c}^*), r_2(\theta_{\bm c}^*))$ will move from $\bm r^A$ to $\bm r^B$, coupled with the value $t(\bm r(\theta_{\bm c}^*))$ decreasing to $t(\hat{\bm r})$ (not reaching the value).

(b) 

Without loss of generality, assume that $R'(r_1)  > T'(r_1; t_0) >0 $, for instance, as Figure \ref{fig_dR_dT_consistent} (b) shows. According to Eq(\ref{Eq-derivative-difference}), we know that 
\begin{equation*}
    T_R'(r_1) > 0.
\end{equation*}
It means that $T_R(r_1)$ is increasing. Thus, 
\begin{equation*} 
    t_A \triangleq t (\bm r^A) < t( \bm r^B )  \triangleq t_B  .
\end{equation*}

Next, we need to compare the value between $t_A$ and $t(\hat{\bm r})$. 

For the contour curve at the level $t_A$, its intersection with the segment between $\hat{\bm r}$ and $\bm r^B$ is $(\hat{r}_1, T(\hat{r}_1; t_A))$. Because $T(r_1; t_A)$ is increasing, we have 
\begin{equation*}
\hat{r}_2 = T(R^{-1}(\hat{r}_2); t_A) < T(\hat{r}_1; t_A). 
\end{equation*}

Because $\frac{\partial t}{\partial r_2} > 0$, we have 
\begin{equation*} 
    t(\hat{\bm r}) = t(\hat{r}_1, \hat{r}_2) < t(\hat{r}_1, T(\hat{r}_1; t_A) ) = t_A.
\end{equation*}

Finally, 
\begin{equation*} 
    t(\hat{\bm r}) < t_A = t (\bm r^A) < t( \bm r^B) = t_B  .
\end{equation*}

So, when increasing $c_1$, $\bm r(\theta_{\bm c}^*)=(r_1(\theta_{\bm c}^*), r_2(\theta_{\bm c}^*))$ will move from $(R^{-1}(\hat{r}_2), \hat{r}_2)$ to $(\hat{r}_1, R(\hat{r}_1) )$, coupled with the value $t(\bm r(\theta_{\bm c}^*))$ increasing from $t(\hat{\bm r})$ (not reaching the value).

(c)

Without loss of generality, assume that $R'(r_1) <0 < T'(r_1; t_0) $, for instance, as Figure \ref{fig_dR_dT_consistent} (c) shows. According to Eq(\ref{Eq-derivative-difference}), we know that 
\begin{equation*}
    T_R'(r_1) < 0.
\end{equation*}
Thus we have
\begin{equation*} 
    t_A \triangleq t (\bm r^A) > t( \bm r^B )  \triangleq t_B  .
\end{equation*}

Next, we need to prove that $R(r_1)$ and $T(r_1; t(\hat{\bm r}))$ has a unique intersection in the open interval $(R^{-1}(\hat{r}_2), \hat{r}_1)$.

\[
D(r_1) \triangleq R(r_1) - T(r_1; t(\hat{\bm r})).
\]
Then
\[
D'(r_1) = R'(r_1) - T'(r_1; t(\hat{\bm r})) < 0 .
\]

Because $T(r_1; t(\hat{\bm r}))$ is increasing, we have $T(R^{-1}(\hat{r}_2); t(\hat{\bm r})) < T(\hat{r}_1; t(\hat{\bm r})) = \hat{r}_2$. Thus, 
\[
D(R^{-1}(\hat{r}_2)) = R(R^{-1}(\hat{r}_2)) - T(R^{-1}(\hat{r}_2); t(\hat{\bm r})) > 0. 
\]

Because $R(r_1)$ is decreasing, we have $R(\hat{r}_1) < R(R^{-1}(\hat{r}_2)) = \hat{r}_2$. Thus
\[
D((\hat{r}_1) = R(\hat{r}_1) - T((\hat{r}_1; t(\hat{\bm r})) < 0.
\]

So, there must be a unique $r_1^0 \in (R^{-1}(\hat{r}_2), \hat{r}_1)$ such that $D(r_1^0)=0$, i.e. 
$R(r_1^0) = T(r_1^0; t(\hat{\bm r}))$. So we have 
\[
t(r_1^0, R(r_1^0)) = t(r_1^0, T(r_1^0; t(\hat{\bm r}))) = t(\hat{\bm r}).
\]
Then we have
\[
t_A > t(\hat{\bm r}) > t_B
\]

So, when increasing $c_1$, $\bm r(\theta_{\bm c}^*)=(r_1(\theta_{\bm c}^*), r_2(\theta_{\bm c}^*))$ will move from $\bm r^A$ to $\bm r^B$, coupled with the value $t(\bm r(\theta_{\bm c}^*))$ decreasing through $t(\hat{\bm r})$.
\end{proof}

\section{Discussion on the general sub-property space} 
\label{appendix-generalSubproperty-2D}

\begin{figure}[!htbp]
  \centering
  \includegraphics[width=0.45\textwidth]{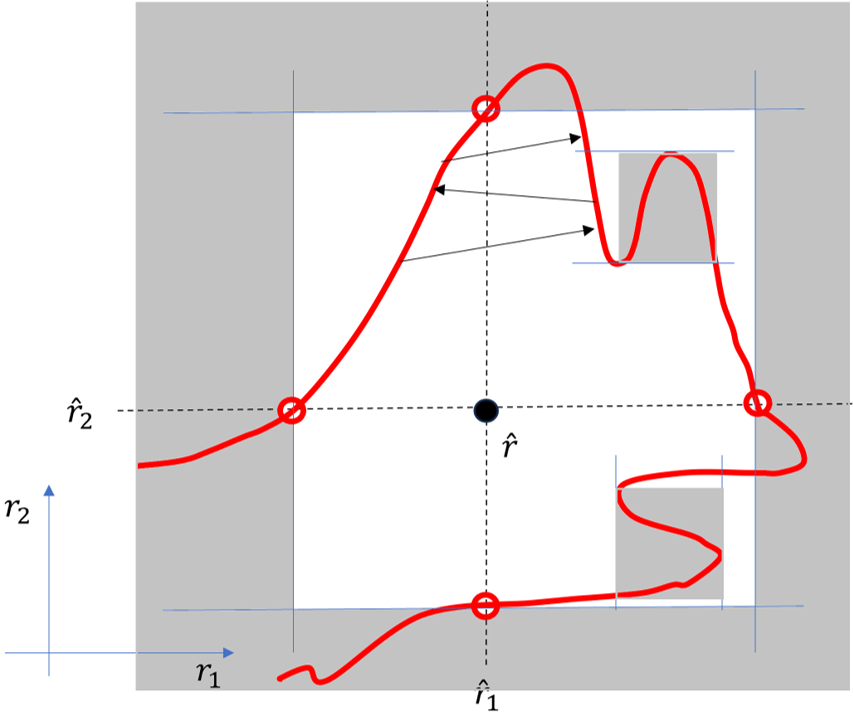} 

  \caption{ The impossible area of $\bm r(\theta^*_{\bm c})$ for all $\bm c$ with a general  $\mathcal{R}_{\Theta}$ in a 2-D space, marked with gray shading. } 
  \label{fig_2D_model_general_illustration} 

\end{figure}

In Section \ref{section-2D-conditionForSubproperties}, we have proved in Theorem \ref{theorem-conditionForSubproperties-2D} that a strictly monotone $\bm r (\theta)$ in 2-D cases makes the possible area of $\bm r(\theta^*_{\bm c})$ for all $\bm c$ intrinsically restricted to one quadrant centered at the true value of sub-properties $\hat{\bm r}$. Now we discuss the general structure of $\bm r (\theta)$ such as the example shown in Figure \ref{fig_2D_model_general_illustration}. We have the following theorem. 

\begin{theorem}\label{theorem-SubpropertiesPossibleArea-2D}
For Problem \ref{main_problem} with $M=2$, suppose that $\theta^*_{\bm c}(\bm p) $ exists for all $\bm p$ and $\bm c$ and all sub-losses are accuracy-rewarding. For any given $\bm p$, the possible area for candidates of $\bm r(\theta^*_{\bm c})$ for all $\bm c$ would be composed of strictly monotone pieces within each quadrant centered at $\hat{\bm r} (\bm p)$, and those pieces are not dominated by the intersections between the model curve $\mathcal{R}_{\Theta}$ and the coordinate axes centered at $\hat{\bm r} (\bm p)$. 
\end{theorem}

\begin{proof}
    The proof is straightforward and can be illustrated by a general example of $\mathcal{R}_{\Theta}$ as shown in Figure \ref{fig_2D_model_general_illustration}. First, within each quadrant, the non-monotone portion of $\bm r(\theta)$ can be trimmed to be monotone. If $\bm r(\theta)$ and $\bm r(\tilde{\theta})$ have the same value on one of the coordinates, then the point with the other coordinate farther away from $\hat{r}_i$ incurs a higher loss and thus cannot be a candidate of $\bm r(\theta^*_{\bm c})$ for each $\bm c$. Second, with the similar proof used in Theorem \ref{theorem-conditionForSubproperties-2D}, we see that the outer surrounding shaded area can also be ruled out for candidates of $\bm r(\theta^*_{\bm c})$ for all $\bm c$. 
\end{proof}

Theorem \ref{theorem-conditionForSubproperties-2D} says that $\bm r(\theta)$ corresponding to a strictly monotone function can guarantee the condition (A) no matter what the accuracy-rewarding sub-losses are, and in the meantime with Theorem \ref{theorem-SubpropertiesPossibleArea-2D} we can show how the strict monotonicity assumption about $\mathcal{R}_{\Theta}$ can be used to handle general situations. For instance as shown in Figure \ref{fig_2D_model_general_illustration}, if the remaining possible area passes through more than one quadrant, then $\bm r(\theta^*_{\bm c})$ might jump among different quadrants when changing the value of $\bm c$, as illustrated by the arrows in the top middle of Figure \ref{fig_2D_model_general_illustration}. But, in order to reach a final conclusion about what $c_i^*$ is for all $i$, it suffices to only consider the remaining strictly monotone pieces of $\bm r(\theta)$ within one quadrant at a time and then consider the results in all quadrants together. This lets us decompose a complex situation into simpler, strictly monotone ones. Combining all the above insights, we can see that the study of strictly monotone $\bm r (\theta)$ is elementary.


\section{Proof for M-D cases} \label{appendix-proofs-MD}

\subsection{Proof for the condition (A)} \label{appendix-proofs-MD-subproperties}

Here, we will prove Theorem \ref{theorem-conditionForSubproperties-MD}. We mentioned before that we can apply the induction method. The base case is for $M=2$ and we have proved it in Theorem \ref{theorem-conditionForSubproperties-2D}. In Section \ref{section-3D}, we have also shown how the conclusion holds for $M=3$. Now we use the similar trick to prove that the conclusion will hold for $M= m $ if it holds for $M = m-1 \geq 2$ for any $m \geq 3$.

\begin{proof}
(To obtain an easier understanding of the proof, please first refer to the proof for the case of $M=3$ shown in Section \ref{section-3D}).

{\bf Base case $M=2$.} It has been proved in Theorem \ref{theorem-conditionForSubproperties-2D}.

{\bf Induction step.} Suppose that the conclusion holds true for $M=m-1$ with any $m\geq 3$, then we need to show that the conclusion hold for $M= m$. 

The axis-$r_i$ centered at $\hat{\bm r} $ corresponds to the precise value of $\hat{r}_{-i}$ for each $i$. Denote the intersection between $\mathcal{R}_{\Theta}$ and axis-$r_i$ centered at $\hat{\bm r} (\bm p)$ by $A_i$, the intercept of $A_i$ on axis-$r_i$ by $\varepsilon_i$, and then $A_i = (\hat{r}_i + \varepsilon_i, \hat{r}_{-i} ) $. Without loss of generality, we assume that $\varepsilon_i > 0$ for all $i$.

Now we slice $\mathcal{R}_{\Theta}$ orthogonal to axis-$r_j$ at $r'_j$ for each $j$ and $r'_j$. Denote the slicing hyperplane at any $r'_j$ by $H_{r'_j} = \{ \bm r | r_j = r'_j \}$ and the sliced $\mathcal{R}_{\Theta}$ by $S_{r'_j} = \{ \bm r (\bm \theta) | r_j(\bm \theta) = r'_j \}$. 

Use $\varepsilon'_k (r'_j)$ for $k \neq j$ to denote the intercept of the slice $S_{r'_j}$ on each axis of the hyperplane $H_{r'_j}$ centered at $(\hat{r}_{-j}, r'_j)$. Denote $\Delta r_k \triangleq r_k - \hat{r}_k$ for all $k \neq j$.
    
We have the following two claims about $S_{r'_j}$. 
\begin{itemize}
    \item[(1)] When $r'_j \geq \hat{r}_j + \varepsilon_j$, the slice $S_{r'_j}$ is above the intersection $A_j$, and thus is dominated by $A_j$ and can be ruled out for the optimal estimation for all $\bm c$. 
    \item[(2)] When $r'_j < \hat{r}_j + \varepsilon_j$, the slice $S_{r'_j}$ is strictly below the intersection $A_j$. We can verify that the intercept of the slice $S_{r'_j}$ on each axis of the hyperplane $H_{r'_j}$, that is, $\varepsilon'_k (r'_j)$ keeps the same sign of $\varepsilon_k$ for all $k \neq j$ due to the strict monotonicity assumption specified in the theorem. 

    Since $M= m$, the slice $S_{r'_j}$ on the hyperplane $H_{r'_j}$ would have one-lower dimension than $\mathcal{R}_{\Theta}$. Also, with being restricted to $S_{r'_j}$, the sub-loss $L_j(r_j (\bm \theta), \bm p)$ is a constant and thus can be ignored in the total loss. So looking for candidates of $\bm r (\bm \theta_{\bm c}^*)$ within $S_{r'_j}$ is equivalent to a $(m-1)$ dimensional case. Based on our assumption, the possible area of $\bm r (\bm \theta_{\bm c}^*)$ for all $\bm c$ within each slice $S_{r'_j}$ with $r'_j < \hat{r}_j + \varepsilon_j$ would be in the portion $\{ \bm r | r_j = r'_j, \Delta r_k >0, k \neq j \}$. 
\end{itemize}

Combining (1) and (2) above, we see that the possible area of $\bm r (\bm \theta_{\bm c}^*)$ for all $\bm c$ would be in the portion $\{ \bm r (\bm \theta) | r_k(\bm \theta) \geq \hat{r}_k, k \neq j \}$. Because of this, we see that in total the possible area of $\bm r (\bm \theta_{\bm c}^*)$ for all $\bm c$ would be in the portion $\{ \bm r (\bm \theta) | r_k(\bm \theta) \geq \hat{r}_k, \forall k \}$
    
Combing the base case and the induction step, we see that the conclusion hold for all $m \geq 2$.

\end{proof}

\subsection{Proof for the linear special case} \label{appendix-proofs-MD-linearCases}

This section provides the proof for Theorem \ref{theorem-linearCases}.

\begin{lemma}\label{lemma-linearCases}
Suppose that $\Gamma = t(\bm r)$ with $t: \mathbb{R}^M \rightarrow \mathbb{R}$ ($M \geq 2$) is a linear function, and $\mathcal{T}$ is a line. Then $\Gamma = t(\bm r)$ is constant or strictly monotone along the line $\mathcal{T}$.
\end{lemma}

\begin{proof}
If there are any two different points on the line $\mathcal{T}$ leading to the same value of $\Gamma = t(\bm r)$, then they are on an isofurface of $\Gamma = t(\bm r)$. Any two points are on a line and thus on one hyperplane. Since $\Gamma = t(\bm r)$ is linear, its isosurfaces are parallell hyperplanes on $\mathbb{R}^M$ with the same normal vector. Consequently, the line spanned by the two points i.e. $\mathcal{T}$ is in the linear isosurface. 

If no two points on the line $\mathcal{T}$ share the same value of $\Gamma = t(\bm r)$, then it means $\Gamma = t(\bm r)$ is strictly monotone along the line $\mathcal{T}$.
\end{proof}

Now we will prove Theorem \ref{theorem-linearCases}  .

\begin{proof}
    First, if $\bm r(\bm \theta)$ is linear, then obviously it satisfies the assumption in Theorem \ref{theorem-conditionForSubproperties-MD}. So, the optimal estimation for sub-properties $\bm r (\bm \theta^*_{\bm c} (\bm p))$ would be in the same orthant centered at $\hat{\bm r} (\bm p)$ for any given $\bm p$ and all $\bm c$, that is, the condition (A) is satisfied in Theorem \ref{theorem-general}
    
    Then, all trajectories $\mathcal{T}_{c_i} (\bm p))$ are linear, so they are lines. Following from Lemma \ref{lemma-linearCases}, we see that the condition (B) is also satisfied. 

    Because the condition (A) and (B) are both satisfied, the conclusion in Theorem \ref{theorem-general} holds.
\end{proof}


\section{Experiments for the skewness property} \label{appendix-section-experimentsForSkewness}


\begin{table}[!htbp] 
    \caption{ The experiment design for setting base weights $1./\bm k$ with $\Gamma = skew(\cdot)$. (---- means that $\bm p_0$ is in the same family as $\bm q_{\bm \theta}$.) }
    \vspace{0.01pt}
    \centering
    \begin{tabular}{c | c | c | c | c}
        \hline
        $\Gamma$ & $\hat{\bm r}$ & $t(\bm r)$ & $\bm q_{\bm \theta}$ & $\bm p_{0}$ \\
        \hline
        \multirow{9}*{$skew(\cdot)$} & \multirow{9}*{$(E[X], E[X^2], E[X^3])$} & \multirow{9}*{$\frac{r_3 - 3r_1 (r_2 -r_1^2) - r_1^3}{(r_2 -r_1^2)^{3/2}}$} &\multirow{3}*{ $LogN(u,v^2)$} & $LogN(u_0,v_0^2)$ \\
        \cline{5-5}
        & & & & $|N(\mu_0,\sigma_0^2)|$ \\
        \cline{5-5}
        & & & & $\sum_i LogN(u_i,v_i^2)$ \\
        \cline{4-5}
        & & & \thead[c] {$LogLogistic(a, b)$\\ $a>0$, $b>3$} & ----\\
        \cline{4-5}
        & & & \thead[c] {$Gamma(a, b)$\\ $a>0$, $b>0$} & ---- \\
        \cline{4-5}
        & & & \thead[c] {$Beta(a, b)$\\ $a>0$, $b>0$} & ---- \\
        \hline
        
                
    \end{tabular}
        \label{table-experiment-design-for-skewness}
\end{table}

Table \ref{table-experiment-design-for-skewness} shows the specific experiment setting for $\Gamma = skew(\cdot)$. Here, we try four well-known probability distribution models. For log-normal model, we try three different classes of distributions for $\bm p_0$ to see how the choice of $\bm p_0$ affects results.

$X\sim |N(\mu_0,\sigma_0^2)|$ means that $X = |Y|$ and $Y\sim N(\mu_0,\sigma_0^2)$. We use $|N(\mu_0,\sigma_0^2)|$ instead of $N(\mu_0,\sigma_0^2)$, because $\bm q_{\theta} = LogN(u,v^2)$ assumes that the random variable only takes positive values while about a half of samples from $N(\mu_0,\sigma_0^2)$ would be negative. It is too improper to fit the distribution $N(\mu_0,\sigma_0^2)$ into the model $\bm q_{\theta} = LogN(u,v^2)$, causing weird results. We will also show some evidence for that later. $X\sim \sum_i LogN(u_i,v_i^2)$ means that $X = \sum_i Y_i$ and $Y_i \sim LogN(u_i,v_i^2)$ for each $i$. 


As mentioned in Section \ref{section-3D-empiricalAnalysis}, to relieve the issue of optimization difficulties, we need to remove the dominance of poorly-behaved sub-losses with renormalizing the weights in the way $\tilde{\bm c} = \bm c ./ \bm k$, where $./$ refers to element-wise division and $1/\bm k$ with $\bm k \in \mathbb{R}_{++}^M$ characterizes the base weights to balance sub-losses. Originally, we assume all elements of $\bm k $ to be $1$. But later, we try setting $k_i = \hat{r}^2(\hat{\bm p})$ for all $i$ for the experiments on skewness property. More details about the finding of weight renormalization can be seen in Appendix \ref{appendix-LogN-optimizationDifficulties} and \ref{appendix-weightRenormalization}.

Figure \ref{fig-property-c-curves-skew-qLogN} $\sim$ \ref{fig_Beta_Gamma} show the major results with different settings of $\bm q_{\bm \theta}$ and $\bm p_0$. Again, $\gamma(\bm \theta_{\bm c}^*)$ is monotone in most tested cases.

\begin{figure*}[t]
  \centering
  \subfigure[$\bm p_0 = LogN(0,3^2)$]{\includegraphics[width=0.3\textwidth]{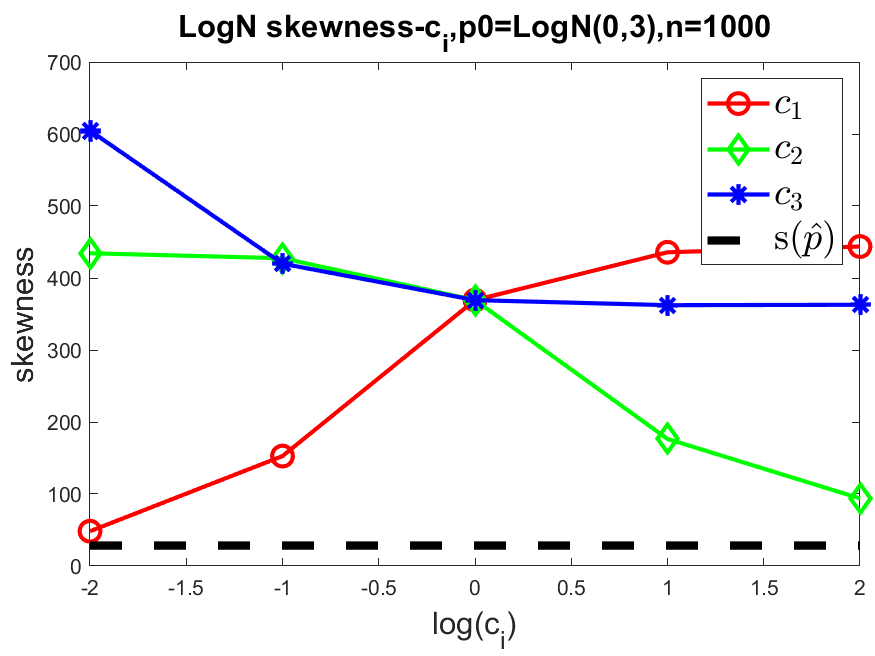}}
  \subfigure[$\bm p_0 = |N(0,1^2)|$]{\includegraphics[width=0.3\textwidth]{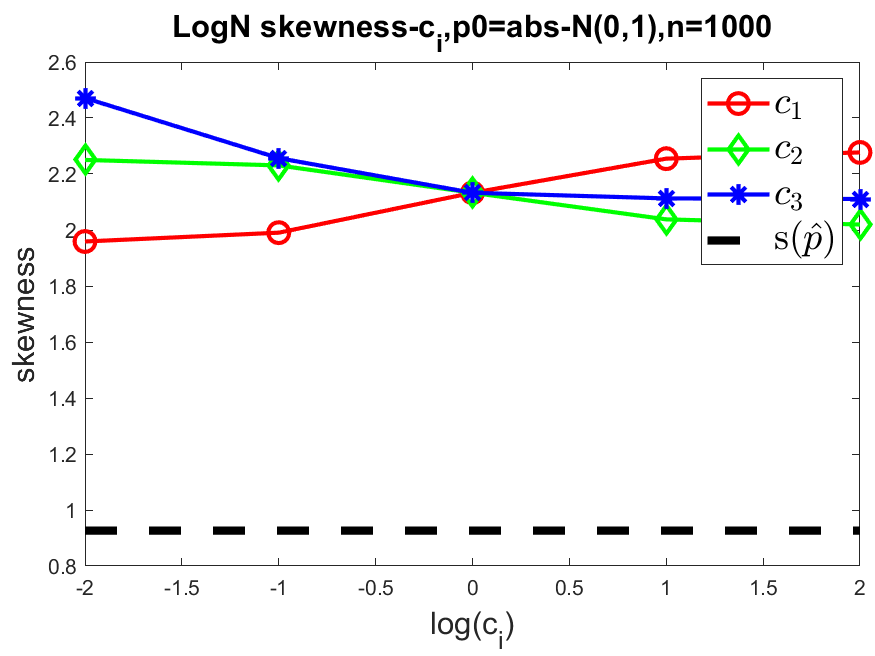}}
  \subfigure[$\bm p_0 = \sum LogN(u_i,v_i^2)$]{\includegraphics[width=0.3\textwidth]{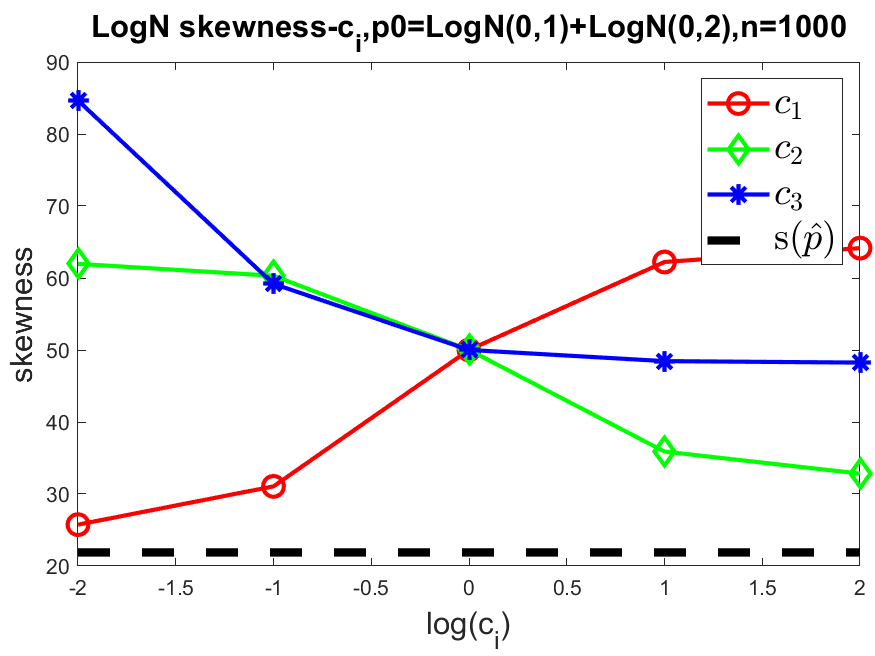}}
  
  \caption{ The $\Gamma_{c_i}$ curves for the cases of $\Gamma = skew(\cdot)$ and $\bm q_{\bm \theta} =  LogN(u,v^2)$ with different $\bm p_0$. } 
  \label{fig-property-c-curves-skew-qLogN}
\end{figure*}

\begin{figure*}[t]
  \centering
  \subfigure[$\bm q_{\bm \theta} = LogLogistic$]{\includegraphics[width=0.3\textwidth]{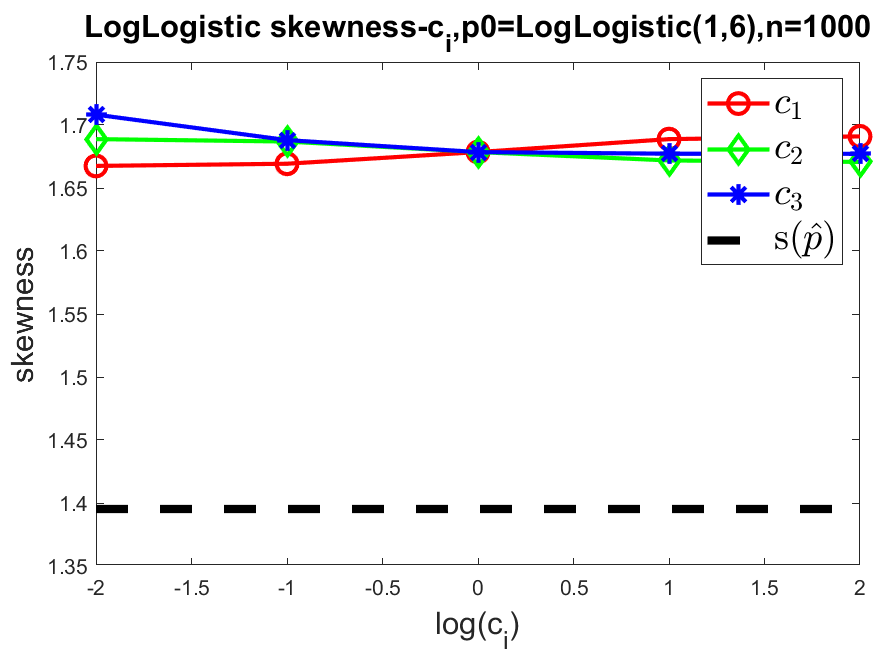}}
  \subfigure[$\bm q_{\bm \theta} = Gamma$]{\includegraphics[width=0.3\textwidth]{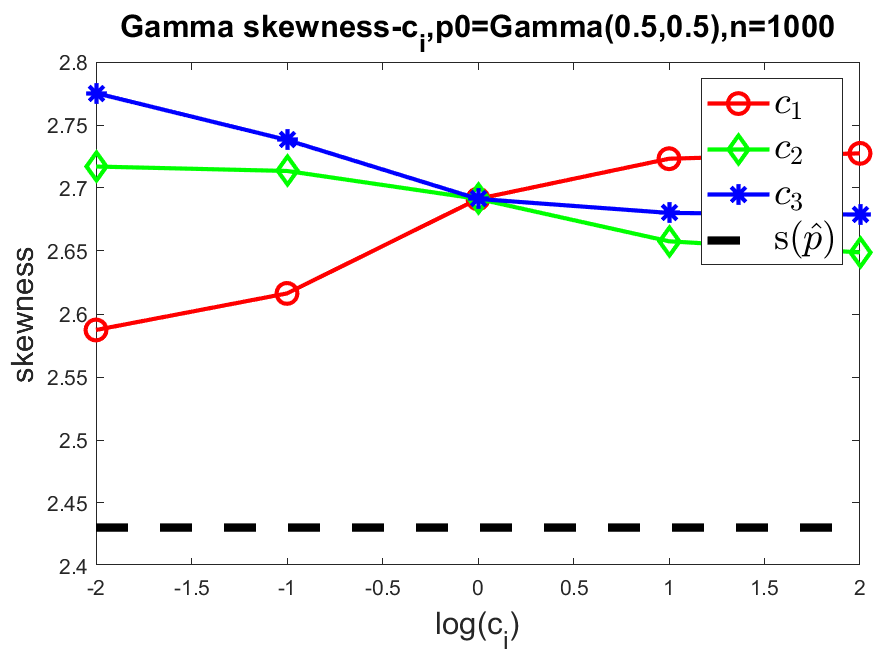}}
  \subfigure[$\bm q_{\bm \theta} = Beta$]{\includegraphics[width=0.3\textwidth]{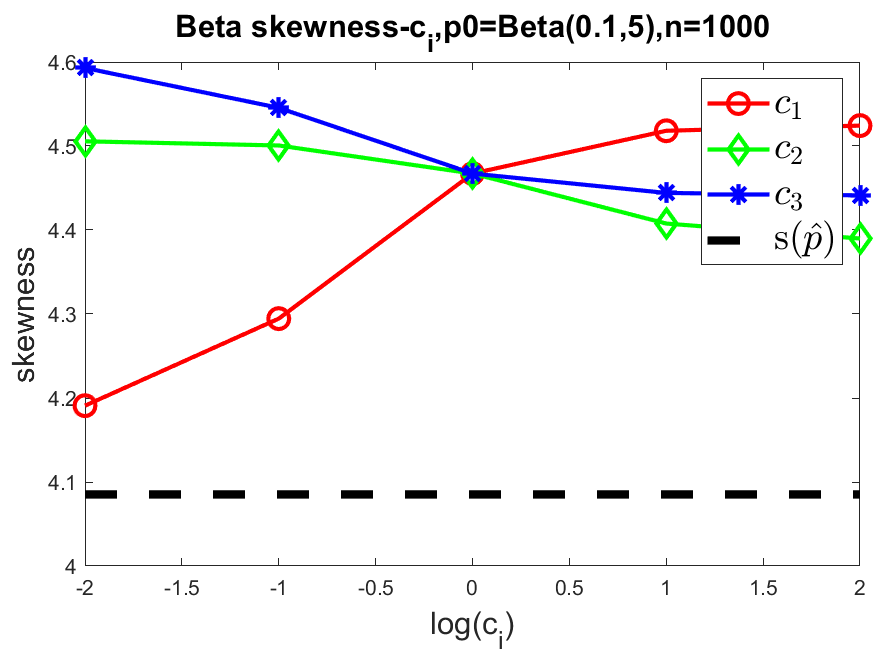}}
  
  \caption{ The $\Gamma_{c_i}$ curves for the cases of $\Gamma = skew(\cdot)$ with different $\bm q_{\bm \theta}$. } 
  \label{fig-property-c-curves-skew}
\end{figure*}

\begin{figure}[!htbp]
  \centering
  \subfigure[$b_0 = 4$]{\includegraphics[width=0.3\textwidth]{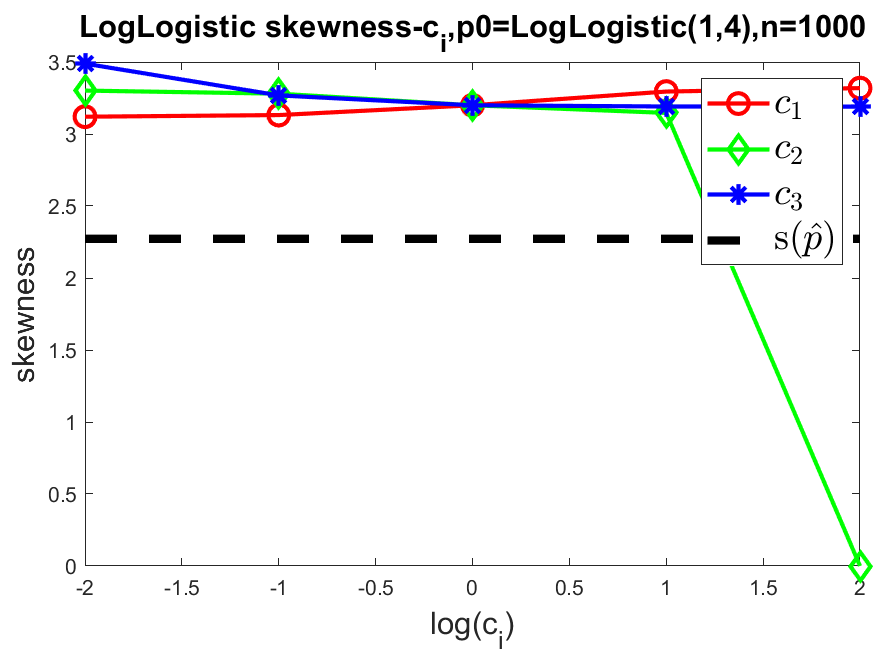}}
  \subfigure[$b_0 = 6$]{\includegraphics[width=0.3\textwidth]{LogLogistic_skewness_c_1D,p0=LogLogistic1,6,n=1000_random3_new2.png}}
  \subfigure[$b_0 = 8$]{\includegraphics[width=0.3\textwidth]{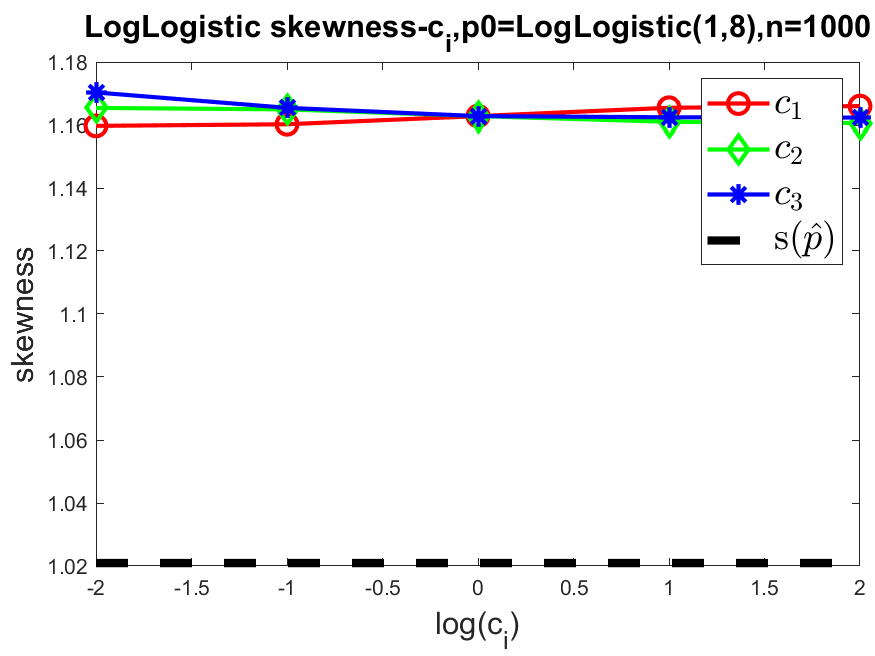}}
  
  \caption{ The $\Gamma_{c_i}$ curves for the cases of $\Gamma = skew(\cdot)$ and $\bm q_{\theta} = LogLogistic(a, b)$ with different $\bm p_0 = LogLogistic(1, b_0)$. }\label{fig_LogLogistic}
\end{figure}

\begin{figure}[!htbp]
  \centering
  \includegraphics[width=0.95\textwidth]{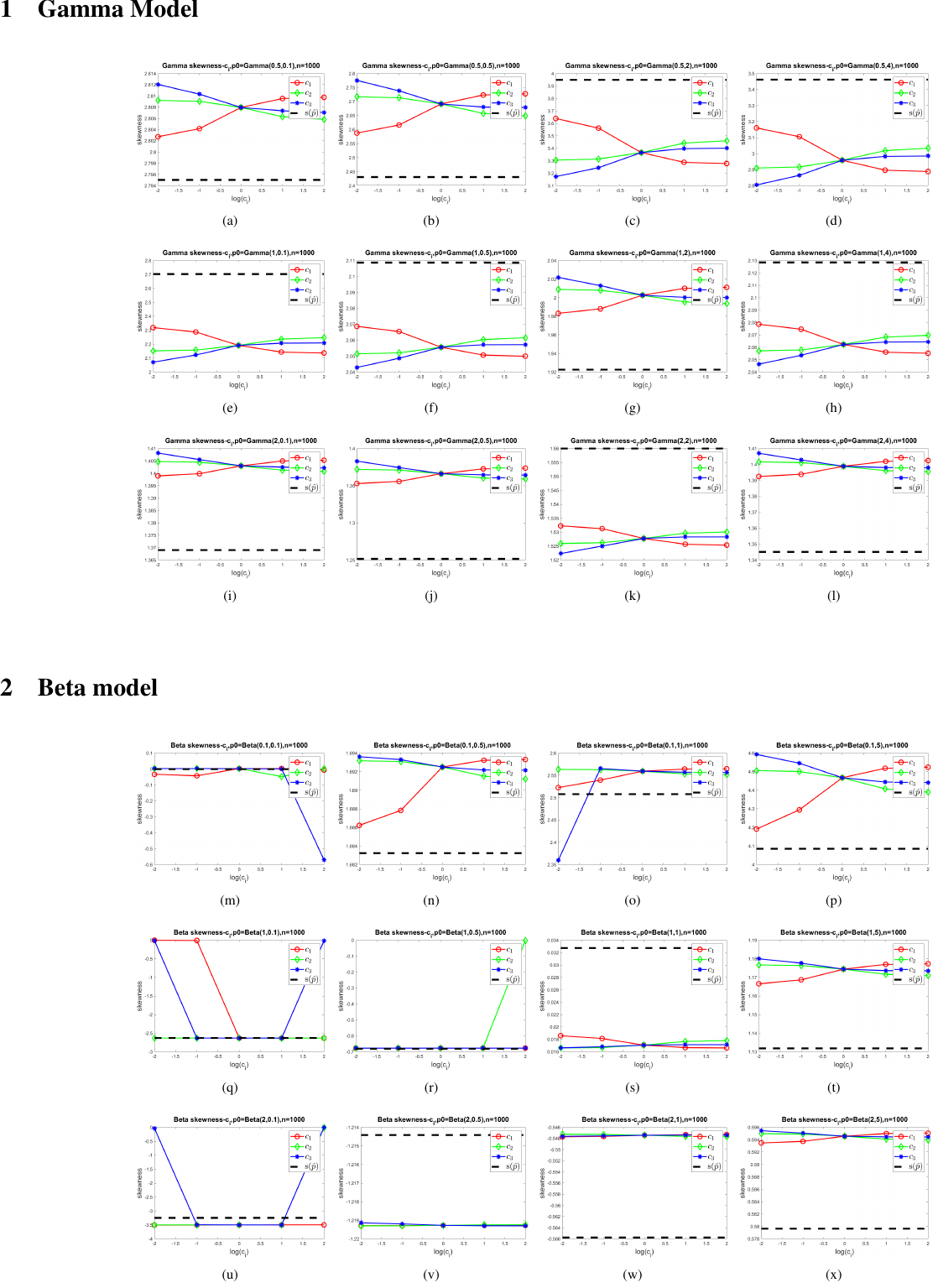}
  \caption{ The $\Gamma_{c_i}$ curves for the cases of $\Gamma = skew(\cdot)$, $\bm q_{\theta} = Gamma(a, b)$ or $\bm q_{\theta} =Beta(a, b)$ with different $\bm p_0 $. Each sub-figure corresponds to a specific setting for the parameter values $(a_0, b_0)$ of $\bm p_0$. }\label{fig_Beta_Gamma}
\end{figure}


\section{Details about the optimization difficulties for the log-normal model} \label{appendix-LogN-optimizationDifficulties}

For the log-normal model of distributions, we find that the numerical optimal solutions $\bm \theta_{\bm c}^*$ given by the regular gradient descent algorithm are sensitive to the choice of the initial iteration point $\bm \theta^{(0)}$. Next, we will show how the issue is found and what the root reason is. 

This issue is initially noticed from the $\Gamma_{c_i}$ curves for $\Gamma = skew$, $\bm q_{\bm \theta}=LogN(u,v^2)$, $\bm p_0=LogN(0,1^2)+LogN(0,2^2)$, as shown in Figure \ref{fig-initialPoints-debugging-skew-qLogN-p0sumLogN}. It is obvious that different choices of $\bm \theta^{(0)}$ lead to very different optimization results. We tried smaller error tolerances, larger maximum iteration steps, multi-starts, and even different optimization algorithms such as interior point iteration algorithm, but none of them helps. Actually, the algorithm always stops much earlier before reaching the maximum iterations. 

At this point, one might guess that there are multiple local minima and the algorithm is blocked by local minima, but it is not precise here. We will discover the truth step by step. 

We intentionally choose one of the initial points to be the total loss minimizer on a divided meshgrid around the global minimizer, since such a choice should have been a good guess about the global minimizer and the optimization result should have been trusting. However, when we check the optimization detail further, we find that such a ``good'' guess is still not convincing. We change the scale of the meshgrid, that is, the width of each cube, and then we get very different minimizers $\bm \theta_s^*$ and $\bm \theta_d^*$ where the sub-index $s$ represents a sparser grid and ``d'' represents a denser grid, and the corresponding values of $\Gamma$ are also different. This is a very bizzare but interesting phenomenon. Table \ref{table-initialPoint-debugging-optimizationDetail-skew-qLogN} shows the numerical detail of two examples of $\bm p_0$. The sensitivity w.r.t. $\bm \theta^{(0)}$ is stronger in the second example $\bm p_0=LogN(0,3^2)$. Then we want to figure out what happens exactly here. 

\begin{figure}[!htbp]
  \centering
  \subfigure[$\bm \theta^{(0)} = (0,0)$]{\includegraphics[width=0.3\textwidth]{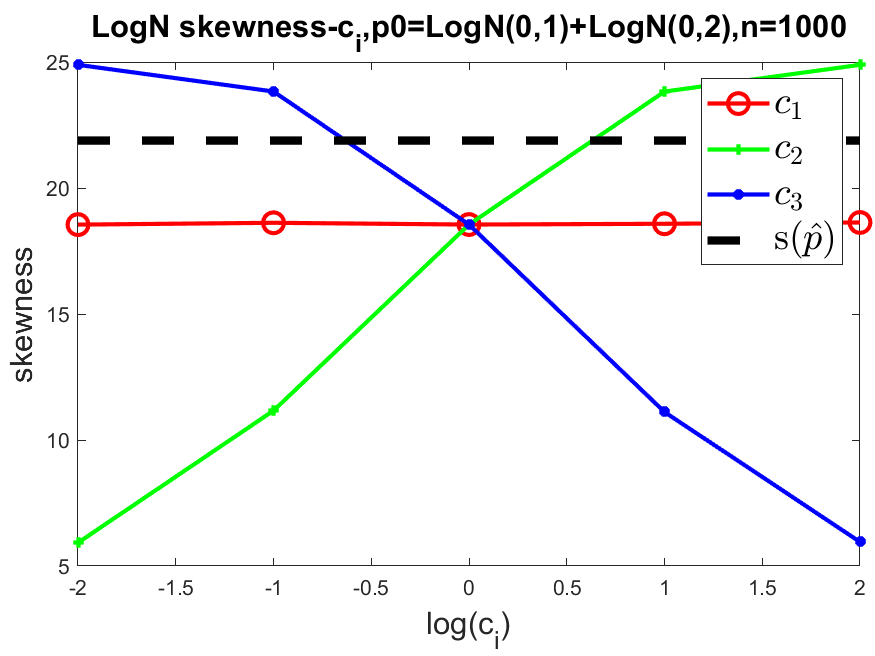}}
  \subfigure[$\bm \theta^{(0)} = (0,1.5)$]{\includegraphics[width=0.3\textwidth]{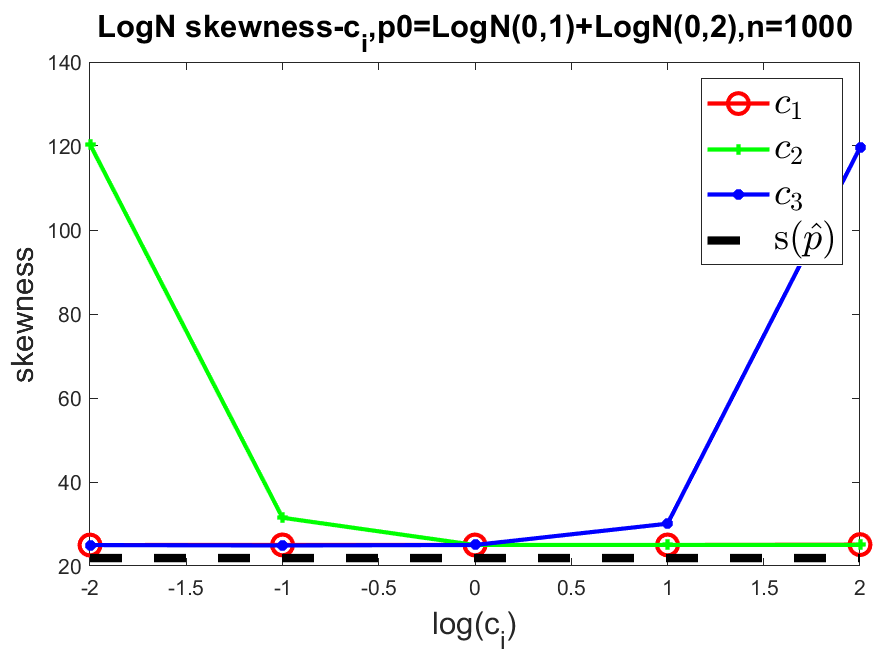}}
  \subfigure[$\bm \theta^{(0)} = \bm \theta_0^*$]{\includegraphics[width=0.3\textwidth]{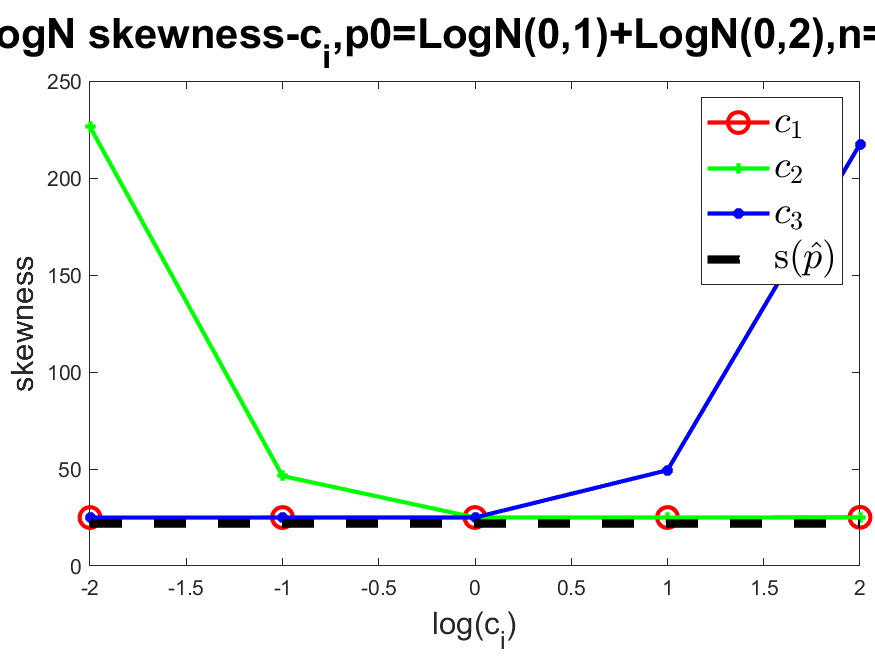}}
  \caption{ The $\Gamma_{c_i}$ curves with different initial points for the gradient descent algorithm. It shows that the numerical optimization algorithm is sensitive to the choice of the initial point for iterations. Here, $\Gamma = skew(\cdot)$, $\bm q_{\bm \theta} = LogN(u, v^2)$ with $\bm \theta =(u, v)$, $\bm p_0 = LogN(0,1^2)+LogN(0,2^2)$. The initial point $ \bm \theta_0^*$ for Figure (c) is the minimizer of the total loss on a coarse meshgrid division around the global minimizer with cubes of width 0.1, which is regarded as a good guess about the global minimizer. }\label{fig-initialPoints-debugging-skew-qLogN-p0sumLogN}
\end{figure}

\begin{table}[!htbp] 
    \caption{ Original experimental results with different initial iteration points. Here, $\Gamma = skew(\cdot)$, $\bm q_{\theta} = LogN(u, v^2)$ with $\bm \theta =(u, v)$. $\tilde{\bm c} = (1,1,1)$. $\bm \theta_s^*$ and $\bm \theta_d^*$ are coarse minimizers of the total loss on a meshgrid division around the global minimizer with cubes of different widths respectively. $\bm \theta_s^*$ is from the sparser meshgrid with width-0.1 cubes, and $\bm \theta_d^*$ is from the denser meshgrid with width-0.01 cubes. } 
    \vspace{0.01pt}
    \centering
    \begin{tabular}{c | c | c | c}
		\hline
		\multicolumn{4}{c}{$\bm p_0 = LogN(0,1^2)+LogN(0,2^2)$} \\
		\multicolumn{4}{c}{$\hat{\bm r} = (8.38,991.87,1.89\times 10^6)
$} \\
        \hline
        $\bm \theta^{(0)})$ & $\bm \theta^*$ & $\Gamma=skew$ & $\mathcal{L}_{\tilde{\bm c}}^*$ \\
		\hline
		$(0, 0)$ & $(2.09,-1.35)$ & $18.55$ & $-3.58 \times 10^{12}$ \\
		$(0, 1.5)$ & $(1.76,1.43)$ & $25.01$ & $-3.58\times 10^{12}$ \\
		$\bm \theta_s^*$ & $(2.09,-1.35)$ & $18.57$ &  $-3.58\times 10^{12}$ \\
		$\bm \theta_d^*$ & $(1.76,-1.43)$ & $25.01$ &  $-3.58\times 10^{12}$ \\

        \hline                
    \end{tabular} 
   

    \begin{tabular}{c | c | c | c | c | c | c}
		\hline
		\multicolumn{7}{c}{$\bm p_0 = LogN(0,3^2)$} \\
		\multicolumn{7}{c}{$\hat{\bm r} = (121.15,5.02\times 10^6, 3.20\times 10^{11})$} \\
        \hline
        $\bm \theta^{(0)}$ & $\bm \theta^*$ & $\Gamma=skew$ & $\mathcal{L}_{\tilde{\bm c}}^*$ & $L_1 (\bm r(\bm \theta^*))$ & $L_2 (\bm r(\bm \theta^*))$ & $L_3 (\bm r(\bm \theta^*))$ \\
		\hline
		$(0, 0)$ & $(6.36,0.00)$ & $0.00$ & 
\textcolor{red}{$-1.25\times 10^{20}$} & $1.96\times 10^{5}$ & $-3.27\times 10^{12}$ & 
\textcolor{red}{$-1.25\times 10^{20}$} \\
		$(0, 3)$ & $(6.52,0.85)$ & $4.21$ & 
\textcolor{red}{$-5.20\times 10^{21}$} & $7.17\times 10^{5}$ & $-1.59\times 10^{13}$ & 
\textcolor{red}{$-5.20\times 10^{21}$} \\
		$\bm \theta_s^*=(4.5, 1.7)$ & $(4.5,1.7)$ & $82.42$ & 
\textcolor{red}{$-1.02\times 10^{23}$} & $5.33\times 10^{4}$ & $-1.94\times 10^{13}$ & 
\textcolor{red}{$-1.02\times 10^{23}$} \\
		$\bm \theta_d^*=(2.89, 1.99)$ & $(2.89,1.99)$ & $390.70$ & \textcolor{red}{$-1.02\times 10^{23}$} & $-1.46\times 10^{4}$ & $-8.14\times 10^{12}$ & 
\textcolor{red}{$-1.02\times 10^{23}$} \\
        \hline
        
                
    \end{tabular}
    \label{table-initialPoint-debugging-optimizationDetail-skew-qLogN}
\end{table}

\begin{figure}[!htbp]
  \centering
  \subfigure[zooming-out view]{\includegraphics[width=0.3\textwidth]{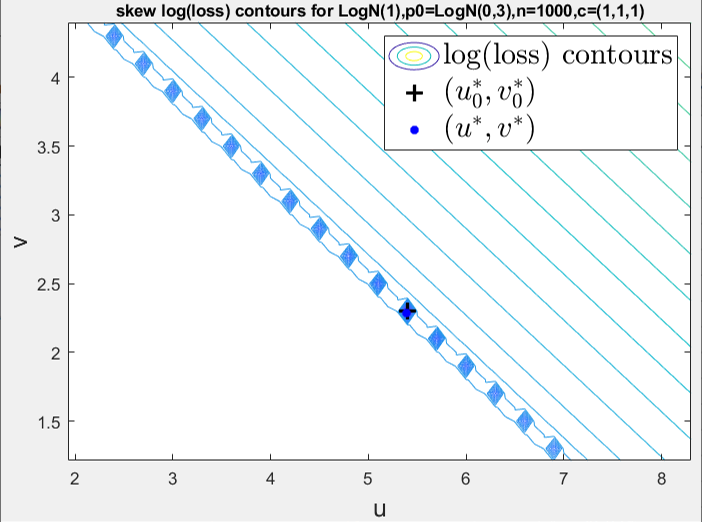}}
  \subfigure[zooming-in view]{\includegraphics[width=0.3\textwidth]{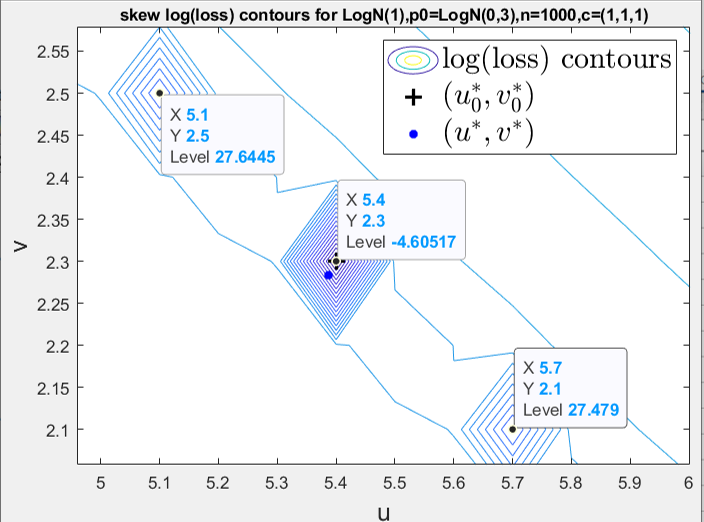}}
  \subfigure[The partial list of sorted loss values on a meshgrid]{\includegraphics[width=0.25\textwidth]{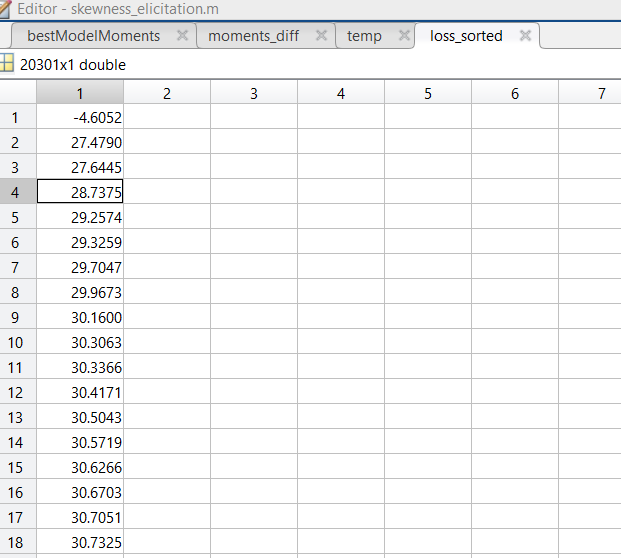}}
  
  \caption{ The contours of $\log(\mathcal{L}_{\tilde{\bm c}})$ with $\bm c = (1,1,1)$. Here, $\Gamma = skew(\cdot)$, $\bm q_{\theta} = LogN(u, v^2)$ with $\bm \theta =(u, v)$, $\bm p_0 = LogN(0,3^2)$. (Notice that the vertical axis refers to $v^2$ instead of $v$ itself.) The loss contours are plot based on the total loss values on the meshgrid of width 0.1 generated on the plane of $(u,v^2)$. The initial point $ \bm \theta^{(0)} = (u_0^*, v_0^*)$ is the loss minimizer on the meshgrid. The sub-figure (a) shows that the contours are nearly linear. The sub-figure (b) shows the contour detail around the global minimizer. The sub-figure (c) shows the sorted loss values on the meshgrid. Be careful that in (b) the concentric diamond loops are not the real contours but an illusion due to the visualization difficulty for the contour at the level near the global minimum. Actually no meshgrid points within the loops. With (b) and (c), we can see that the loss roughly decreases along the contour area at the level near the global minimum from both sides to the center. }
  \label{fig-initialPoints-debugging-lossContoursDetail-skew-qLogN-p0LogN}
\end{figure}

\begin{figure}[!htbp]
  \centering
  \subfigure[$\mathcal{L}_{\tilde{\bm c}}$]{\includegraphics[width=0.3\textwidth]{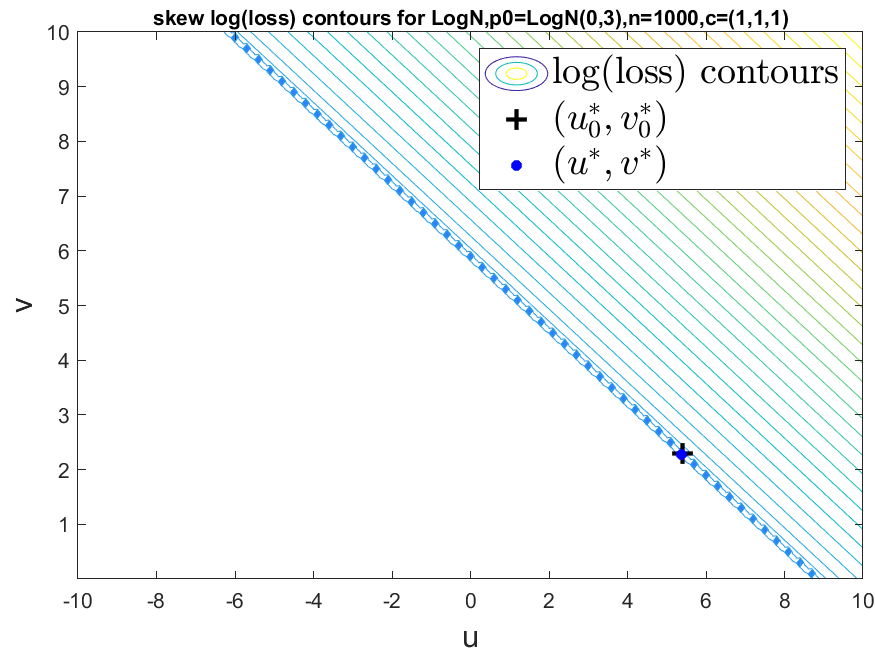}}\\
  \subfigure[$L_1$]{\includegraphics[width=0.3\textwidth]{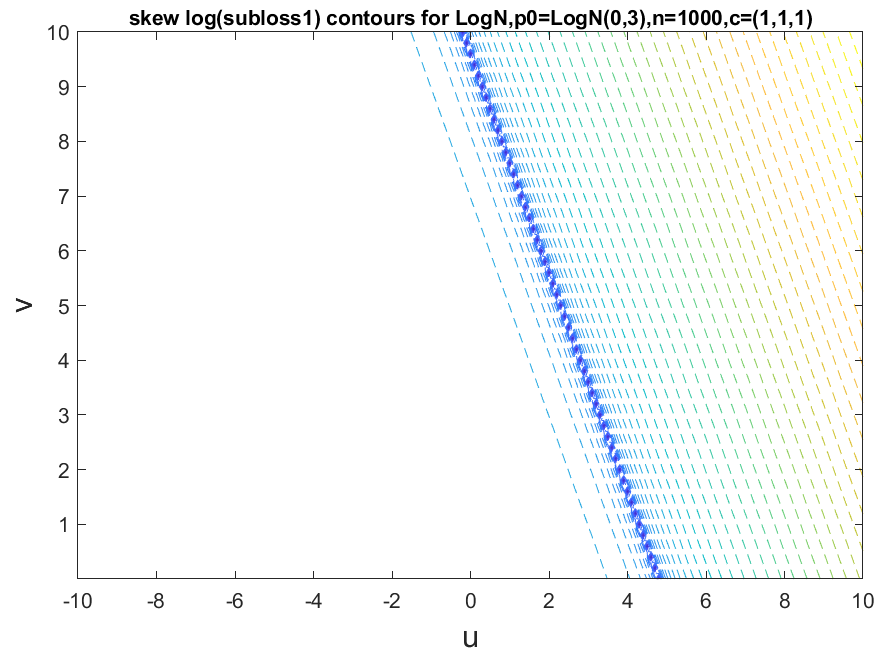}}
  \subfigure[$L_2$]{\includegraphics[width=0.3\textwidth]{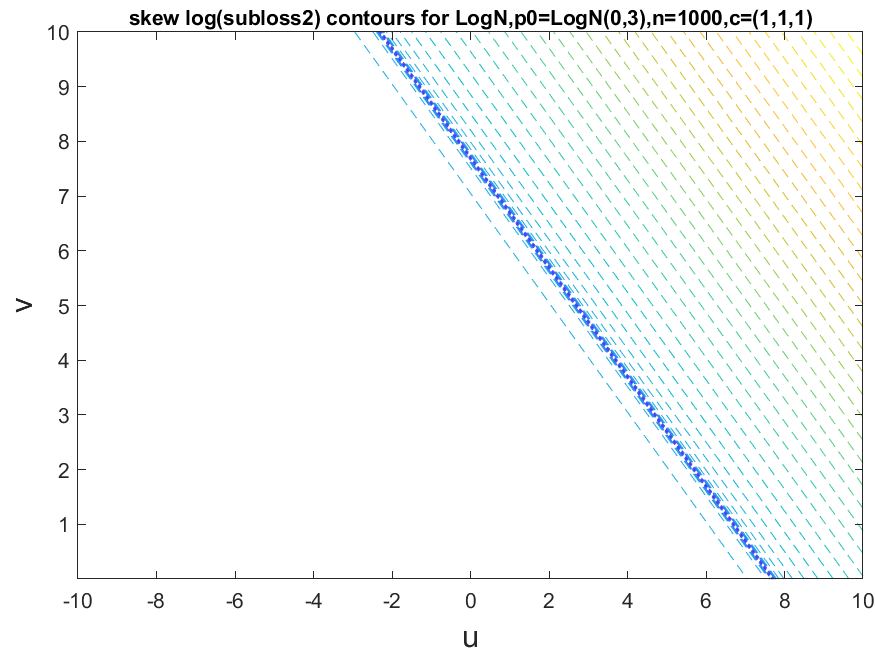}}
  \subfigure[$L_3$]{\includegraphics[width=0.3\textwidth]{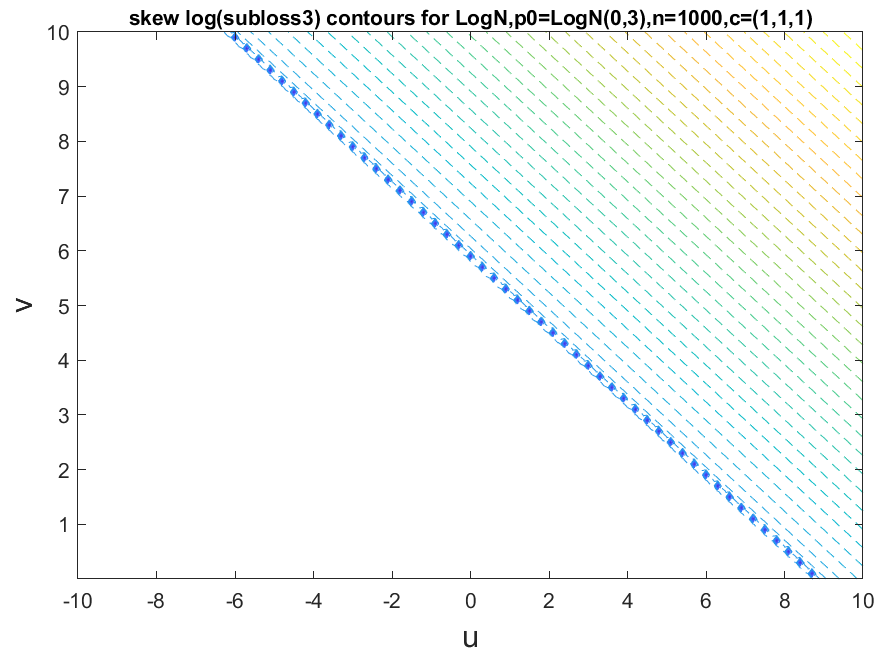}}
  
  \caption{ The contours of the total loss $\mathcal{L}_{\tilde{\bm c}}$ with $\bm c = (1,1,1)$, and also other sub-losses. Here, $\Gamma = skew(\cdot)$, $\bm q_{\theta} = LogN(u, v^2)$ with $\bm \theta =(u, v)$, $\bm p_0 = LogN(0,3^2)$. (Notice that the vertical axis refers to $v^2$ instead of $v$ itself.) The loss contours are plot based on the total loss values on the meshgrid of width 0.1 generated on the plane of $(u,v^2)$. It is obvious that $\mathcal{L}_{\tilde{\bm c}}$ in (a) is almost same to the sub-loss $L_3$ in (d). } \label{fig-lossContours-skew-qLogN-p0LogN-totalANDsub}
\end{figure}

We observe the loss contours of the previous example with $\bm p_0 = LogN(0,3^2)$, see Figure \ref{fig-initialPoints-debugging-lossContoursDetail-skew-qLogN-p0LogN}. Those sub-figures show that the loss contours are almost linear except that at the level near the global minimum. It looks like that the total loss has an infinite number of local minima, but we mention in the caption of the figure that possibly it is a visualization mistake for the weird contours. The explanation for the shape of the contours is as follows. For the log-normal distribution model $\bm q_{\theta} = LogN(u, v^2)$, the $i$-order moment formulation is $E[X^i] = e^{iu+\frac{1}{2} i^2 v^2}$, and thus for each sub-loss $L_i(e^{iu+\frac{1}{2} i^2 v^2}; \hat{\bm p})$ the contour at any level would be a straight line or a pair of parallell lines. According to Table \ref{table-initialPoint-debugging-optimizationDetail-skew-qLogN}, the sub-loss $L_3$ is much larger than other sub-losses and thus lead the total loss $\mathcal{L}_{\tilde{\bm c}}$  It means that $\mathcal{L}_{\tilde{\bm c}}$ is almost same to $L_3$, and its contours are similar to the linear contours of $L_3$, demonstrated in Figure \ref{fig-lossContours-skew-qLogN-p0LogN-totalANDsub}. But, the dominance of $L_3$ is weakened when approaching to the global minimum of $\mathcal{L}_{\tilde{\bm c}}$, and then the contour of $\mathcal{L}_{\tilde{\bm c}}$ at the level near the global minimum is not nearly-linear any more. That being said, the contour area of $\mathcal{L}_{\tilde{\bm c}}$ around the level of the global minimum is very slim and long, and the optimization algorithm would stop soon once the iteration points enter the slim area. This is why there are an infinite number of traps in the visualized contour image of the total loss $\mathcal{L}_{\tilde{\bm c}}$ and the optimal solutions given by different optimization algorithms are always sensitive to the choice of the initial iteration point.

\begin{figure}[!htbp]
  \centering
  \subfigure[Original images]{\includegraphics[width=0.45\textwidth]{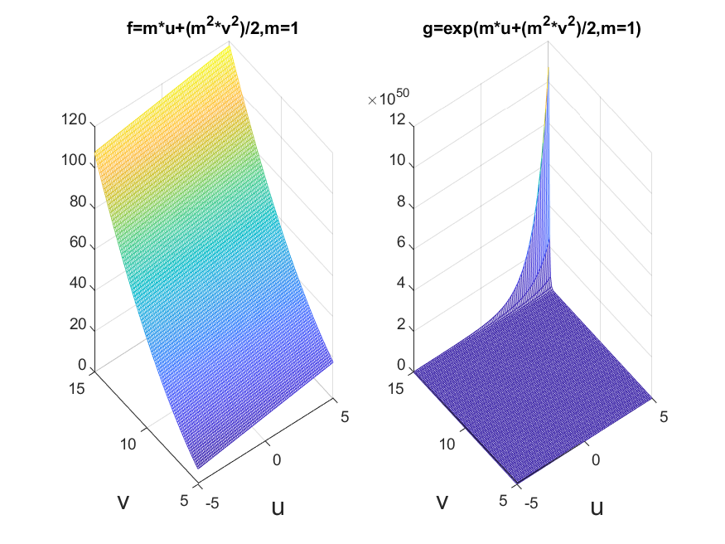}}
  \subfigure[Contours]{\includegraphics[width=0.45\textwidth]{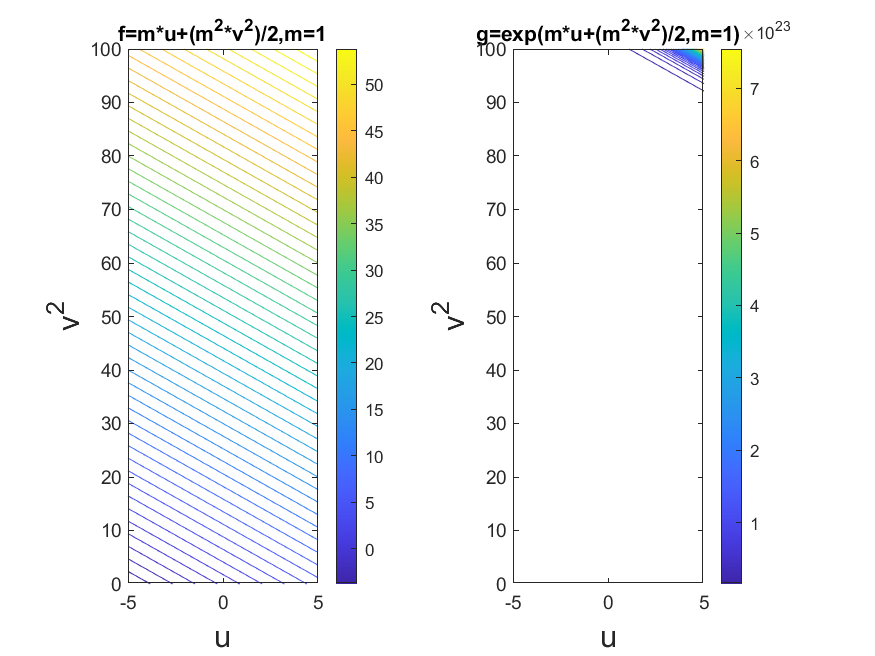}}\\
  
  \caption{ The original images and also contours of functions $f(u,v)=mu+\frac{1}{2}m^2 v^2$ and $g(u,v) = e^{f(u,v)}$. }\label{fig-exp-illustration}
\end{figure}

\begin{figure}[!htbp]
  \centering
  \subfigure[$v_0 = 0.1$]{\includegraphics[width=0.3\textwidth]{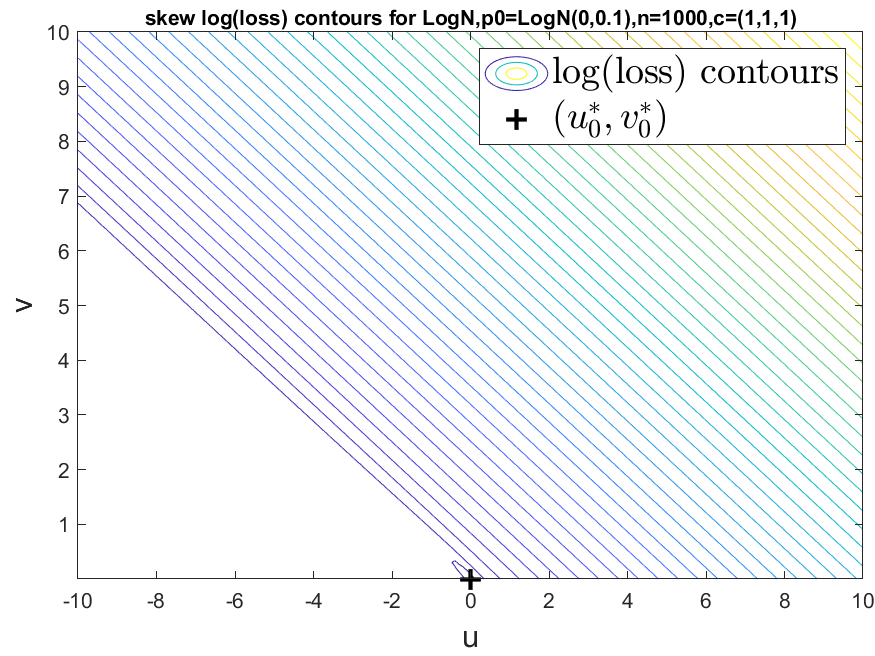}}
  \subfigure[$v_0 = 0.5$]{\includegraphics[width=0.3\textwidth]{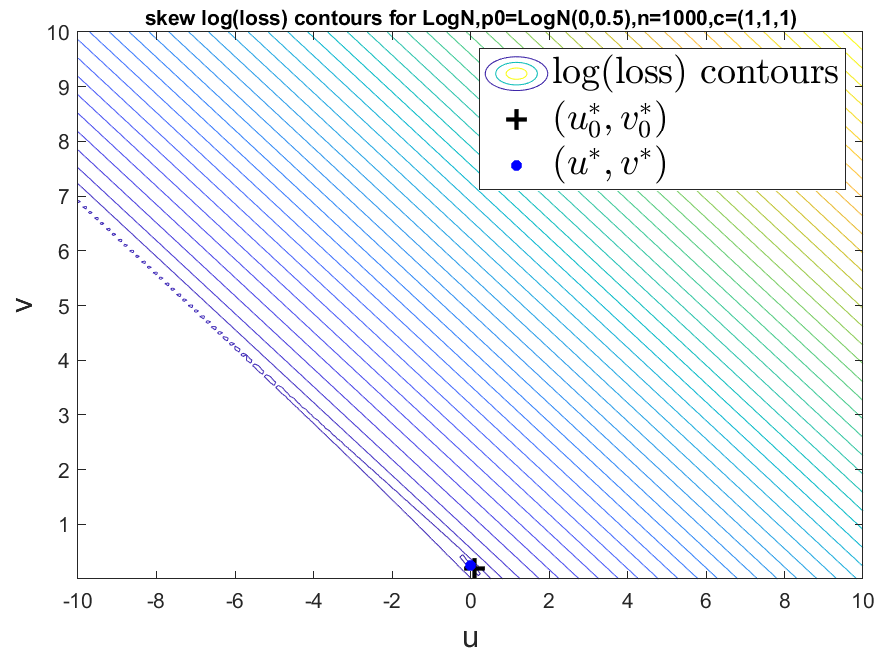}}  
  \subfigure[$v_0 = 3$]{\includegraphics[width=0.3\textwidth]{skew_logloss_contours_for_LogN,p0=LogN0,3,n=1000,c=1,1,1_random3.png}}
  
  \caption{ The contours of the total loss $\mathcal{L}_{\tilde{\bm c}}$ for different $\bm p_0$ with $\bm c = (1,1,1)$. Here, $\Gamma = skew(\cdot)$, $\bm q_{\theta} = LogN(u, v^2)$ with $\bm \theta =(u, v)$, $\bm p_0 = LogN(0,v_0^2)$.  (Notice that the vertical axis refers to $v^2$ instead of $v$ itself.) } \label{fig-lossContours-skew-qLogN-p0}
\end{figure}

In the above example, the dominance of $L_3$ is caused by the exponential operator $exp(\cdot)$. The moments of $\bm q_{\theta} = LogN(u, v^2)$ are exponential with  $E[X^i] = e^{iu+\frac{1}{2} i^2 v^2}$ for all $i$, see Figure \ref{fig-exp-illustration}, and the samples are generated from $\bm p_0 = LogN(u_0, v_0^2)$ that is the exponential of a normal distribution. All values are exponentially exaggerated, and different sub-losses are in different magnitudes, especially if $i$ and $v_0$ are large. Figure \ref{fig-lossContours-skew-qLogN-p0} shows that when $v_0$ is getting larger, the contour area around the level of the global minimum is getting slimmer and longer. The optimization difficulty gets bigger as $v_0$ increases. 

So far, we have figured out how the optimization difficulty, i.e., the sensitivity w.r.t. the choice of the initial iteration points for optimization algorithms, happens to the the example with $\bm q_{\bm \theta} = LogN(u,v^2)$. The reasons are from two aspects. One is that the sub-loss $L_3$ have linear contours and a whole straight line of minimizers. The other is that the sub-loss $L_3$ takes too large values and thus dominates the total loss $\mathcal{L}_{\tilde{\bm c}}$. The two aspects make the total loss $\mathcal{L}_{\tilde{\bm c}}$ have a very slim and long contour area near the global minimum and cause an optimization difficulty. To relieve the issue, basically we need to remove the dominance of poorly-behaved sub-losses.

\section{Renormalization of weights} \label{appendix-weightRenormalization}

We try the skewness property for $\Gamma$ whose sub-properties are the first three moments $\hat{r}_i = E_{X \sim \bm p} [X^i]$, with different settings of the distribution model and the true distribution. The experiment design detail and major results can be seen in Appendix \ref{appendix-section-experimentsForSkewness}. At the beginning, we set $ c_{-i} = 1 $ by default, but later we find that it is not enough. Appendix \ref{appendix-LogN-optimizationDifficulties} shows the detail about the optimization difficulties for the skewness property estimation with log-normal distribution model assumption. As an abstract of the discussed challenge, when the following two issues occur, 
\begin{itemize}
    \item[(1)] some sub-loss is hard to optimize individually;
    \item[(2)] the sub-loss takes too large values, naturally dominating other sub-losses;
\end{itemize}
then the total loss will also have similar optimization difficulties. Thus, the weight renormalization is needed before running $\Gamma_{c_i}$ curves in order to relieve the intrinsic dominance of some poorly-behaved sub-losses and also to guide us to a more effective range of weights. In this section we will show the performance of weight renormalization.

\subsection{Major results with weight renormalization} \label{appendix-weightRenormalization-majorResults}

\begin{figure*}[t]
  \centering
  \subfigure[Before renormalizing weights]{\includegraphics[width=0.4\textwidth]{skew_logloss_contours_for_LogN,p0=LogN0,3,n=1000,c=1,1,1_random3.png}}
  \subfigure[After renormalizing weights]{\includegraphics[width=0.4\textwidth]{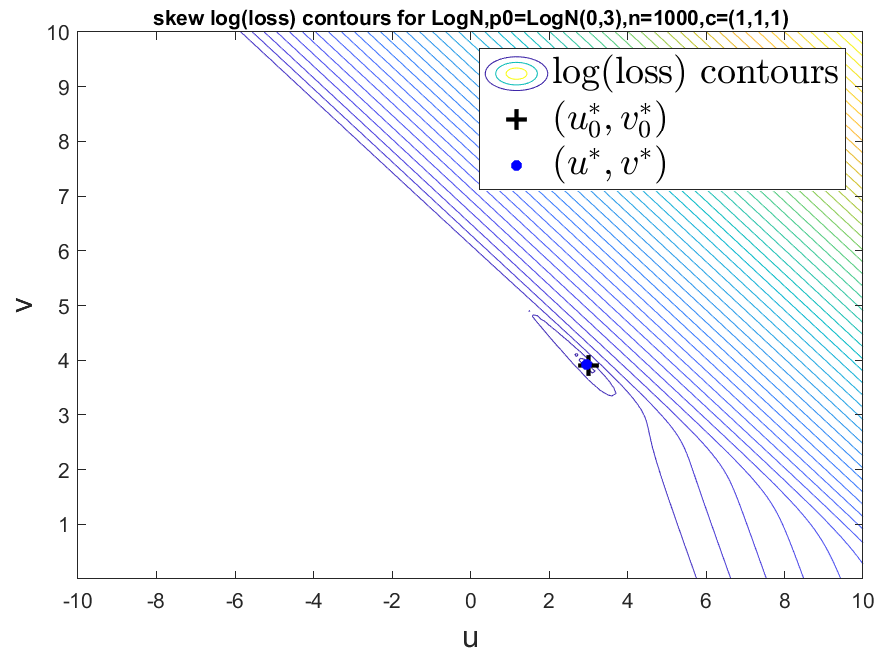}}
  
  \caption{ The contours of the total loss $\mathcal{L}_{\tilde{\bm c}}$ with $\bm c = (1,1,1)$ and different settings for $\bm k$. Here, $\Gamma = skew(\cdot)$, $\bm q_{\theta} = LogN(u, v^2)$ with $\bm \theta =(u, v^2)$, $\bm p_0 = LogN(0,3^2)$.  (Notice that the vertical axis refers to $v^2$ instead of $v$ itself.) } \label{fig-lossContours-skew-qLogN-p0LogN-differentK}
\end{figure*}

To relieve the issue of optimization difficulties, basically we need to remove the dominance of poorly-behaved sub-losses with renormalizing the weights in the way $\tilde{\bm c} = \bm c ./ \bm k$, where $./$ refers to element-wise division and $1/\bm k$ with $\bm k \in \mathbb{R}_{++}^M$ characterizes the base weights to balance sub-losses. Originally, we assume all elements of $\bm k $ to be $1$.

We try setting $k_i = \hat{r}^2(\hat{\bm p})$ for all $i$. Figure \ref{fig-lossContours-skew-qLogN-p0LogN-differentK}(b) shows the contours of the total loss with the new setting of $\bm k$, which obviously decreases the optimization difficulty. However, it only works near the global minimum, which is reasonable since the renormalization of $\tilde{\bm c}$ would matter less as the sub-loss values go larger. 

Furthermore, we also observe that off-setting the intrinsic dominance of sub-losses by setting base weights $1/\bm k$ properly guides us to find a more effective range of $\tilde{\bm c}$. Figure \ref{fig-property-c-curves-skew-qLogN-p0LogN-differentK} shows that for that example the new setting of $\bm k$ makes the $\Gamma_{c_i}$ curves much more sensitive to the values of $\bm c$ and even helps $\Gamma_{c_i}$ reach a value closer to the true value of the target property. We give more results about how the weight renormalization affects the $\Gamma_{c_i}$ curves in Appendix \ref{appendix-section-moreExperiments-k}. After resetting $\bm k$, we observe that among the simulated results for the skewness property, the monotonicity pattern of $\gamma(\bm \theta_{\bm c}^*)$ with increasing $c_i$ for all $i$ is still common for different settings of $\bm p_0$ and $\mathcal{D}_{\Theta}$.

\begin{figure*}[t]
  \centering
  \subfigure[Before renormalizing weights]{\includegraphics[width=0.4\textwidth]{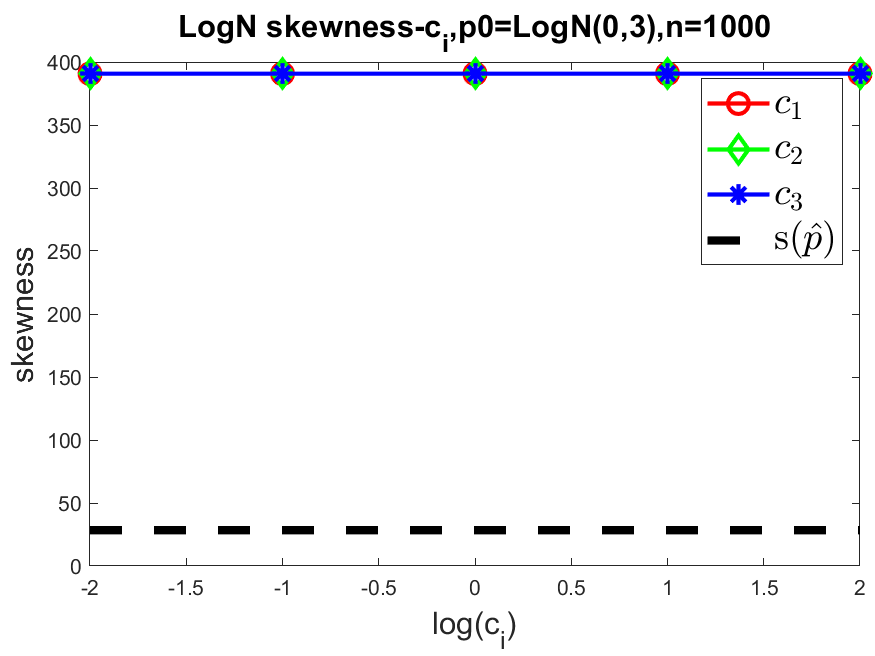}}
  \subfigure[After renormalizing weights]{\includegraphics[width=0.4\textwidth]{LogN_skewness_c_1D,p0=LogN0,3,n=1000_random3_new2.png}\label{fig-property-c-curves-skew-qLogN-p0LogN-new2}}
  
  \caption{ The $\Gamma_{c_i}$ curves with different different settings of $\bm k$. Here, $\Gamma = skew(\cdot)$, $\bm q_{\bm \theta} = LogN(u, v^2)$ with $\bm \theta =(u, v)$, $\bm p_0 = LogN(0,3^2)$. The initial point $ \bm \theta_0^*$ for the optimization algorithm is the minimizer of the total loss on a meshgrid generated around the global minimizer with cubes of width 0.1. }\label{fig-property-c-curves-skew-qLogN-p0LogN-differentK}
\end{figure*}

Besides, we also tried another different approach to reset the base weights $\bm k$, but have not noticed any obvious advantages through experimental results.

\subsection{More experimental results about base weights} \label{appendix-section-moreExperiments-k}

In this section we show more results about the comparison between without and with weight renormalization in Figure \ref{fig1} $\sim$ \ref{fig3}. "new2" refers to the way of weight renormalization used throughout this paper as described in \ref{appendix-weightRenormalization-majorResults}

\begin{figure}[!htbp]
  \centering
  \subfigure[]{\includegraphics[width=0.3\textwidth]{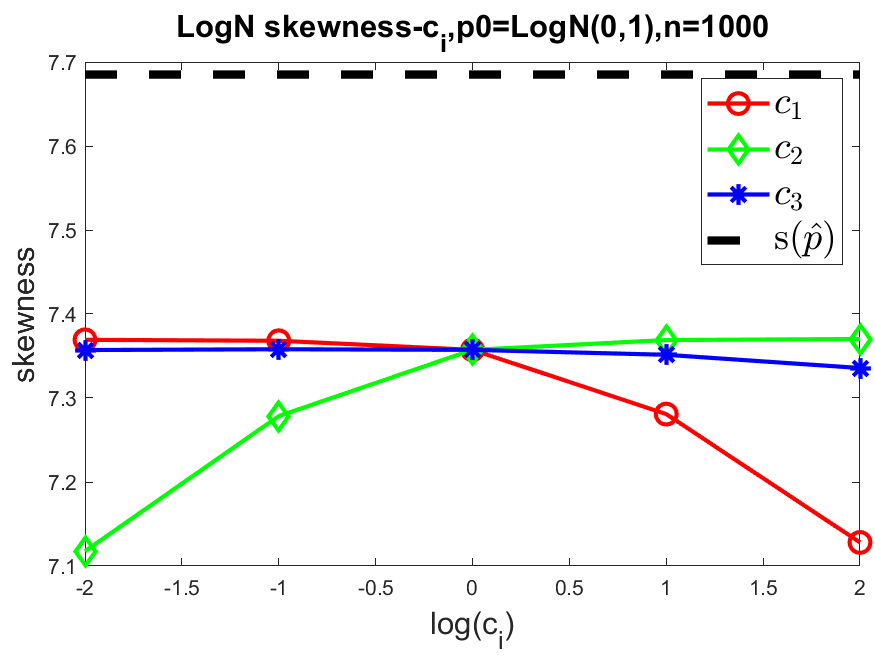}}
  \subfigure[new2]{\includegraphics[width=0.3\textwidth]{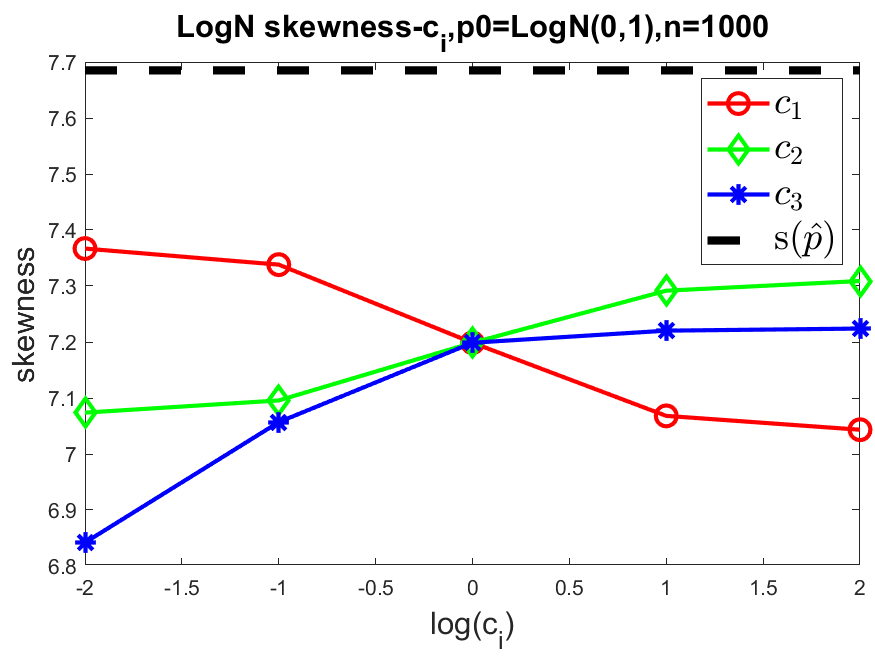}}
  
  \subfigure[]{\includegraphics[width=0.3\textwidth]{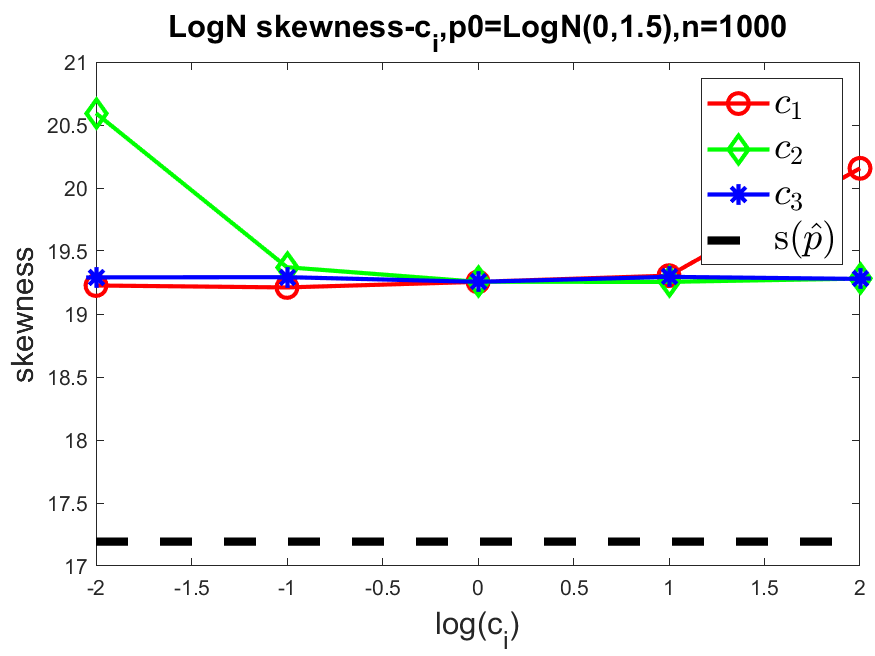}}
  \subfigure[new2]{\includegraphics[width=0.3\textwidth]{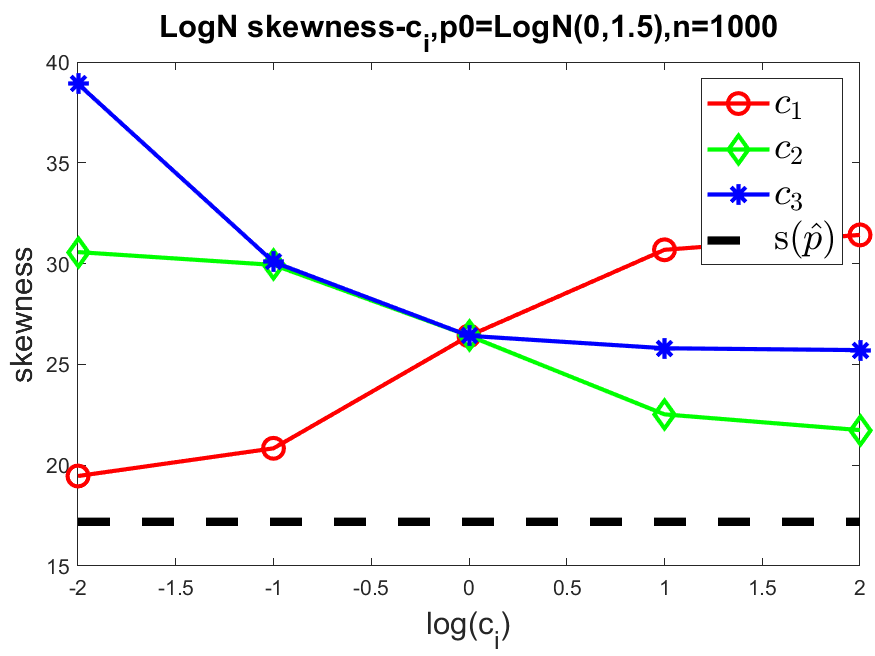}}
  
  \subfigure[]{\includegraphics[width=0.3\textwidth]{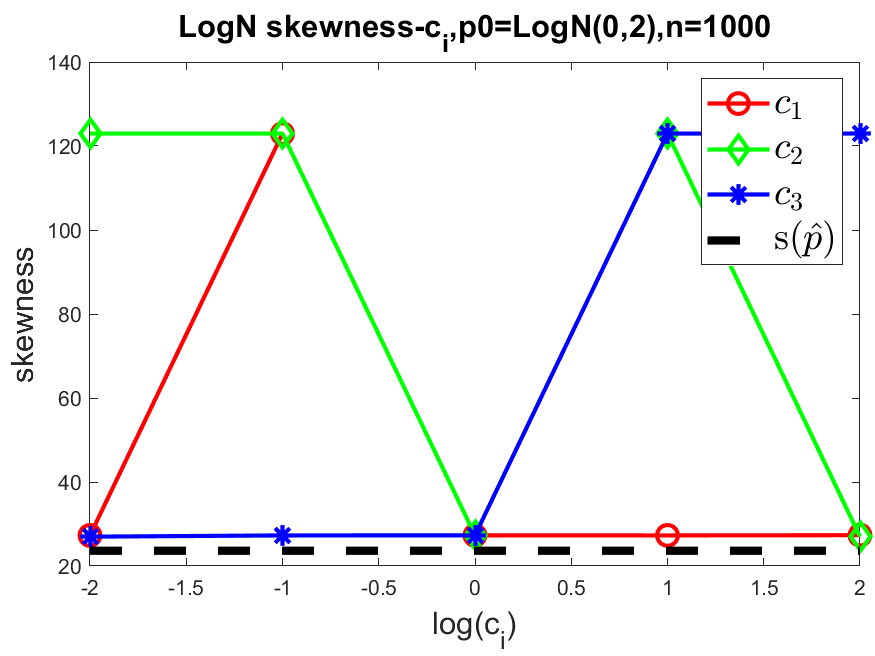}}
  \subfigure[new2]{\includegraphics[width=0.3\textwidth]{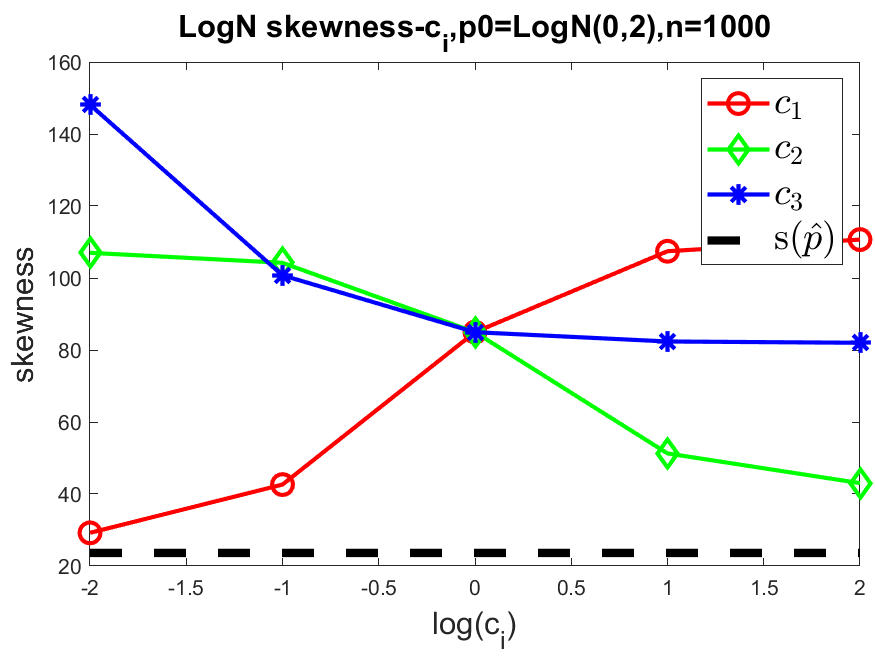}}
  
  \subfigure[]{\includegraphics[width=0.3\textwidth]{LogN_skewness_c_1D,p0=LogN0,3,n=1000_random3.png}}
  \subfigure[new2]{\includegraphics[width=0.3\textwidth]{LogN_skewness_c_1D,p0=LogN0,3,n=1000_random3_new2.png}}
  
  \caption{ The $\Gamma_{c_i}$ curves with different settings for $\bm k$ for the cases of $\bm q_{\theta} = LogN(u, v^2)$ and $\bm p_0 = LogN(0, v_0^2)$. Each row corresponds to a different value of $v_0$. Here, $\Gamma = skew(\cdot)$. }\label{fig1}
\end{figure}

\begin{figure}[!htbp]
  \centering
  \subfigure[]{\includegraphics[width=0.3\textwidth]{LogN_skewness_c_1D,p0=LogN0,1,n=1000_random3.png}}
  \subfigure[new2]{\includegraphics[width=0.3\textwidth]{LogN_skewness_c_1D,p0=LogN0,1,n=1000_random3_new2.png}}
  
  \subfigure[]{\includegraphics[width=0.3\textwidth]{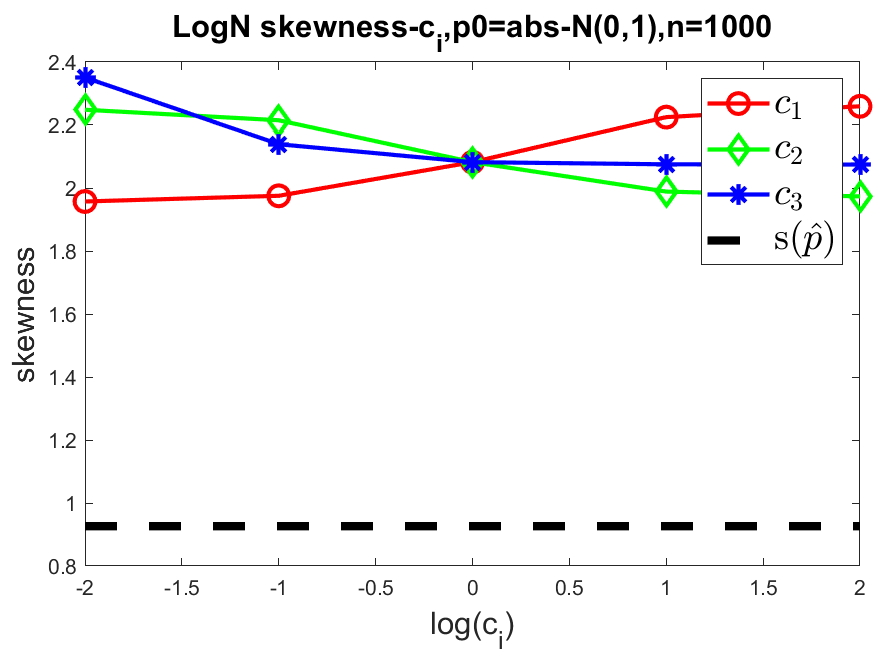}}
  \subfigure[new2]{\includegraphics[width=0.3\textwidth]{LogN_skewness_c_1D,p0=abs-N0,1,n=1000_random3_new2.png}}
  
  \subfigure[]{\includegraphics[width=0.3\textwidth]{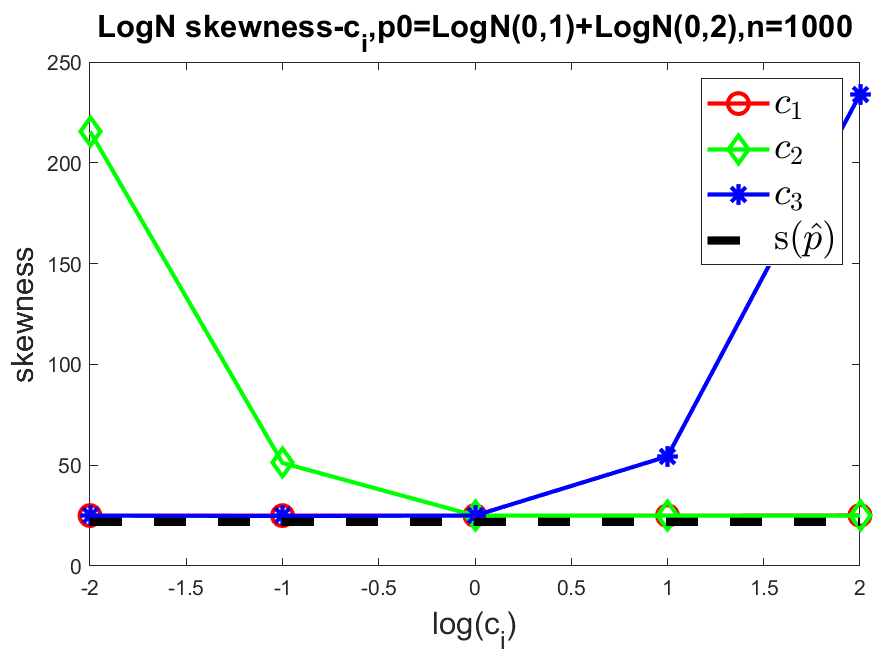}}
  \subfigure[new2]{\includegraphics[width=0.3\textwidth]{LogN_skewness_c_1D,p0=LogN0,1+LogN0,2,n=1000_random3_new2.png}}
  
  \caption{ The $\Gamma_{c_i}$ curves with different settings for $\bm k$ for the cases of $\bm q_{\theta} = LogN(u, v^2)$ and $\bm p_0$ from different families. Each row corresponds to a different family of $\bm p_0$. Here, $\Gamma = skew(\cdot)$. }\label{fig2}
\end{figure}

\begin{figure}[!htbp]
  \centering
  \subfigure[]{\includegraphics[width=0.3\textwidth]{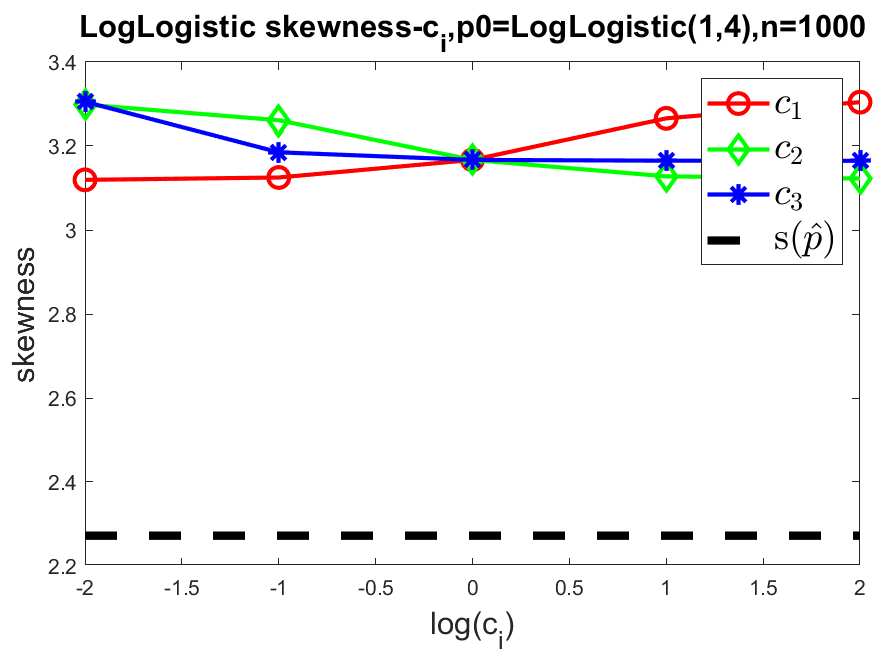}}
  \subfigure[new2]{\includegraphics[width=0.3\textwidth]{LogLogistic_skewness_c_1D,p0=LogLogistic1,4,n=1000_random3_new2.png}}
  
  \subfigure[]{\includegraphics[width=0.3\textwidth]{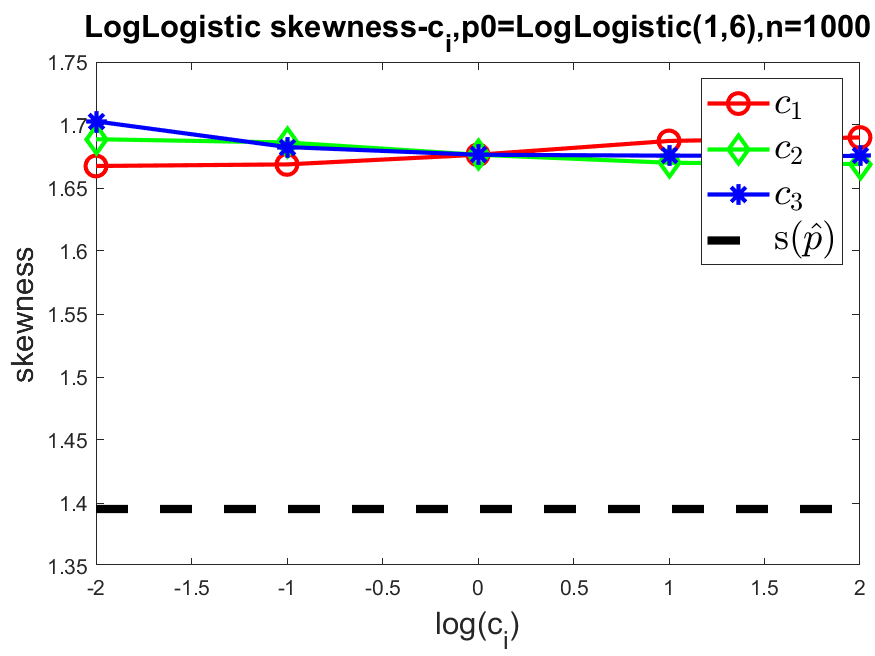}}
  \subfigure[new2]{\includegraphics[width=0.3\textwidth]{LogLogistic_skewness_c_1D,p0=LogLogistic1,6,n=1000_random3_new2.png}}
  
  \subfigure[]{\includegraphics[width=0.3\textwidth]{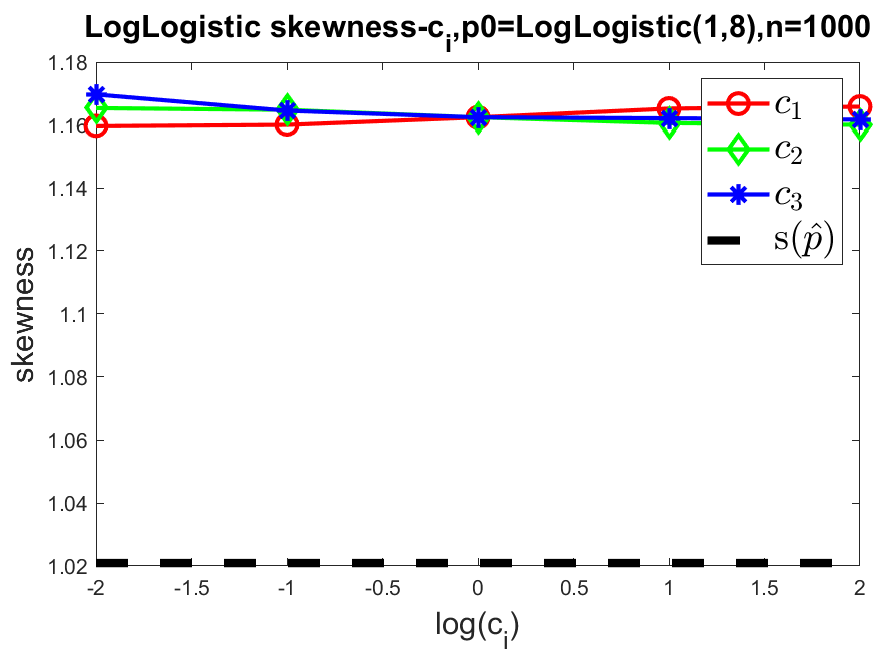}}
  \subfigure[new2]{\includegraphics[width=0.3\textwidth]{LogLogistic_skewness_c_1D,p0=LogLogistic1,8,n=1000_random3_new2.png}}
  
  \caption{ The $\Gamma_{c_i}$ curves with different settings for $\bm k$ for the cases of $\bm q_{\theta} = LogLogistic(a, b)$. Each row corresponds to a different parameter setting of $b_0$ for $\bm p_0 = LogLogistic(a_0, b_0)$. Here, $\Gamma = skew(\cdot)$. }\label{fig3}
\end{figure}

\end{document}